\documentclass[journal]{IEEEtran}

\usepackage[utf8]{inputenc} 
\usepackage[T1]{fontenc}    

\usepackage{url}            
\usepackage{booktabs}       
\usepackage{amsfonts}       
\usepackage{nicefrac}       
\usepackage{microtype}      
\usepackage{xcolor}         

\usepackage{ulem}
\usepackage{graphicx}
\usepackage{wrapfig}
\usepackage{mathtools}
\usepackage{amsmath, amssymb,amsthm}
\usepackage{enumitem}
\usepackage[caption=false,font=footnotesize]{subfig}
\usepackage{cite}
\usepackage{algorithm}
\usepackage{algpseudocode}
\usepackage{multirow}
\usepackage{soul}    
\usepackage{verbatim}
\usepackage{array}
\usepackage{graphbox}
\usepackage{hyperref}       
\usepackage[capitalize,noabbrev]{cleveref}
\theoremstyle{plain}
\newtheorem{theorem}{Theorem}[section]

\newtheorem{lemma}[theorem]{Lemma}

\theoremstyle{definition}
\newtheorem{definition}[theorem]{Definition}

\title{Conditional Unbalanced Optimal Transport Maps: An Outlier-Robust Framework for Conditional Generative Modeling}

\author{
Jiwoo Yoon, Kyumin Choi, Jaewoong Choi%
\thanks{Jiwoo Yoon and Jaewoong Choi are with the Department of Statistics, Sungkyunkwan University. Kyumin Choi is with the Department of Applied Artificial Intelligence, Sungkyunkwan University. E-mail: yoon6019@skku.edu, skku02ky1224@g.skku.edu, jaewoongchoi@skku.edu. Jiwoo Yoon and Kyumin Choi contribute equally to this work. (Corresponding author: Jaewoong Choi)}%
\thanks{This work was supported by the National Research Foundation of Korea (NRF) grant funded by the Korea government (MSIT) (No. RS-2024-00349646).}%
\thanks{This work has been submitted to the IEEE for possible publication. Copyright may be transferred without notice, after which this version may no longer be accessible.}
}

\begin{document}
\maketitle

\begin{abstract}
Conditional Optimal Transport (COT) problem aims to find a transport map between conditional source and target distributions while minimizing the transport cost.
Recently, these transport maps have been utilized in conditional generative modeling tasks to establish efficient mappings between the distributions. 
However, classical COT inherits a fundamental limitation of optimal transport, i.e., sensitivity to outliers, which arises from the hard distribution matching constraints. This limitation becomes more pronounced in a conditional setting, where each conditional distribution is estimated from a limited subset of data. 
To address this, we introduce the Conditional Unbalanced Optimal Transport (CUOT) framework, which relaxes conditional distribution-matching constraints through Csiszár divergence penalties while strictly preserving the conditioning marginals. We establish a rigorous formulation of the CUOT problem and derive its dual and semi-dual formulations. Based on the semi-dual form, we propose Conditional Unbalanced Optimal Transport Maps (CUOTM), an outlier-robust conditional generative model built upon a triangular $c$-transform parameterization. We theoretically justify the validity of this parameterization by proving that the optimal triangular map satisfies the $c$-transform relationships. 
Our experiments on 2D synthetic and image-scale datasets demonstrate that CUOTM achieves superior outlier robustness and competitive distribution-matching performance compared to existing COT-based baselines, while maintaining high sampling efficiency.
\end{abstract}


\section{Introduction}
\IEEEPARstart{C}{onditional} generative modeling aims to approximate a conditional distribution $p_{\mathrm{data}}(\cdot \mid y)$, where $y$ denotes a conditioning variable, such as a class label \cite{miyato2018cgans}, a text prompt \cite{esser2024scaling}, or noisy measurements in an inverse problem \cite{wang2025efficient}. 
Recent advances in generative modeling, including GANs \cite{miyato2018cgans, brock2018large, styleganada2020}, diffusion models \cite{song2020score}, and flow matching models \cite{liu2022flow}, have achieved remarkable success in modeling such conditional distributions.
Among existing approaches, \textbf{Conditional Optimal Transport (COT)} has emerged as a principled framework for conditional generation due to its rigorous theoretical formulation and sampling efficiency \cite{hosseini2025conditional, kerrigan2024dynamic, chemseddine2025conditional}. While standard Optimal Transport (OT) investigates the transport problem between two \textit{marginal distributions} (the source $\eta$ and the target $\nu$), the COT problem focuses on the transport between \textit{conditional distributions}. Specifically, for each fixed condition $y$, the objective is to learn a transport map (or transport plan) that transforms the source conditional distribution $\eta(\cdot|y)$ into the target conditional distribution $\nu(\cdot|y)$. This provides a natural principled foundation for conditional generative modeling by setting $\nu = p_{\mathrm{data}}$.

However, classical OT formulations (both marginal and conditional) are known to be sensitive to outliers \cite{NEURIPS2020_9719a00e, uot-robust}. This sensitivity originates from the hard distribution-matching constraints, which force the transport plan to account for every empirical sample \cite{uotm, uot-robust}. When the data contains noise or corrupted samples, the optimal transport map can be significantly distorted, resulting in unstable estimation depending on sampling variation.
\textbf{This vulnerability is further amplified in the conditional setting}. In conditional generative modeling, the dataset is partitioned according to the conditioning variable. As a result, each conditional distribution is estimated from a smaller subset of data.

In these data-sparse regimes, even a few outliers can exert a disproportionately large influence on the learned conditional transport map, thereby amplifying the instability inherent in classical OT. In practical applications, where imperfect or contaminated data are common \cite{10.1145/2983323.2983660, hu2024anomalydiffusion}, this limitation can severely affect conditional generation quality.

To overcome this limitation, we introduce the \textbf{\textit{Conditional Unbalanced Optimal Transport (CUOT)}} framework. Our key idea is to relax the rigid distribution-matching constraints of COT while preserving the structural constraint on the conditioning variable. Specifically, we incorporate Csiszár divergence penalties to allow controlled deviations within each conditional distribution, while strictly enforcing the marginal constraint on the conditioning variable $y$. This formulation preserves conditional alignment while introducing robustness against outliers.

Building upon this formulation, we propose \textbf{\textit{Conditional Unbalanced Optimal Transport Maps (CUOTM)}}, a conditional generative model derived from the semi-dual formulation of the CUOT problem. CUOTM parameterizes the conditional $c$-transform using a triangular transport map and enables efficient adversarial-style training. We theoretically justify this parameterization by proving that the optimal conditional transport map satisfies this relationship. In addition, we establish bounds that quantify how the relaxed conditional marginals deviate from the original distributions, explicitly characterizing the robustness–accuracy trade-off.

Experimental results on 2D synthetic datasets and class-conditional CIFAR-10 demonstrate that CUOTM achieves strong robustness to outliers while maintaining competitive distribution-matching performance. In particular, our model outperforms existing dynamic approaches that require 100 NFEs, while employing only one NFE for CIFAR-10 conditional generation task. Our contributions can be summarized as follows:
\begin{itemize}
    \item We introduce the first mathematical formulation of the Conditional Unbalanced Optimal Transport problem, which relaxes conditional distribution-matching constraints via divergence penalties while preserving the conditioning marginals.
    
    \item We establish the dual and semi-dual formulations of CUOT, extending classical Unbalanced Optimal Transport (UOT) theory to the conditional setting.
    
    \item We propose Conditional Unbalanced Optimal Transport Maps (CUOTM), a conditional generative model derived from the semi-dual formulation of CUOT through a triangular $c$-transform parameterization. We further provide theoretical justification for the validity of this parameterization.
    \item CUOTM demonstrates superior robustness to outliers while maintaining competitive distribution-matching performance on conditional generation benchmarks.
\end{itemize}

\section{Preliminaries} 
\subsubsection*{1) Notations and Assumptions}
Let $\mathcal{V}, \mathcal{U} \subseteq \mathbb{R}^d$ and $\mathcal{Y} \subseteq \mathbb{R}^c$ be Polish spaces, where $\mathcal{V}$ and $\mathcal{U}$ are closures of connected open sets in $\mathbb{R}^d$, and $\mathcal{Y}$ is the closure of a connected open set in $\mathbb{R}^c$. 
We define $\mathcal{V}$ and $\mathcal{U}$ as the \textbf{source and target data spaces}, respectively, and $\mathcal{Y}$ as the \textbf{conditioning variable space}. The set of Borel probability measures on a space $(\cdot)$ is denoted by $\mathcal{P}(\cdot)$. For a measure $\eta \in \mathbb{P}(\mathcal{V})$ and a measurable map $T: \mathcal{V} \to \mathcal{U}$, the pushforward measure of $\eta$ through $T$ is denoted by $T_\# \eta$, i.e., $T_\# \eta(A) := \eta(T^{-1}(A))$ for any Borel set $A \subseteq \mathcal{U}$. We also denote the space of continuous functions on a given domain by $\mathcal{C}(\cdot)$. Throughout this paper, we consider the quadratic cost function $c(z, v; y, u) = c(v,u) = \tau \| v-u \|_{2}^{2}$ for positive constant $\tau > 0$.

\subsubsection*{2) Optimal Transport}
Optimal Transport (OT) theory investigates the problem of transforming a source probability distribution $\mu \in \mathbb{P}(\mathcal{V})$ into a target distribution $\nu \in \mathbb{P}(\mathcal{U})$ with minimal transport cost \cite{villani, santambrogio}. This transformation can be formulated as a deterministic transport map (\textit{Monge's OT problem}) or a stochastic transport plan (\textit{Kantorovich's OT problem}). 
Formally, in Monge's OT formulation, the objective is to find a deterministic transport map $T: \mathcal{V} \to \mathcal{U}$ that minimizes the total transport cost, given by a cost function $c: \mathcal{V} \times \mathcal{U} \to \mathbb{R}$:

\begin{equation} \label{eq:monge}
\inf_{T_\# \mu = \nu} M(T), \quad \text{where} \quad M(T) := \int_{\mathcal{V}} c(v, T(v)) d\mu(v).
\end{equation}

The optimization is conducted under the constraint $T_\# \mu = \nu$, which requires that the transport map $T$ exactly transforms $\mu$ to $\nu$.
Because $T$ is a measurable function, it defines a deterministic coupling between the source sample $v \sim \mu$ and the target $T(v)$. However, due to this restriction to deterministic maps, the existence of an \textbf{optimal transport map} $T^{\star}$, i.e., the minimizer of $M(T)$, is often not guaranteed depending on the selection of $\mu$ and $\nu$. For instance, a solution may not exist if $\mu$ is a Dirac mass and $\nu$ is a continuous distribution \cite{villani}. 

To address the existence challenge of Monge's formulation, Kantorovich proposed a relaxed formulation that explores an optimal transport plan rather than a deterministic map \cite{Kantorovich1948}:
\begin{equation} \label{eq:kantorovich}
\inf_{\pi \in \Pi(\mu,\nu)}
K(\pi), \,\,\, \text{where} \,\,
K(\pi) := \int_{\mathcal{V}\times\mathcal{U}} c(v,u) \, d\pi(v,u).
\end{equation}
$\Pi(\mu,\nu)$ denote the set of all couplings between $\mu$ and $\nu$:
\begin{equation}
    \Pi(\mu,\nu) := \{ \pi \in \mathbb{P}(\mathcal{V}\times\mathcal{U}) : \pi_{\mathcal{V}} = \mu, ~ \pi_{\mathcal{U}} = \nu \}.
\end{equation}
where $\pi_{\mathcal{V}}$ and $\pi_{\mathcal{U}}$ denote the marginal distributions over $\mathcal{V}$ and $\mathcal{U}$, respectively. The coupling $\pi$ allows a stochastic transformation through the conditional distribution $\pi(\cdot | v)$. Intuitively, this represents how the probability mass of a source sample $v$ is distributed across the target space $\mathcal{U}$. We refer to the minimizer $\pi^{\star}$ of $K(\pi)$ as the \textbf{optimal transport plan}. 

Theoretically, the Kantorovich's OT problem (Eq. \ref{eq:kantorovich}) is a convex relaxation of the non-convex Monge's OT problem (Eq. \ref{eq:monge}). Moreover, while the optimal transport map may fail to exist, the optimal transport plan $\pi^{\star}$ is guaranteed to exist under mild regularity conditions on $\mu$ and $\nu$ \cite{villani}. 
When $T^{\star}$ exists, the solutions to the Monge and Kantorovich problems coincide, i.e., $\pi^\star = (\mathrm{Id}, T^\star)_\# \mu$.
Throughout this paper, we refer to either the transport map $T$ or the transport plan $\pi$ depending on the context.

\subsubsection*{3) Conditional Optimal Transport}
Consider two joint distributions over the data spaces $\mathcal{V}, \mathcal{U}$ and the conditioning variable space $\mathcal{Y}$: a source measure $\eta \in \mathbb{P}(\mathcal{Y} \times \mathcal{V})$ and a target measure $\nu \in \mathbb{P}(\mathcal{Y} \times \mathcal{U})$. In the task of conditional generative modeling, the goal is to recover the conditional distribution $\nu(\cdot | y)$ for any given condition variable $y \in \mathcal{Y}$. 

In this case, the standard OT formulations between $\eta$ and $\nu$ (e.g., Eq. \ref{eq:monge} and \ref{eq:kantorovich}) fail because they do not consider the coupled data-condition structure. This leads to undesirable mappings where a source sample $(y, v)$ is transported to a target $(y', u)$ where $y \neq y'$.
Consequently, \textbf{standard OT is generally insufficient} for tasks requiring condition alignment, such as conditional generation.

To address this, recent works \cite{hosseini2025conditional, kerrigan2024dynamic} proposed the \textbf{Conditional Optimal Transport (COT)} framework, which imposes a \textbf{triangular} structure on the transport map. A map $T: \mathcal{Y} \times \mathcal{V} \to \mathcal{Y} \times \mathcal{U}$ is called \textit{triangular} if it takes the form:
\begin{equation}
    T(y, v) = (T_\mathcal{Y}(y), T_\mathcal{U}(T_\mathcal{Y}(y), v)).
\end{equation}
This parameterization characterizes how the source conditional distribution $\eta(\cdot | y)$ is transported to the target conditional distribution $\nu(\cdot | T_\mathcal{Y}(y))$ \cite{hosseini2025conditional}. 

In particular, by assuming that the conditioning variables are preserved via an identity map, i.e., $T_\mathcal{Y} = \mathrm{Id}$ (which requires matching conditioning marginals $\eta_\mathcal{Y} = \nu_\mathcal{Y}$), the resulting conditional transport map $T_\mathcal{U}$ satisfies the following desired property:
\begin{equation} \label{eq:triangular_pushforward}
    T_\mathcal{U}(y, \cdot)_\# \eta(\cdot|y) = \nu(\cdot|y), \quad \text{for } \eta_\mathcal{Y}\text{-a.e. } y.
\end{equation}

As shown in Eq. \ref{eq:triangular_pushforward}, this formulation ensures that each target conditional $\nu(\cdot|y)$ is generated by the pushforward measure of the source conditional $\eta(\cdot|y)$ with the same conditioning variable $y$. This property motivates the following formal definitions of the Conditional Monge (Eq. \ref{eq:cond_monge}) and Kantorovich problems (Eq. \ref{eq:cond_kantorovich}).

Formally, given $\eta \in \mathcal{P}(\mathcal{Y} \times \mathcal{V})$, $\nu \in \mathcal{P}(\mathcal{Y} \times \mathcal{U})$ with matching conditioning marginals $\eta_\mathcal{Y} = \nu_\mathcal{Y}$, and a cost function $c: (\mathcal{Y} \times \mathcal{V}) \times (\mathcal{Y} \times \mathcal{U}) \to \mathbb{R}$, the \textbf{Conditional Monge problem} \cite{hosseini2025conditional} is defined as: 
\begin{multline} \label{eq:cond_monge}
\inf_{T^C} \int_{\mathcal{Y} \times \mathcal{V}} c\left((y, v), T^{C}(y, v)\right) \, d\eta(y, v) 
\\  \text{s.t.} \;\; T^C(y, v) = (y, T^C_{\mathcal{U}}(y, v)) \;\; \text{and} \;\; T^C_\# \eta = \nu. 
\end{multline} 
Similarly, the \textbf{Conditional Kantorovich problem} \cite{hosseini2025conditional} is defined as a relaxed formulation using transport plans:
 
\begin{equation} \label{eq:cond_kantorovich}
\begin{aligned}
    \inf_{\pi^C \in \Pi_Y(\eta, \nu)} \int_{(\mathcal{Y} \times \mathcal{V}) \times (\mathcal{Y} \times \mathcal{U})} c(z, v; y, u) d\pi^C(z, v; y, u) \\
    \text{s.t.} \;\; \Pi_Y(\eta, \nu) := \{ \gamma \in \Pi(\eta, \nu) \mid z = y  \}.
    \end{aligned}
\end{equation}

These two conditional OT problems satisfy relationships analogous to the standard OT problems. Specifically, under suitable regularity conditions and the quadratic-like cost function $c(z, v; y, u) = \tau \| v-u \|_{2}^{2}$, the optimal conditional transport plan $\pi^{C, \star}$ and the optimal conditional transport map $T^{C, \star}$ satisfy $\pi^{C, \star}  = (Id, T^{C, \star})_{\#} \eta$ (see \cite[Prop. 3.8]{hosseini2025conditional} for precise conditions). 

Finally, the conditional OT problem admits dual formulations analogous to the Kantorovich-Rubinstein duality of standard OT \cite{villani}: 
\begin{small}
\begin{align}
&\sup_{\psi \in L^1(\eta)} \left( \int_{\mathcal{Y} \times \mathcal{U}} \psi_v^c(z, u) d\nu(z, u) - \int_{\mathcal{Y} \times \mathcal{V}} \psi(z, v) d\eta(z, v) \right), \\
&= \sup_{\phi \in L^1(\nu)} \left( \int_{\mathcal{Y} \times \mathcal{U}} \phi(y, u) d\nu(y, u) - \int_{\mathcal{Y} \times \mathcal{V}} \phi_u^c(y, v) d\eta(y, v) \right),
\end{align}    
\end{small}
where $\psi_v^c(z, u) := \inf_v \left( \psi(z, v) + c(z, v; z, u) \right)$ and $\phi_u^c(y, v) := \sup_u \left( \phi(y, u) - c(y, v; y, u) \right)$ denote the partial $c$-transforms of $\psi$ and $\phi$, respectively. 
In \cref{sec:CUOT}, we extend this theory to establish the first formulation of the Conditional Unbalanced Optimal Transport (CUOT) problem. Building upon this formulation, we propose a Neural Optimal Transport algorithm that learns the transport map $T$ via neural networks.

\section{Method} \label{sec:CUOT}
In this section, we introduce our model, called \textit{\textbf{Conditional Unbalanced Optimal Transport Map (CUOTM)}}, a framework designed to achieve robustness against outliers in the conditional optimal transport setting. 
In conditional generation, the impact of outliers is significantly more amplified than in unconditional generation because the number of available data points is reduced for each specific condition. Consequently, handling outliers is more important for conditional modeling.
To address this, CUOTM is a neural OT model that learns the conditional optimal transport map for the \textbf{Conditional Unbalanced Optimal Transport (CUOT)} problem via neural networks.
\begin{itemize}
    \item In \cref{sec:cuotm_theory}, we introduce the Conditional Unbalanced OT problem and establish the existence and uniqueness of its solutions.
    
    \item In \cref{sec:cuotm_dual}, we derive the dual and semi-dual formulations of the CUOT problem.
    \item In \cref{sec:cuotm_model}, we propose our model and its corresponding learning algorithm based on the derived semi-dual form.  
    
\end{itemize}

\subsection{Conditional Unbalanced Optimal Transport} \label{sec:cuotm_theory}

The \textit{Conditional Unbalanced Optimal Transport (CUOT) problem} generalizes the Conditional Kantorovich OT formulation (Eq. \ref{eq:cond_kantorovich}) by incorporating principles from Unbalanced Optimal Transport (UOT) \cite{chizat2018unbalanced, liero2018optimal, uotm}. While standard OT requires strict matching constraints on the marginal distributions, UOT introduces soft-matching constraints via Csiszár divergence to mitigate sensitivity to outliers \cite{uotm, gazdieva2025robust, balaji2020robust}. 

Following this intuition, we define the CUOT problem by relaxing the conditional marginal constraints on the data spaces $\mathcal{V}$ and $\mathcal{U}$.

Formally, given source measure $\eta \in \mathbb{P}(\mathcal{Y} \times \mathcal{V})$ and target measure $\nu \in \mathbb{P}(\mathcal{Y} \times \mathcal{U})$ with matching conditioning marginals $\eta_Y = \nu_Y$, the CUOT problem is defined as: 
\begin{equation}
\label{eq:condUOT}
\begin{aligned}
\inf_{\pi \in \mathcal{M}_{+}} \; 
\biggl[\int_{\mathcal{Y} \times \mathcal{V} \times \mathcal{Y} \times \mathcal{U}}& c(y,v,y',u) \, d\pi(y,v,y',u) \\
& + \int_{\mathcal{Y}} D_{\Psi_1} \bigl( \pi_1(\cdot | y) \| \eta(\cdot | y) \bigr) \, d\eta_{Y}(y) \\ 
& + \int_{\mathcal{Y}} D_{\Psi_2} \bigl( \pi_2(\cdot | y') \| \nu(\cdot | y') \bigr) \, d\nu_{Y}(y')\biggr] \\
&\text{s.t.} \,\, \pi_{Y} = \eta_{Y} = \nu_{Y}, 
\\
&\operatorname{support}(\pi) \subset \{(y,v,y',u):y=y'\}
\end{aligned}
\end{equation}
where $\mathcal{M}_{+} (\cdot)$ denotes the set of positive Borel measures on $(\cdot)$, and $D_{\Psi_1}$, $D_{\Psi_2}$ are Csiszár divergences associated with the convex functions $\Psi_1$ and $\Psi_2$, respectively (See Appendix \ref{app:def} for the formal definition). Here, $\pi_1$ and $\pi_2$ denote the marginals of $\pi$ over $(\mathcal{Y} \times \mathcal{V})$ and $(\mathcal{Y} \times \mathcal{U})$. This formulation minimizes the discrepancy between conditional distributions for each conditioning variable $y, \, y' \in \mathcal{Y}$ leading to a soft-matching of the conditional distributions:
\begin{equation}
\label{eq:softmatching}
D_{\Psi_1} \bigl( \pi_1(\cdot | y) \| \eta(\cdot | y) \bigr) \approx 0 \,\, \text{and} \,\,
D_{\Psi_2} \bigl( \pi_2(\cdot | y') \| \nu(\cdot | y') \bigr) \approx 0.
\end{equation}

Note that in the standard Conditional OT case, the conditional transport map $T^{C}$ in Eq. \ref{eq:cond_monge} and the corresponding conditional transport plan $\pi^{C}$ exactly match conditional distributions, because of Eq. \ref{eq:triangular_pushforward}.
Our CUOT formulation relaxes this constraint (Eq. \ref{eq:softmatching}).
Furthermore, note that we maintain the strict constraint assumption on the conditioning marginals ($\pi_Y = \eta_Y = \nu_Y$) because we follow the triangular transport assumption $T(z,v) = (z, T_{\mathcal{U}}(z,v))$ (Eq. \ref{eq:tri4c-transform}). Consequently, our unbalanced formulation specifically relaxes the marginal constraints on the $\mathcal{V}$ and $\mathcal{U}$ components while preserving the structure of the conditioning variable $\mathcal{Y}$.

This relaxation allows for greater flexibility in the transport problem, effectively addressing the sensitivity to outliers---a major limitation of standard COT. Furthermore, as we show in Section \ref{sec:experiments}, this flexibility leads to more accurate matching of the target distribution in generative modeling tasks. Notably, this improvement is observed not only in the presence of outliers but also in standard, noise-free settings.

To ensure that the CUOT problem is mathematically well-posed, we now proceed to establish the existence and uniqueness of its solution. The following theorem provides the well-posedness under some mild assumptions, such as compactness of data space and condition space.

\begin{theorem}[Existence and Uniqueness of the Minimizer in CUOT]
\label{thm:existence}
Assume that $\mathcal{Y}, \mathcal{V},$ and $\mathcal{U}$ are compact metric spaces. Let $\mathcal P(\cdot)$ be equipped with the topology of weak convergence and $\eta \in \mathcal{P}(\mathcal{Y} \times \mathcal{V})$ and $\nu \in \mathcal{P}(\mathcal{Y} \times \mathcal{U})$ be given source and target measures such that $\eta_Y = \nu_Y$. Suppose the cost function $c: (\mathcal{Y} \times \mathcal{V}) \times (\mathcal{Y} \times \mathcal{U}) \rightarrow \mathbb{R}$ is lower semi-continuous. 
Then, there exists a minimizer $\pi^\star$ for the CUOT problem (Eq. \ref{eq:condUOT}). Furthermore, if the entropy functions $\Psi_1$ and $\Psi_2$ are strictly convex, the minimizer $\pi^\star$ is unique.
\end{theorem}

\subsection{Dual and Semi-dual Formulation of Conditional UOT} \label{sec:cuotm_dual}
In this section, we derive the \textbf{dual and semi-dual formulations} for the Conditional Unbalanced Optimal Transport problem (Eq. \ref{eq:condUOT}). These formulations are essential for constructing our neural optimal transport algorithm, as they transform the constrained optimization over measures into an unconstrained optimization over potential functions. 

\begin{theorem}[Duality for CUOT] \label{thm:dual} 
    For a non-negative cost function $c$, the Conditional Unbalanced Optimal Transport problem has the following dual (Eq. \ref{eq:cuot-dual}) and semi-dual (Eq. \ref{eq:cuot-semidual}) formulation: 
    
    \textbf{Dual Formulation:}
    \begin{multline} \label{eq:cuot-dual}
        \sup_{\phi + \varphi \le c} \left[ \int_{\mathcal{Y}\times\mathcal{V}} -\Psi_1^* \bigl( -\phi(y,v) \bigr) \, \mathrm{d}\eta(y,v) \right. \\
        \left. + \int_{\mathcal{Y}\times\mathcal{U}} -\Psi_2^* \bigl( -\varphi(y,u) \bigr) \, \mathrm{d}\nu(y,u) \right]
    \end{multline} 
    
    \textbf{Semi-Dual Formulation:}
    \begin{multline} \label{eq:cuot-semidual}
        \sup_{\varphi} \left[ \int_{\mathcal{Y}\times\mathcal{V}} -\Psi_1^* \bigl( -\varphi^{c}(y,v) \bigr) \, \mathrm{d}\eta(y,v) \right. \\
        \left. + \int_{\mathcal{Y}\times\mathcal{U}} -\Psi_2^* \bigl( -\varphi(y,u) \bigr) \, \mathrm{d}\nu(y,u) \right]
    \end{multline}
    where $\Psi_1^*$ and $\Psi_2^*$ denote the Legendre-Fenchel conjugates of the divergence-defining functions $\Psi_1$ and $\Psi_2$ in the CUOT problem, respectively. The functional variables $\phi \in \mathcal{C(Y \times V)}$ and $\varphi \in \mathcal{C(Y \times U)}$ are referred to as the \textbf{potential functions}.
    Furthermore, $\varphi^{c}(y,v)$ represents the $y$-conditional $c$-transform, which is defined as:
    \begin{equation} \label{eq:cond-c-transform}
    \varphi^{c}(y,v) := \inf_{u \in  \mathcal{U}} \left\{ c(y, v; y, u) - \varphi(y, u) \right\}.
    \end{equation}
\end{theorem}

\subsection{Parameterization of Conditional UOTM} \label{sec:cuotm_model}
In this section, we describe how we implement a conditional generative model based on the semi-dual formulation (Eq.\ref{eq:cuot-semidual}) of the CUOT problem, called \textit{\textbf{CUOTM}}. Recall that for a given potential $\varphi : \mathcal{Y} \times \mathcal{U} \rightarrow \mathbb{R}$, the $y$-conditional $c$-transform is given by Eq.\ref{eq:cond-c-transform}. Following the $c$-transform parameterization commonly used in Neural OT models \cite{rout2021generative, uotm, ijcai2019p305, fanTMLR}, we introduce a triangular map $T^\triangle : \mathcal{Y} \times \mathcal{V} \rightarrow \mathcal{Y} \times \mathcal{U}$ to approximate the $y$-conditional $c$-transform $\varphi^c$. Specifically, this map $T_\varphi^\triangle$ transforms a source pair $(y,v) \in \mathcal{Y} \times \mathcal{V}$ as follows:
\begin{equation}
\label{eq:tri4c-transform}
\begin{aligned}
    T_\varphi^\triangle(y, v) = (y, T_\varphi(y,v))
\end{aligned}
\end{equation}
The $c$-transform parameterization is then defined as:
\begin{equation}
\label{eq:c-transform param}
    \begin{aligned}
        &T_\varphi(y,v)\in\operatorname*{argmin}_{u\in  \mathcal U}
        \bigl[c(y,v;y,u)-\varphi(y,u)\bigr] \\
        \;\Longleftrightarrow\; &
        \varphi^{c}(y,v)=c\bigl(y,v;y,T_\varphi(y,v)\bigr)
        -\varphi\bigl(y,T_\varphi(y,v)\bigr).
    \end{aligned}
\end{equation}

By substituting this parameterization into the semi-dual objective (Eq. \ref{eq:cuot-semidual}), we derive the following max-min objective:
\begin{small}
\begin{equation}
\label{eq:cuot-Jphi}
\begin{split}
    J(\varphi, T_\varphi) &:= \sup_ {\varphi} \biggl[\int_{\mathcal Y\times\mathcal V} -\Psi_1^{*}\biggl(
    -\inf_{T_{\varphi}}\Bigl[c\bigl(y,v;y,T_\varphi(y,v)\bigr) \\
    &\qquad\qquad\qquad\qquad -\varphi\bigl(y,T_\varphi(y,v)\bigr)\Bigr]
    \biggr)\,d\eta(y,v) \\
    &\quad +\int_{\mathcal Y\times\mathcal U}
    -\Psi_2^{*}\bigl(-\varphi(y,u)\bigr)\,d\nu(y,u)\biggr].
\end{split}
\end{equation}
\end{small}

In practice, there is no closed-form expression of the optimal triangular map $T_\varphi^\triangle$ for each potential $\varphi$. Hence, as in GAN-style training \cite{goodfellow2014generative, uotm,korotin2023neural, hosseini2025conditional}, we jointly optimize over the potential $\varphi$ and the triangular optimal transport map $T_\varphi^\triangle$. We parameterize the potential by a neural network $\varphi_\omega : \mathcal{Y} \times \mathcal V \rightarrow \mathcal U$ and the conditional optimal transport map $T_\varphi^\triangle$ by $T_\theta^\triangle : \mathcal{Y} \times \mathcal{U} \rightarrow \mathcal Y \times \mathcal{U}$ with parameters $\omega$ and $\theta$, respectively, writing $T_\theta^\triangle(y,v) =(y,T_{\theta}(y,v))$ (See \cref{sec:cuotm_algo} for details).  

We establish the validity of our $y$-conditional $c$-transform parameterization via Theorem \ref{thm:parameter}. This theorem proves that the optimal triangular transport map satisfies the introduced $c$-transform relationship:

\begin{theorem}[Validity of Parameterization for CUOT] \label{thm:parameter}
Given the existence of an optimal potential $\varphi^\star$ in the CUOT problem [Eq.\ref{eq:cuot-semidual}], there exists an optimal triangular map $T^{\triangle\star} : \mathcal Y \times\mathcal V \rightarrow \mathcal Y \times \mathcal U$, equivalently $T^{\Delta \star}(y,v) = (y,\, T^\star(y,v)).$ 
such that 
\begin{multline}
    T^\star(y,v) \in \operatorname{arginf}_{u \in \mathcal U} [c(y, v; y, u) - \varphi^\star(y, u)] \\ 
    \text{$\eta$-almost surely.}
\end{multline}
 
By the equivalence established in Eq.\ref{eq:c-transform param}, the $y$-conditional $c$-transform $\varphi^c$ is validly parameterized by the triangular map $T_\varphi^\triangle(y, v) = (y, T_\varphi(y, v))$ as follows:
\begin{equation}
    \varphi^c(y, v) = c(y, v; y, T_\varphi(y, v)) - \varphi(y, T_\varphi(y, v)).
\end{equation}
Moreover, the following inequality holds:

\begin{equation}
\label{eq:div_bound}
\begin{aligned}
    \tau \left({\mathcal{W}_2^{\nu_Y}}\right)^2(\eta, \nu) &\geq \int_{\mathcal Y} D_{\Psi_1}\!\big(\tilde{\eta}(\cdot\mid y)\,\|\,\eta(\cdot\mid y)\big)\, \mathrm d\nu_Y(y) \\
&+ \int_{\mathcal Y} D_{\Psi_2}\!\big(\tilde{\nu}(\cdot\mid y)\,\|\,\nu(\cdot\mid y)\big)\, \mathrm d\nu_Y(y).
\end{aligned}
\end{equation}
where $\tilde{\eta} := \pi_{1}^{\star}$ and $\tilde{\nu} := \pi_{2}^{\star}$ denote the relaxed conditional marginal distributions of the optimal CUOT plan, which are explicitly given by $d\tilde\eta(y,v) = \Psi_1^{*'}(-\varphi^*(y,v))\,d\eta(y,v)$ and $d\tilde\nu(y,u) = \Psi_2^{*'}(-\phi^*(y,u))\,d\nu(y,u)$. Here, ${\mathcal{W}_2^{\nu_Y}}$ denotes the Conditional Wasserstein-2 distance (See Appendix \ref{app:def} for definition).
\end{theorem}

Theorem \ref{thm:parameter} guarantees that our $c$-transform parameterization correctly characterizes the optimal CUOT map. Moreover, this parameterization turns Eq.\ref{eq:cuot-semidual} into a max-min objective (Eq.\ref{eq:cuot-Jphi}) over the potential and the triangular map. Furthermore, the upper-bound on the divergence of shifted conditional marginals, $\tilde{\eta}$ and $\tilde{\nu}$, in Theorem \ref{thm:parameter} guarantees that the relaxed marginals remain sufficiently close to the original conditional distributions, as controlled by the parameter $\tau$.

\subsubsection{Algorithm} \label{sec:cuotm_algo}
We present our training algorithm for CUOTM (Algorithm \ref{alg:uotm_conditional}). Our max-min objective (Eq. \ref{eq:cuot-Jphi}) includes an inner-optimization of the transport network $T_{\theta}$ for the given potential network $\varphi_\omega$. In practice, we employ an alternating training between $T_{\theta}$ and $\varphi_\omega$, following the convention of the GAN literature \cite{goodfellow2014generative}.

The integrals in the learning objective are estimated using mini-batch Monte Carlo estimate. Let $\{(y_j,u_j)\}_{j=1}^{M} \sim \nu$ be a mini-batch of target pairs.
For the same conditioning variables $\{y_j\}_{j=1}^{M}$, we then sample a mini-batch of source data samples $\{v_j\}_{j=1}^{M} \sim \eta_{\mathcal V}$ and form the source pairs $\{(y_j,v_j)\}_{j=1}^{M}$.
Replacing the integrals in~\eqref{eq:cuot-Jphi} with these empirical averages
yields the following adversarial learning objectives. The potential network is updated by minimizing:
{\small
\begin{multline}
\label{eq:Lphi-cond}
\mathcal L_{\varphi}(\omega,\theta)
:=\frac{1}{M}\sum_{j=1}^{M}
\Psi_1^{*}\!\Bigl(
-\bigl[c\bigl((y_j,v_j),T_\theta^\triangle(y_j,v_j)\bigr)
      \\ - \varphi_\omega\bigl(T_\theta^\triangle(y_j,v_j)\bigr)\bigr]
\Bigr) 
+\frac{1}{M}\sum_{j=1}^{M}
\Psi_2^{*}\bigl(-\varphi_\omega(y_j,u_j)\bigr).
\end{multline}}
while the generator is updated by minimizing
{\small
\begin{multline}
\label{eq:LT-cond}
\mathcal L_{T^\triangle}(\omega,\theta)
:=\frac{1}{M}\sum_{j=1}^{M}
\Bigl[
c\bigl((y_j,v_j),T_\theta^\triangle(y_j,v_j)\bigr)
\\- \varphi_\omega\bigl(T_\theta^\triangle(y_j,v_j)\bigr)
\Bigr].
\end{multline}}
Note that for training efficiency, we reuse the same conditioning variable samples $y_{j}$ for both source and target pairs. Moreover, following \cite{rout2021generative, korotin2023neural, uotm}, we provide the generator network $T_\theta^\Delta$ with an auxiliary variable $z$ sampled from $\mathcal{N}(0, I)$ as an additional input. This introduces stochasticity into the learned transport network, thereby providing the ability to approximate the stochastic transport plan $\pi(\cdot)$. This has been shown to be useful in generative modeling.

\begin{algorithm}[t]
\caption{Training algorithm of Conditional UOTM}
\label{alg:uotm_conditional}
\begin{algorithmic}[1]
\Require The target joint distribution $\nu$ and the source conditional distribution $\eta$. A cost function $c$. Non-decreasing, differentiable, convex function pair $(\Psi_1^*, \Psi_2^*)$. Generator network $T_\theta^\triangle$ and the discriminator network $\varphi_\omega$. Total iteration number $K$.

\For{$k=0, 1, 2, \dots, K$}
    \State Sample a mini-batch of target pairs $\{(y_j, u_j)\}_{j=1}^M \sim \nu$.
    \State Sample a mini-batch of source conditionals $\{v_j\}_{j=1}^M \sim \eta$ and form source pairs $\{(y_j, v_j)\}_{j=1}^M$.
    \State Sample $z \sim \mathcal N(\textbf{0,I})$.

    {\small
    \State \parbox[t]{\dimexpr\linewidth-\algorithmicindent\relax}{%
    $\displaystyle
    \begin{aligned}[t]
    \mathcal{L}_\varphi &\gets\; \frac{1}{M}\sum_{j=1}^{M} \Psi_1^*\Bigl(
    -c\bigl((y_j,v_j),T_\theta^\triangle(y_j,v_j)\bigr)
    \\&+ \varphi_\omega\bigl(T_\theta^\triangle(y_j,v_j,z)\bigr)
    \Bigr)
    + \frac{1}{M}\sum_{j=1}^{M} \Psi_2^*\bigl(-\varphi_\omega(y_j,u_j)\bigr).
    \end{aligned}
    $}}
    
    \State Update $\omega$ by minimizing the loss $\mathcal{L}_\varphi$.
    
    \State Sample a mini-batch of source conditionals $\{v_j\}_{j=1}^M \sim \eta$ and form a new batch of source pairs $\{(y_j, v_j)\}_{j=1}^M$ using the same $\{y_j\}$ from line 2.
    \State Sample $z \sim \mathcal N(\textbf{0,I})$
    \State \parbox[t]{\dimexpr\linewidth-\algorithmicindent\relax}{%
    $\displaystyle
    \begin{aligned}[t]
    \mathcal{L}_T \gets\; \frac{1}{M}\sum_{j=1}^{M} c\bigl((y_j,v_j)&,T_\theta^\triangle(y_j,v_j)\bigr)\\
    &-\frac{1}{M}\sum_{j=1}^{M} \varphi_\omega\bigl(T_\theta^\triangle(y_j,v_j,z)\bigr).
    \end{aligned}
    $}
    \State Update $\theta$ by minimizing the loss $\mathcal{L}_T$.
\EndFor
\end{algorithmic}
\end{algorithm}

\section{Related Works}

\subsection{Conditional Optimal Transport}

Recently, Conditional Optimal Transport (COT) has been explored through dynamic formulations, such as flow matching \cite{chemseddine2025conditional, kerrigan2024dynamic, cheng2025curseconditionsanalyzingimproving}. By coupling source and target samples, the COT transport plan aims to provide a more guided inference trajectory. For instance, \cite{chemseddine2025conditional} introduces the conditional Wasserstein distance and analyzes its geodesics to propose OT-based flow matching for Bayesian inverse problems. Extending these concepts beyond finite dimensions, \cite{kerrigan2024dynamic} generalizes the dynamic formulation of COT to infinite-dimensional function spaces via triangular vector fields, broadening its applicability to complex functional data.

The above dynamic formulations, such as the flow matching approaches, require computationally expensive multiple function evaluations (NFE) for sampling. In contrast, static formulations like our model present efficient single-NFE sampling \cite{bunne2022supervised, hosseini2025conditional}. \cite{bunne2022supervised} approaches the COT problem by interpreting the gradients of partially input convex neural networks as optimal transport maps. In particular, \cite{hosseini2025conditional} is highly relevant to our work, as it shares the objective of directly modeling the transport map itself. In this work, we extend the existing COT framework to an unbalanced setting through a semi-dual formulation (Eq. \ref{eq:cuot-semidual}), significantly enhancing robustness against outliers. This is particularly critical for conditional generative modeling, which is inherently more susceptible to outlier-induced distortions than standard (marginal) generative modeling. This increased sensitivity arises because conditional modeling deals with a reduced number of data points available for each specific condition.

\begin{table*}[t]
\caption{\textbf{Quantitative results on various 2D synthetic datasets.} Baselines are categorized into static (NFE = 1) and dynamic (NFE $\ge$ 1) frameworks for comparison. Note that COTM corresponds to a special case of CUOTM where $\Psi_1^*(x) = x$ and $\Psi_2^*(x) = x$. The results are reported as the average $\pm$ standard deviation across five independent trials. Bold and underlined values denote the best and second-best performance, respectively.} 

\label{tab:results_2D}
\centering
\begin{tabular}{lllcccc}
\toprule
\multirow{2}{*}{\textbf{Class}} &
\multirow{2}{*}{\textbf{Method}} &
\multirow{2}{*}{\textbf{NFE}} &
\textbf{Checkerboard} &
\textbf{Moons} &
\textbf{Circles} &
\textbf{Swissroll} \\
\cmidrule(lr){4-7}
& & &
$W_2(10^{-2}\!\downarrow)$ &
$W_2(10^{-2}\!\downarrow)$ &
$W_2(10^{-2}\!\downarrow)$ &
$W_2(10^{-2}\!\downarrow)$ \\
\midrule
\multirow{3}{*}{\textbf{Static}} 
& PCP-Map \cite{wang2025efficient}
& 1
& $\underline{6.27 \pm 0.81}$
& $8.44 \pm 1.09$
& $6.19 \pm 0.43$
& $5.35 \pm 0.93$ \\
& COTM
& 1
& $7.81 \pm 0.66$
& $12.00 \pm 0.33$
& $6.34 \pm 0.29$
& $8.52 \pm 0.32$ \\
& CUOTM (Ours)
& 1
& $6.53 \pm 0.40$
& $\underline{6.52 \pm 0.86}$
& $\mathbf{4.46 \pm 0.10}$
& $\underline{4.92 \pm 0.22}$ \\
\midrule
\multirow{3}{*}{\textbf{Dynamic}} 
& COT-Flow \cite{wang2025efficient}
& 16
& $8.20 \pm 0.49$
& $18.49 \pm 2.22$
& $10.04 \pm 1.69$
& $6.47 \pm 0.69$ \\
& FM \cite{lipman2022flow}
& 596
& $8.81 \pm 0.58$
& $15.55 \pm 0.77$
& $7.03 \pm 0.17$
& $8.18 \pm 0.34$ \\
& COT-FM \cite{kerrigan2024dynamic}
& 596
& $\mathbf{4.69 \pm 1.00}$
& $\mathbf{6.50 \pm 1.41}$
& $\underline{5.56 \pm 0.43}$
& $\mathbf{4.64 \pm 1.26}$ \\
\bottomrule
\end{tabular}
\end{table*}

\begin{figure*}[t]
    \centering

    \begin{minipage}{0.15\linewidth} \centering \small \hspace*{14pt}Target \end{minipage} \hfill
    \begin{minipage}{0.15\linewidth} \centering \small \hspace*{14pt}COTM \end{minipage} \hfill
    \begin{minipage}{0.15\linewidth} \centering \small \hspace*{14pt}CUOTM \end{minipage} \hfill
    \begin{minipage}{0.21\linewidth} \centering \small \hspace*{11pt}COTM (KDE) \end{minipage} \hfill
    \begin{minipage}{0.21\linewidth} \centering \small \hspace*{11pt}CUOTM (KDE) \end{minipage}
    
    \vspace{0pt}

    \begin{minipage}[t]{0.15\linewidth}
        \includegraphics[width=\textwidth, align=c]{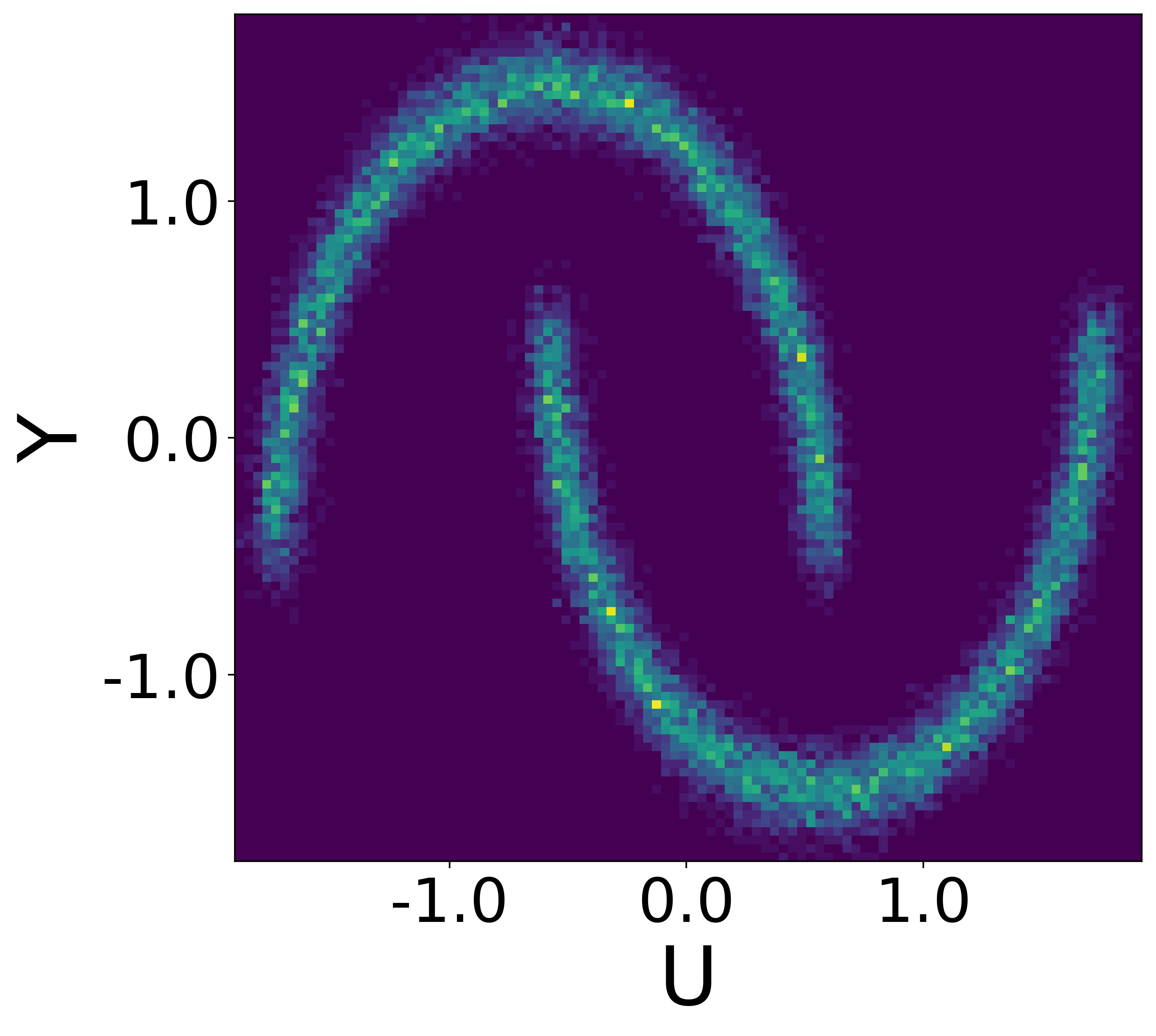}
    \end{minipage} \hfill
    \begin{minipage}[t]{0.15\linewidth}
        \includegraphics[width=\textwidth, align=c]{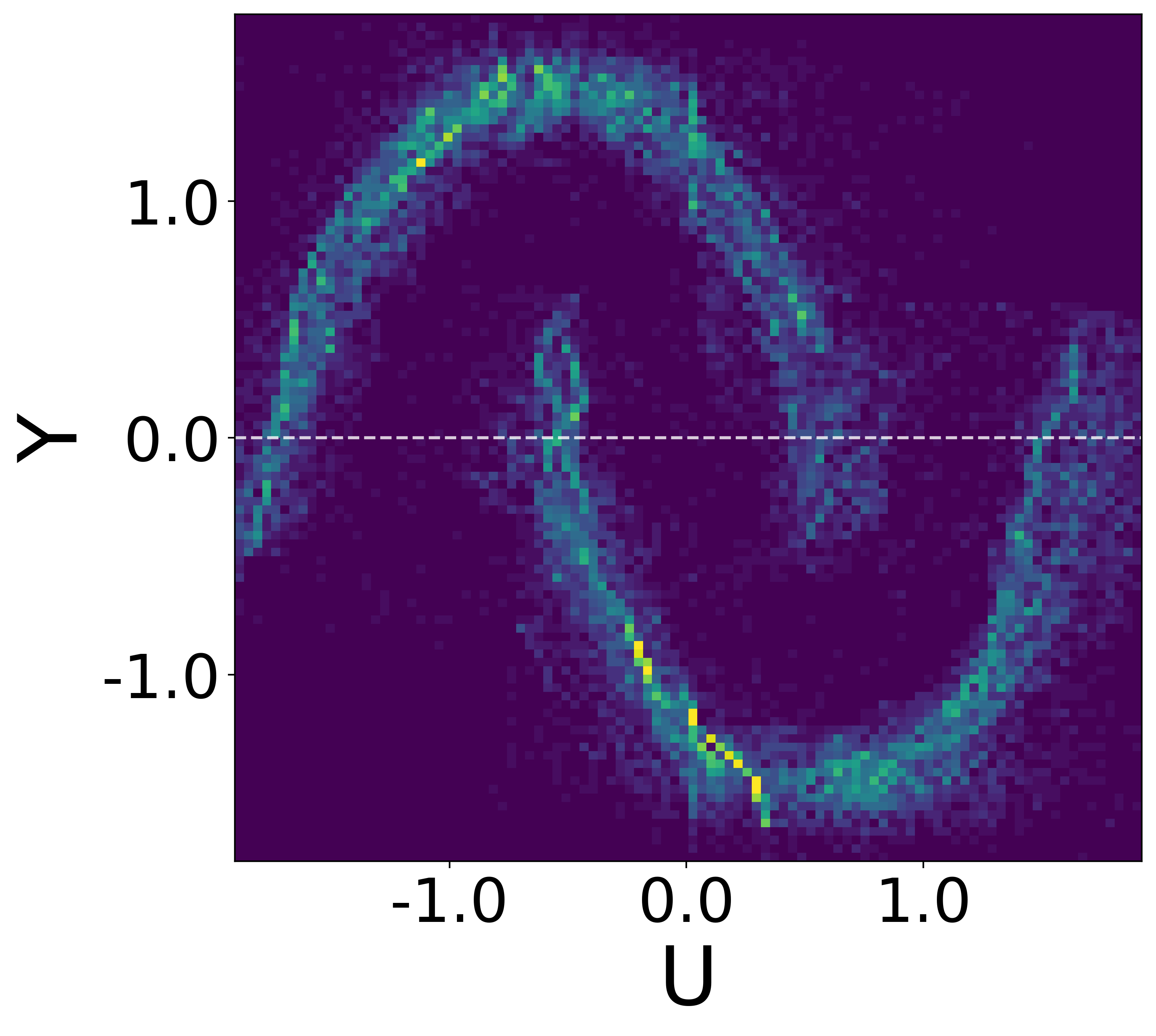}
    \end{minipage} \hfill
    \begin{minipage}[t]{0.15\linewidth}
        \includegraphics[width=\textwidth, align=c]{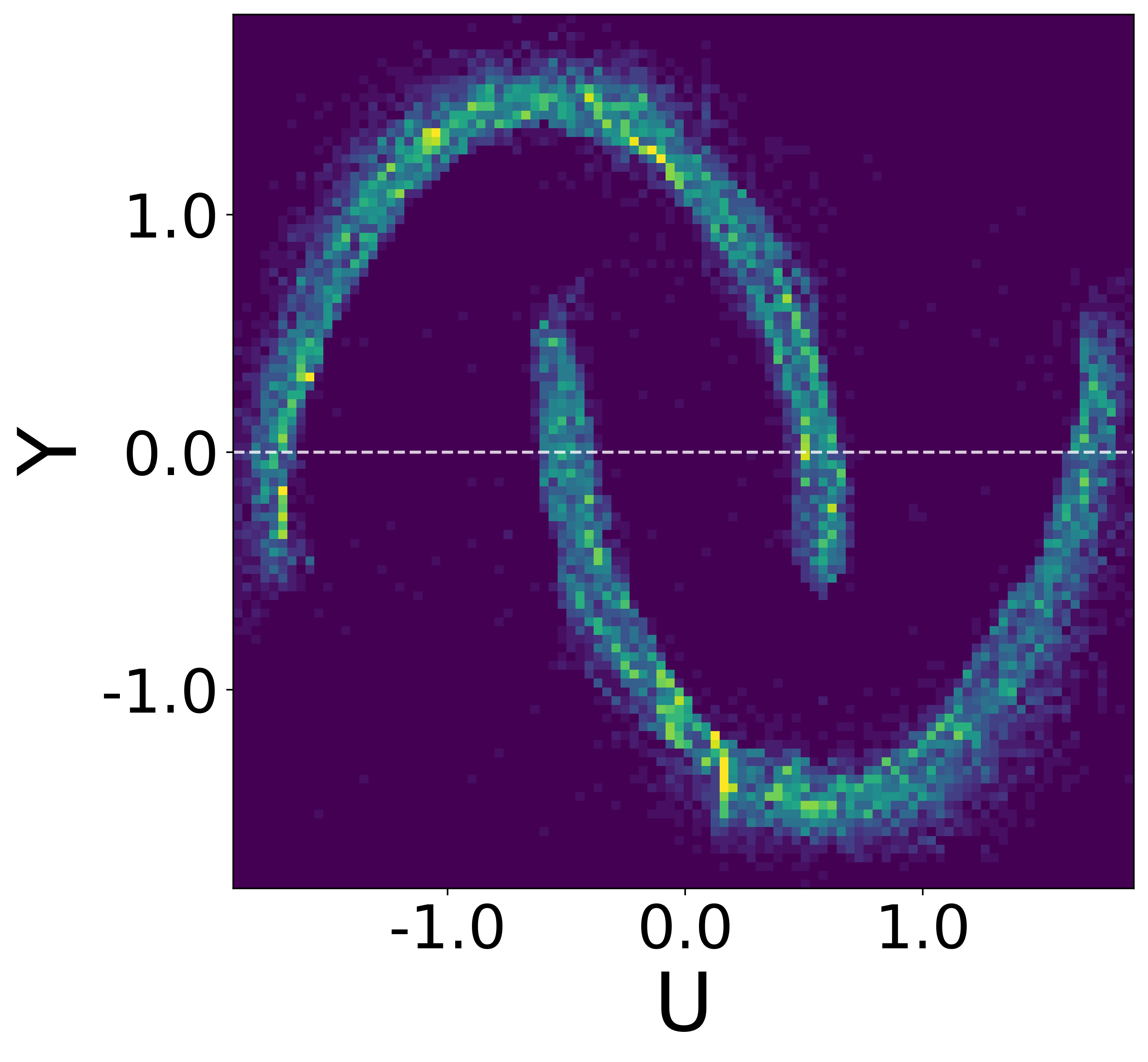}
    \end{minipage} \hfill
    \begin{minipage}[t]{0.21\linewidth}
        \includegraphics[width=\textwidth, align=c]{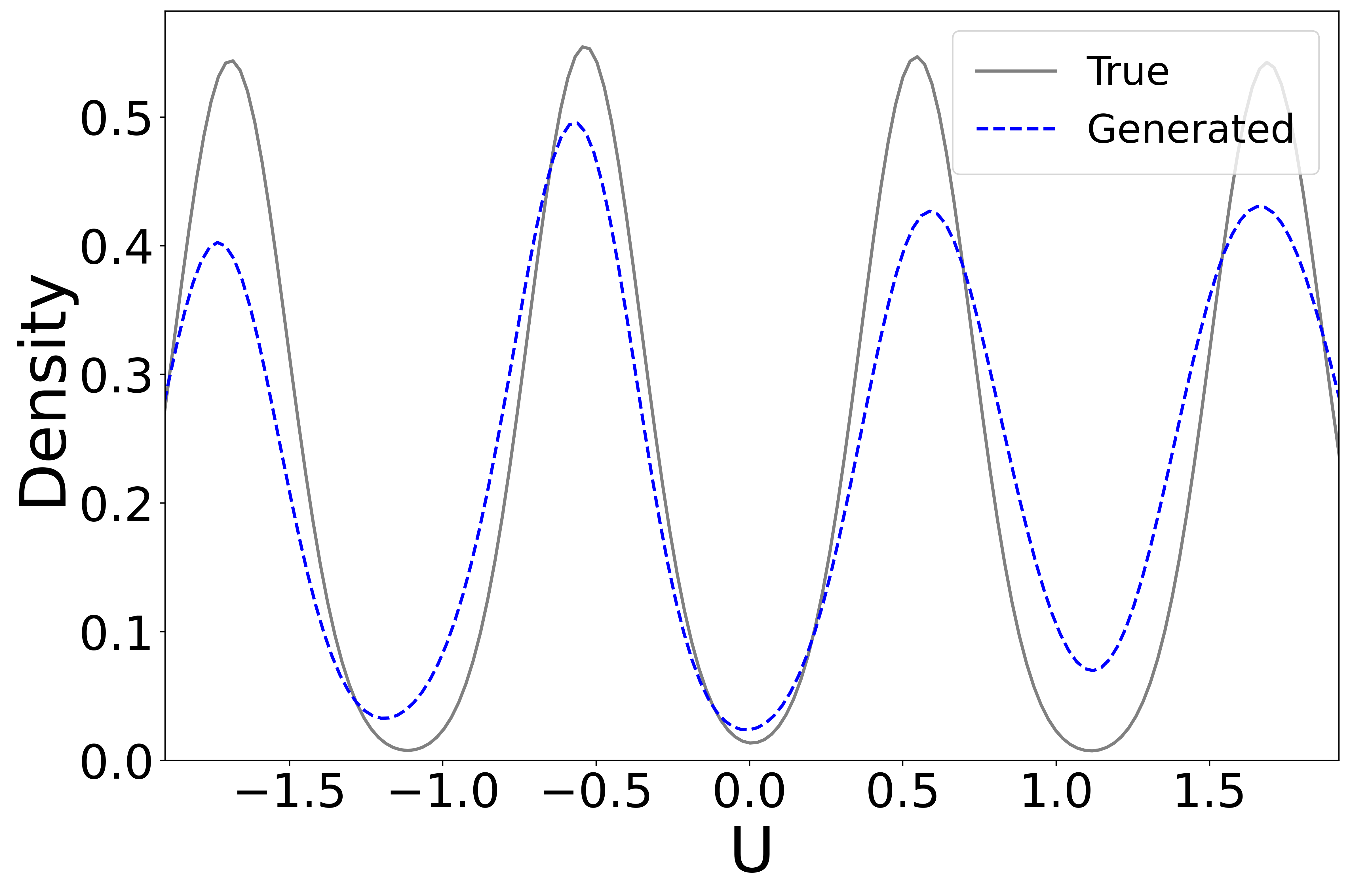}
    \end{minipage} \hfill
    \begin{minipage}[t]{0.21\linewidth}
        \includegraphics[width=\textwidth, align=c]{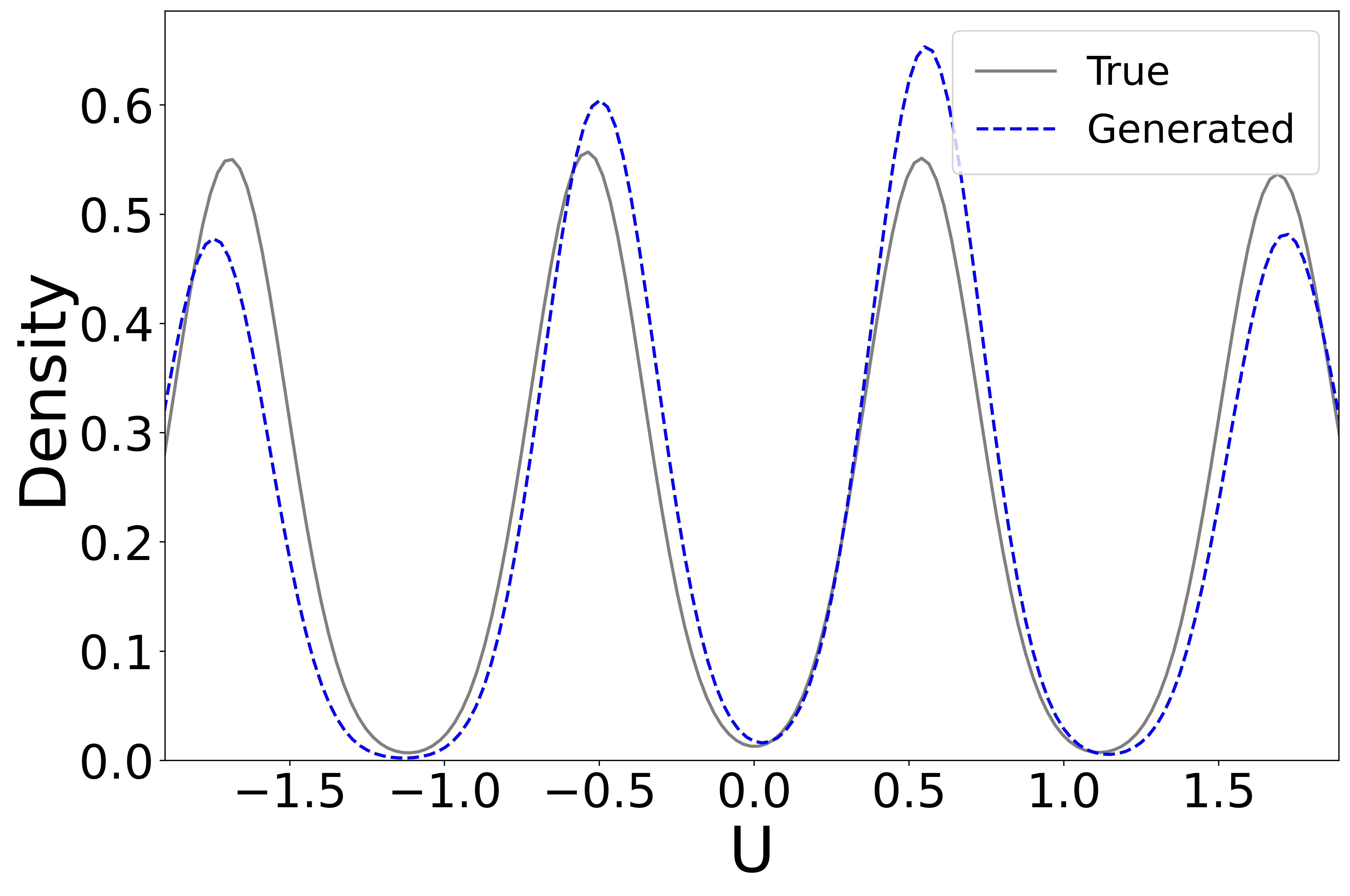}
    \end{minipage}
    
    \medskip
    
    \begin{minipage}[t]{0.15\linewidth}
        \includegraphics[width=\textwidth, align=c]{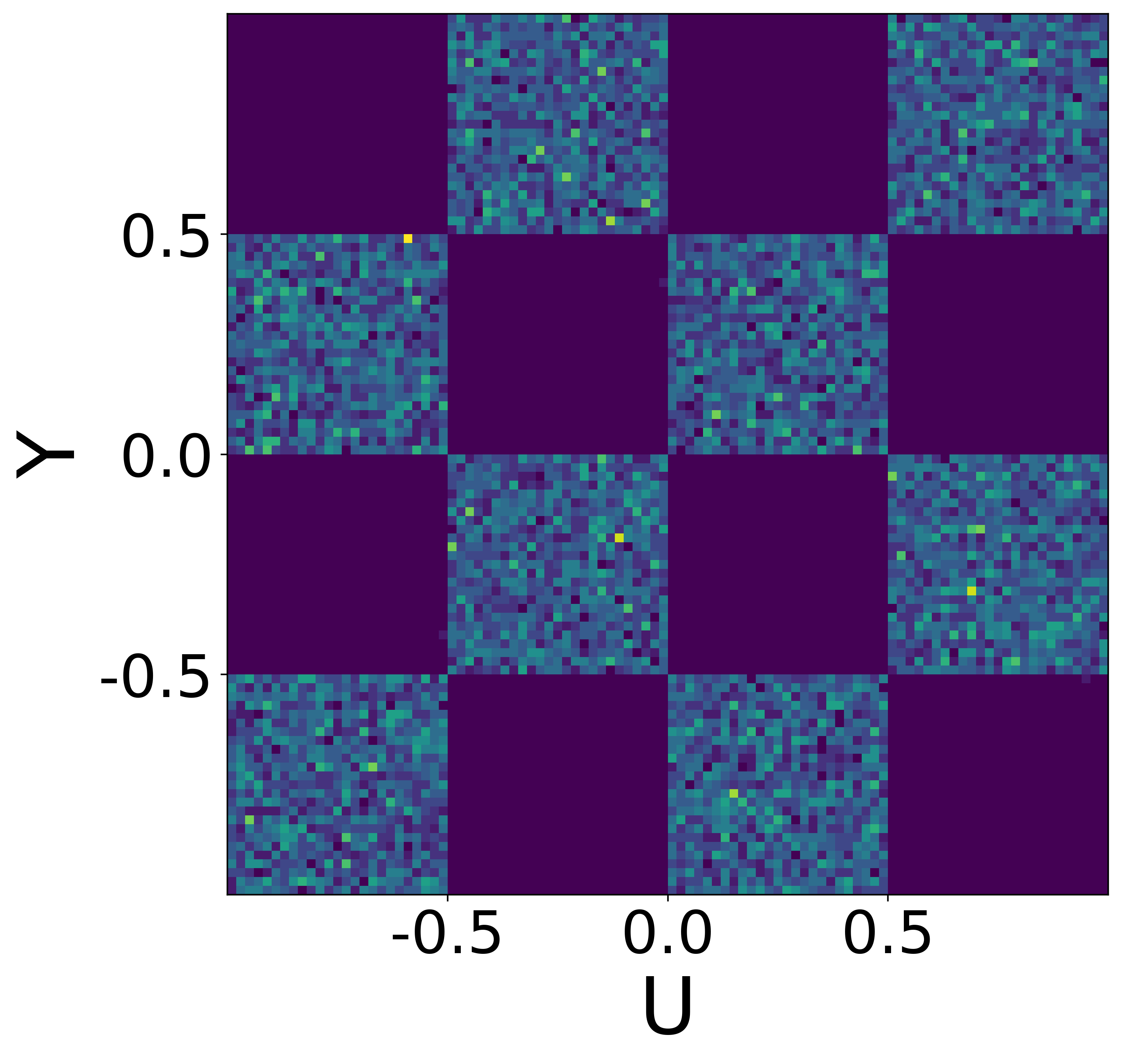}
    \end{minipage} \hfill
    \begin{minipage}[t]{0.15\linewidth}
        \includegraphics[width=\textwidth, align=c]{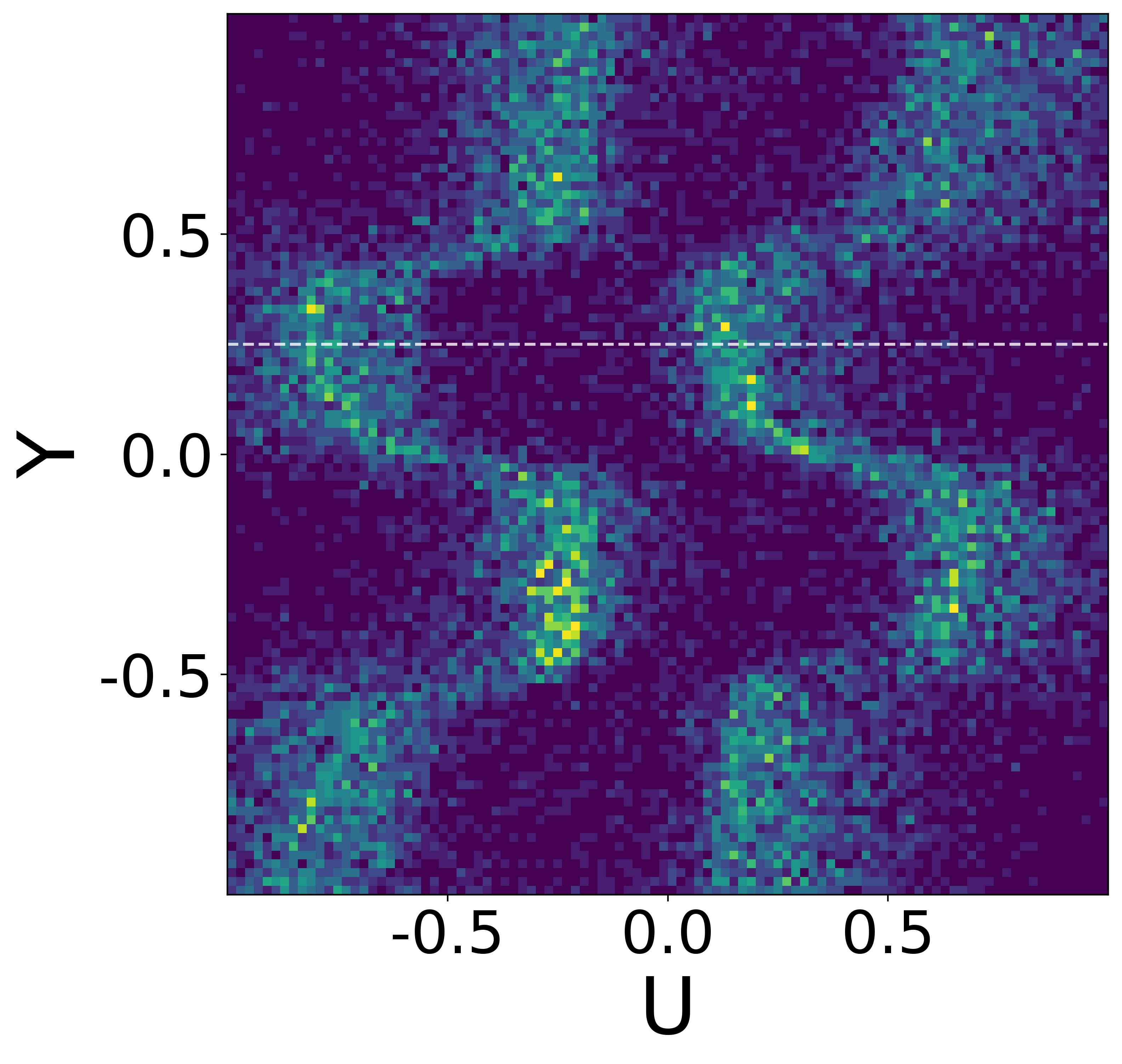}
    \end{minipage} \hfill
    \begin{minipage}[t]{0.15\linewidth}
        \includegraphics[width=\textwidth, align=c]{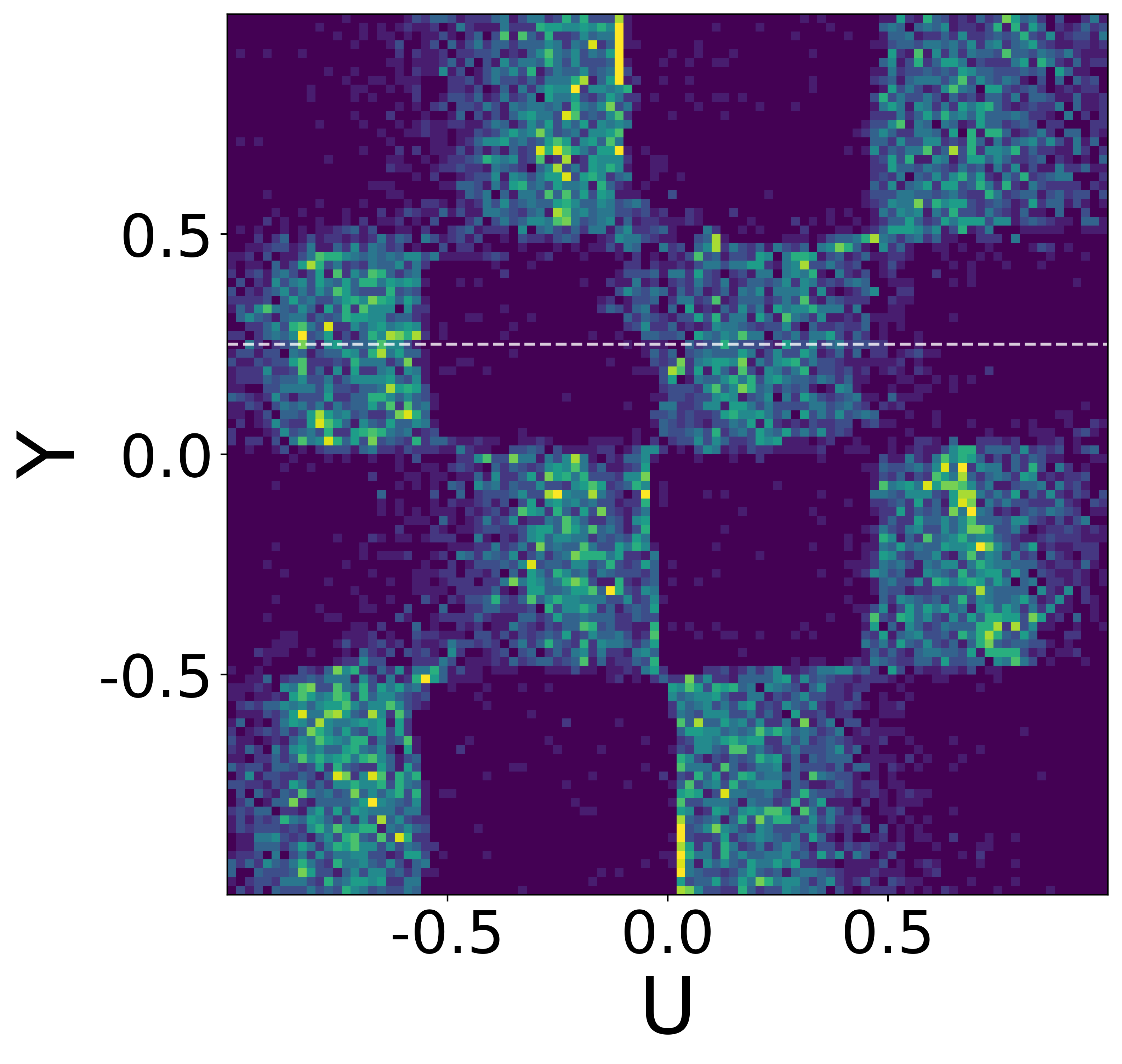}
    \end{minipage} \hfill
    \begin{minipage}[t]{0.21\linewidth} 
        \includegraphics[width=\textwidth, align=c]{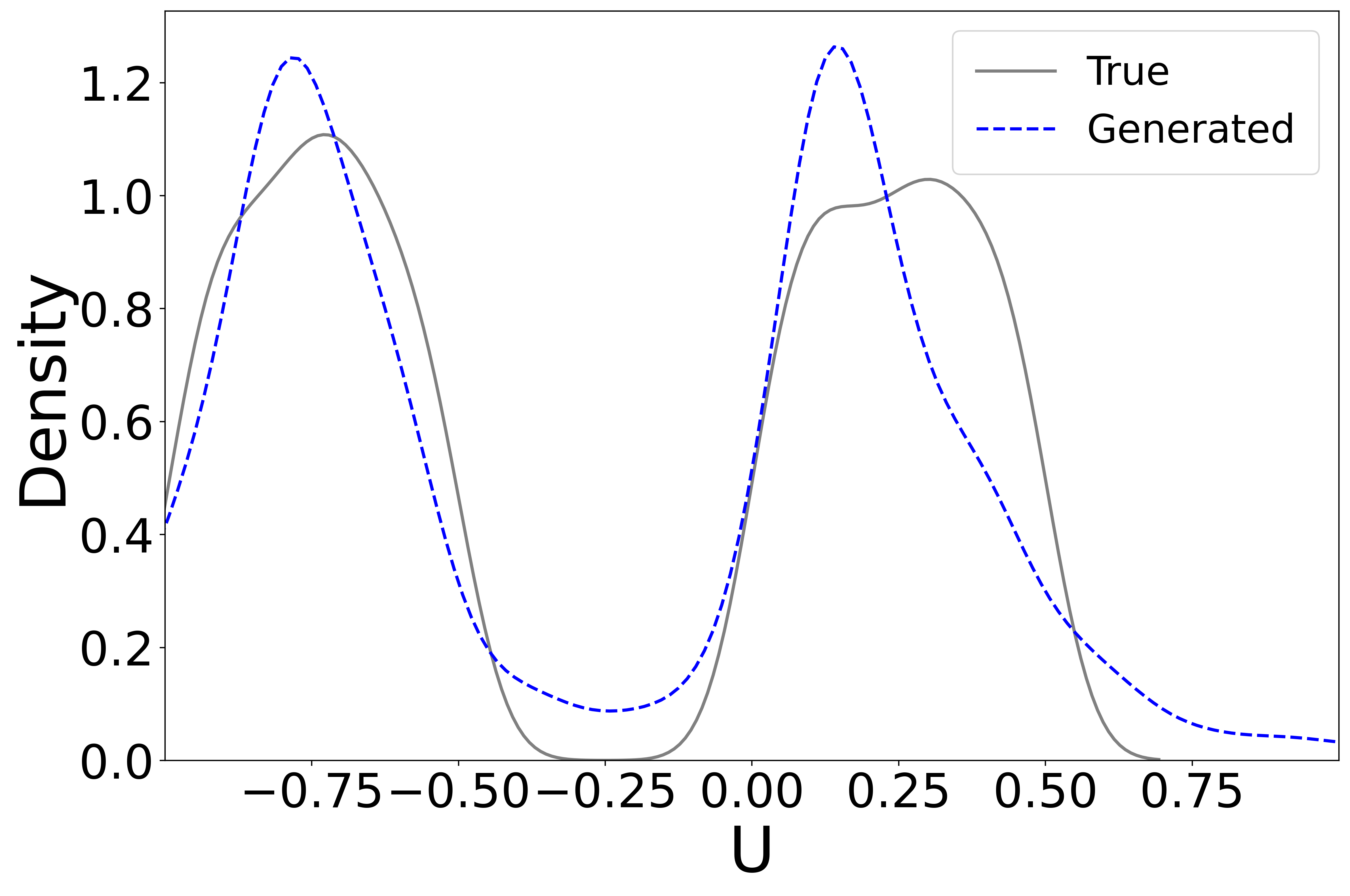}
    \end{minipage} \hfill
    \begin{minipage}[t]{0.21\linewidth} 
        \includegraphics[width=\textwidth, align=c]{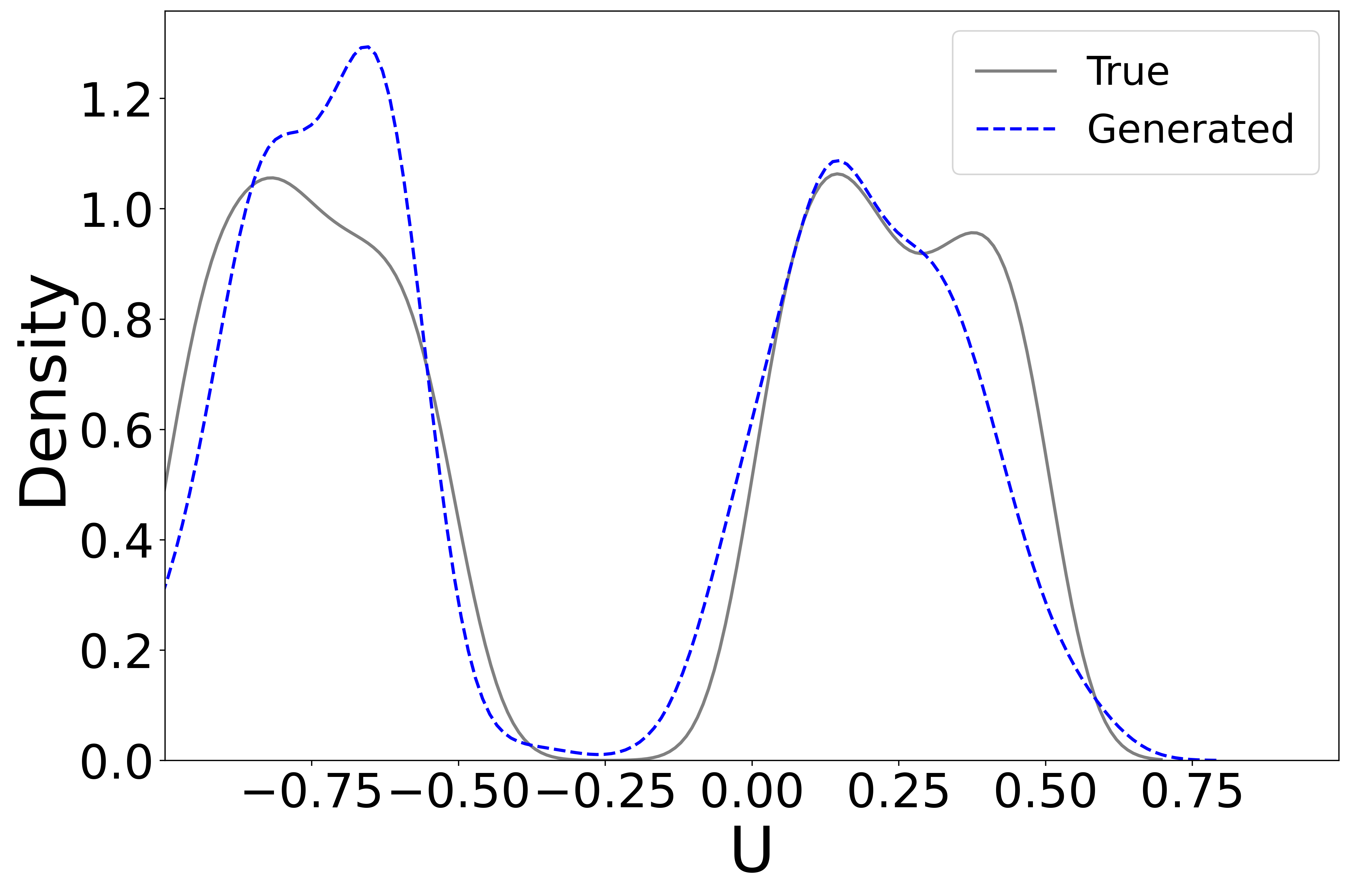}
    \end{minipage}
    
    \medskip

    \caption{\textbf{Qualitative results on 2D synthetic datasets.} From left to right: target distribution, samples from COTM, samples from CUOTM (Ours), and the respective KDE density visualizations. The top and bottom rows show results for the \textit{Moons} and \textit{Checkerboard} datasets, respectively. CUOTM demonstrates superior density recovery and sharper boundary alignment compared to the standard COTM baseline.
    }
    \label{fig:2d_synthetic_results}
\end{figure*}

\subsection{Unbalanced Optimal Transport}

Standard OT problem is known to be highly sensitive to outliers due to the hard constraint on distribution matching \cite{balaji2020robust, uotm}. To overcome this limitation, \cite{chizat2018unbalanced, liero2018optimal} introduces the Unbalanced Optimal Transport (UOT) problem, which relaxes the hard marginal constraints by introducing divergence penalties.
 
In the unconditional setting, \cite{uotm} introduced the Unbalanced Optimal Transport Maps (UOTM) framework, which effectively resolves the drawbacks of classical OT by leveraging neural networks to learn the transport maps.
By employing a semi-dual formulation of UOT \cite{pmlr-v202-vacher23a}, UOTM avoids the instability of training and provides a robust objective that handles outliers \cite{uotm2}. Despite these advantages in the unconditional setting, the application of UOT to conditional settings remains unexplored. To the best of our knowledge, there is currently a lack of a theoretical formulation that integrates conditional structures with unbalanced transport relaxations. Our work bridges this gap by presenting the first framework for conditional UOT (Eq. \ref{eq:condUOT}), combining a triangular COT \cite{hosseini2025conditional} parameterization with a semi-dual UOT objective to enable stable and robust conditional map modeling.

\section{Experiments}
\label{sec:experiments}
In this section, we evaluate our CUOTM across various conditional generation tasks to demonstrate its performance in distribution matching, generation efficiency, and robustness to outliers. Our experimental evaluation is structured as follows:

\begin{itemize}

    \item In Section \ref{sec:gener_exp}, we evaluate our method on 2D synthetic and CIFAR-10 datasets to assess the accuracy and efficiency on conditional generative modeling.
    \item In Section \ref{sec:outlier}, we investigate outlier robustness, showing how relaxing marginal constraints helps mitigate sensitivity to noisy data.
    \item In Section \ref{sec:ablation}, we conduct ablation studies on the cost intensity hyperparameter $\tau$ and the divergence types $D_{\psi}$.
\end{itemize}
We compare our method with existing conditional OT baselines, including dynamic models (\textit{COT-FM} \cite{kerrigan2024dynamic} and \textit{COT-Flow} \cite{wang2025efficient}) and static models (\textit{PCP-Map} \cite{wang2025efficient}). Additionally, we include \textit{COTM}, the standard COT variant of our framework, as a static baseline.

To be more specific, COTM represents a special case of our framework where $\Psi_{1}^\ast(x) = x$ and $\Psi_{2}^\ast(x) = x$ in Eq. \ref{eq:cuot-Jphi}.
This corresponds to setting the Csiszár divergences to zero when measures match exactly and infinity otherwise. In this setting, the CUOT problem (Eq. \ref{eq:condUOT}) reduces to the standard Conditional Kantorovich problem (Eq. \ref{eq:cond_kantorovich}). We validate CUOTM on both 2D synthetic datasets \cite{pedregosa2011scikit} and the CIFAR-10 image dataset \cite{cifar10}.

\subsection{Conditional Generative Modeling}\label{sec:gener_exp}

\paragraph{2D synthetic data}

We evaluate the effectiveness of CUOTM on 2D synthetic datasets \cite{pedregosa2011scikit} as a low-dimensional conditional simulation, following the experimental setup established by \cite{kerrigan2024dynamic}. Specifically, the target data distribution $\nu \in \mathbb{P}(\mathcal{Y} \times \mathcal{U})$ is a 2D synthetic distribution supported on $\mathbb{R}^{2}$. We interpret the first variable as the conditioning variable and the second variable as the data variable, i.e., $\mathcal Y = \mathcal U =\mathbb{R}$. To satisfy our $y$-marginal preservation assumption, the source distribution $\eta$ is defined as the product measure between the $y$-marginal distribution of $\nu$ and the Gaussian prior, i.e., $\eta =\pi^{\mathcal Y}_{\#}\nu \otimes \mathcal{N}(0, 1)$. We use Wasserstein-2 ($\mathcal{W}_2$) distance between the generated distribution $T_{\#}\mu$ and the target distribution $\nu$ as the evaluation metric.

Table \ref{tab:results_2D} present the results on four synthetic datasets: \textit{Checkerboard}, \textit{Moons}, \textit{Circles}, and \textit{Swissroll}. CUOTM achieves superior performance in matching target distributions compared to the static COT baseline \cite{wang2025efficient}. While dynamic models \cite{wang2025efficient,lipman2022flow,kerrigan2024dynamic}, such as COT-FM \cite{kerrigan2024dynamic}, achieve competitive generative performance, they require multiple Number of Function Evaluations (NFE) for sample generation. For example, COT-Flow \cite{wang2025efficient} requires 16 NFE, and COT-FM \cite{kerrigan2024dynamic} and FM \cite{lipman2022flow} require 596 NFEs for 2D synthetic datasets.
In this regard, our CUOTM achieves at least second-best performance while serving as a computationally efficient one-step generator (NFE 1). In particular, our model yields performance comparable to COT-FM on the Moons and Swissroll datasets and even achieves a lower $\mathcal{W}_2$ score on the Circles dataset.
Furthermore, our model effectively recovers the complex manifold of the Checkerboard distribution, including its characteristic discontinuous grid pattern (see Figure \ref{fig:2d_synthetic_results}). This is also reflected in the KDE plots (second row, rightmost two columns): COTM generates spurious density in regions where the true density is near zero, whereas CUOTM preserves these near-zero-density regions (See Appendix \ref{app:addtional_results} for additional visualizations.).

Interestingly, even in outlier-free 2D synthetic tasks requiring precise distribution matching, CUOTM achieves lower $\mathcal{W}_2$ distances than COTM. This is particularly interesting because CUOTM allows the conditional distribution error for outlier robustness, while COTM assumes the exact conditional generative modeling. This demonstrates that the relaxed marginal constraints do not necessarily hinder precision, even under the outlier-free setting. A similar phenomenon has been observed in a marginal generative modeling case under the unbalanced setting \cite{uotm}.

\begin{figure}[t]
    \centering     \includegraphics[width=0.9\linewidth]{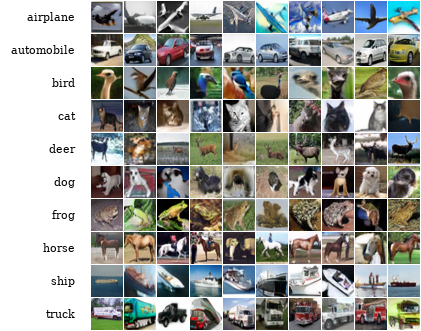}
     \caption{\textbf{Qualitative results for class-conditional CIFAR-10 generation.} Each row corresponds to one of the ten class categories.}
    \label{fig:Exp on cifar10}
\end{figure}

\paragraph{CIFAR-10}
We conducted class-conditional generation on the CIFAR-10 dataset \cite{cifar10} to evaluate our method on image-scale conditional generation (32$\times$32$\times$3). Here, the conditioning variable $y$ corresponds to the discrete class label, while the data variable $u$ lies in the image space $\mathcal{U}=\mathbb{R}^{32\times 32\times 3}$. As in the 2D synthetic experiments, we define the source distribution as $\pi^{\mathcal Y}_{\#}\nu \otimes \mathcal{N}(0, I)$, satisfying the $y$-marginal preservation assumption.

In this experiment, we additionally consider the \textbf{CUOTM+SD}, which incorporates $\alpha$-scheduling \cite{choi2023analyzing} into the standard CUOTM model. Theorem \ref{thm:parameter} (Eq. \ref{eq:div_bound}) implies that decreasing $\tau$ leads to tighter distribution matching. Motivated by this, the $\alpha$-scheduling implicitly down-scales $\tau$ (i.e., $\tau \to \tau / \alpha$) by up-scaling the Csiszár divergences (i.e., $D_{\Psi_{i}} \to \alpha D_{\Psi_{i}}$). This is equivalent to replacing $(\Psi_i)^{\ast}$ in the CUOTM objective as follows:
\begin{equation}
\alpha D_{\Psi_i} = D_{\alpha\Psi_i} \quad (\alpha\Psi_i)^{\ast}(x) = \alpha\,\Psi_i^{\ast}\!\left(\frac{x}{\alpha}\right)
\end{equation}
By adjusting $\alpha$, we control the relative weight of the divergence penalties against the transport cost, and consequently the degree of marginal relaxation in the CUOT objective (Eq. \ref{eq:condUOT}). This mechanism allows $\alpha$-scheduling to dynamically tighten the distribution matching throughout the training process (See Appendix \ref{app:tightness_alpha} for details).

Performance is measured by the Fréchet Inception Distance (FID) \cite{fid} and Inception Score (IS) \cite{salimans2016improved}. FID and IS evaluate the fidelity and diversity of generated images by leveraging representations from a pre-trained Inception network.
To assess class-specific quality, we report the Intra-class FID (IFID) \cite{miyato2018cgans}, calculated as the average of FID scores across all classes. While FID measures the discrepancy between marginal distributions, IFID evaluates discrepancies between class-conditional distributions.

As shown in Table \ref{tab:cifar10_results}, CUOTM and CUOTM+SD achieve competitive performance on FID, IS, and IFID scores compared to other COT baselines, including dynamic and static models. This demonstrates that CUOTM offers scalability to work on image-scale datasets. In particular, CUOTM+SD with only 1 NFE outperforms OT Bayesian Flow \cite{chemseddine2025conditional} using 100 NFE, demonstrating superior distribution matching and sampling efficiency simultaneously. Similar to 2D synthetic results, our standard OT-variant (COTM) failed to scale to image datasets. We consider that this happens because of the following key factor. The hard marginal constraints of COTM make the model sensitive to outliers, as the transport plan is forced to match the entire source mass to the target distribution. This can be one reason that previous COT baselines are restricted to synthetic datasets \cite{wang2025efficient, kerrigan2024dynamic}. Implementation details are provided in Appendix \ref{CIFAR-10 implementations}.

\begin{table}[t]
    \centering
    \caption{\textbf{Quantitative results on class-conditional CIFAR-10 comparing conditional OT-based baselines.} CUOTM+SD denotes our framework implemented with the $\alpha$-scheduling \cite{choi2023analyzing}. 
    \textdagger\ indicates the results reproduced by ourselves.} 
    
    \label{tab:cifar10_results}
    
    \begin{tabular}{lcccc}
        \toprule
        \textbf{Method} & \textbf{NFE} & \textbf{FID ($\downarrow$)} & \textbf{IFID ($\downarrow$)} & \textbf{IS ($\uparrow$)}  \\
        \midrule
        OT Bayesian Flow \cite{chemseddine2025conditional} & 100 & 4.10 & 14.98$^\dagger$ & 8.81$^\dagger$ \\
        COTM$^\dagger$ & 1 & 33.04 & 65.66 & 6.58 \\
        CUOTM$^\dagger$ (Ours) & 1 & 5.30 & 15.55 & 8.79 \\
        CUOTM+SD$^\dagger$ (Ours) & 1 & \textbf{3.71} & \textbf{13.44} & \textbf{8.83} \\
        \bottomrule
    \end{tabular}
    
\end{table}

\begin{figure*}[t]
    \centering

    \subfloat{
        \includegraphics[width=0.3\linewidth]{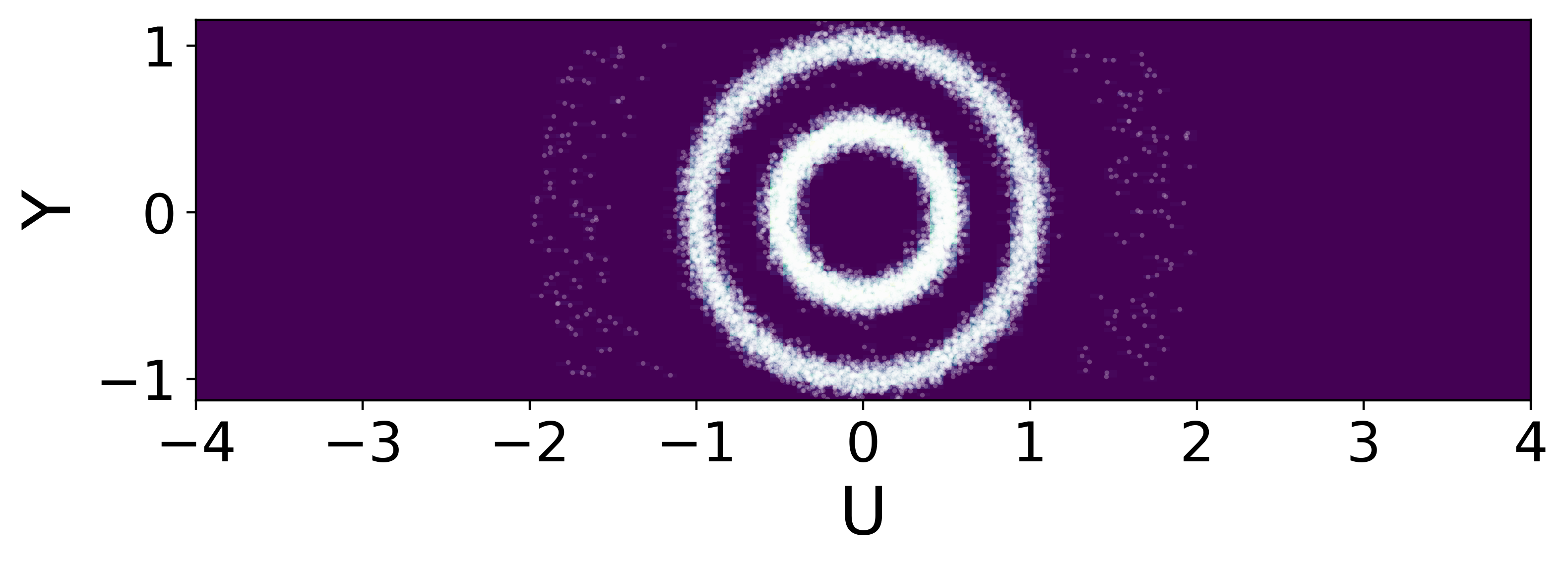}
    }
    \hfill
    \subfloat{
        \includegraphics[width=0.3\linewidth]{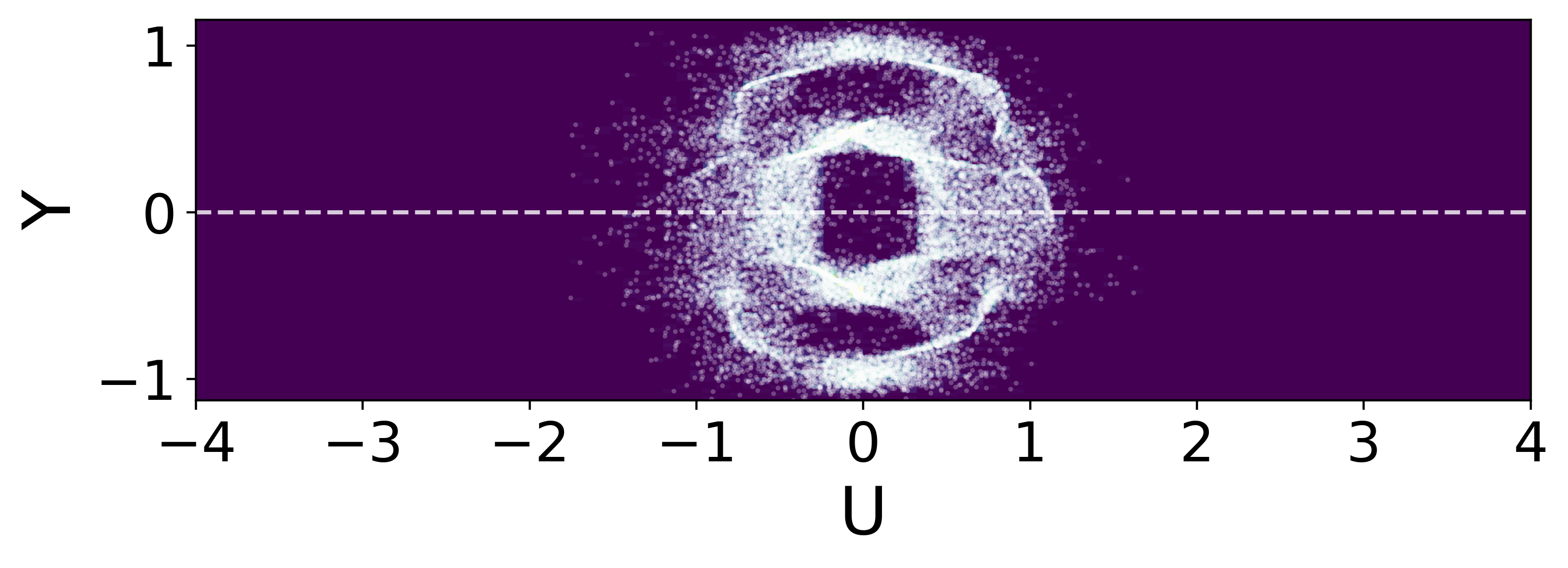}
    }
    \hfill
    \subfloat{
        \includegraphics[width=0.3\linewidth]{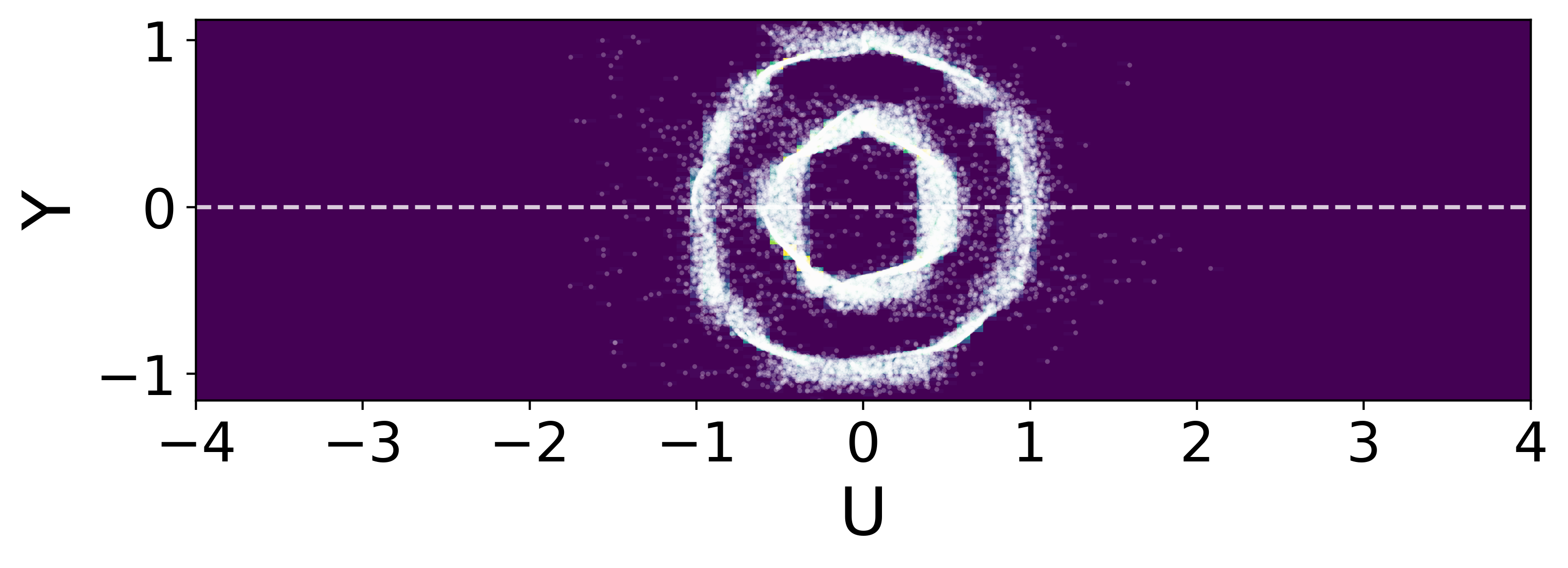}
    }
    
    \medskip

    \setcounter{subfigure}{0}
    \subfloat[True Distribution $\nu$]{
        \includegraphics[width=0.3\linewidth]{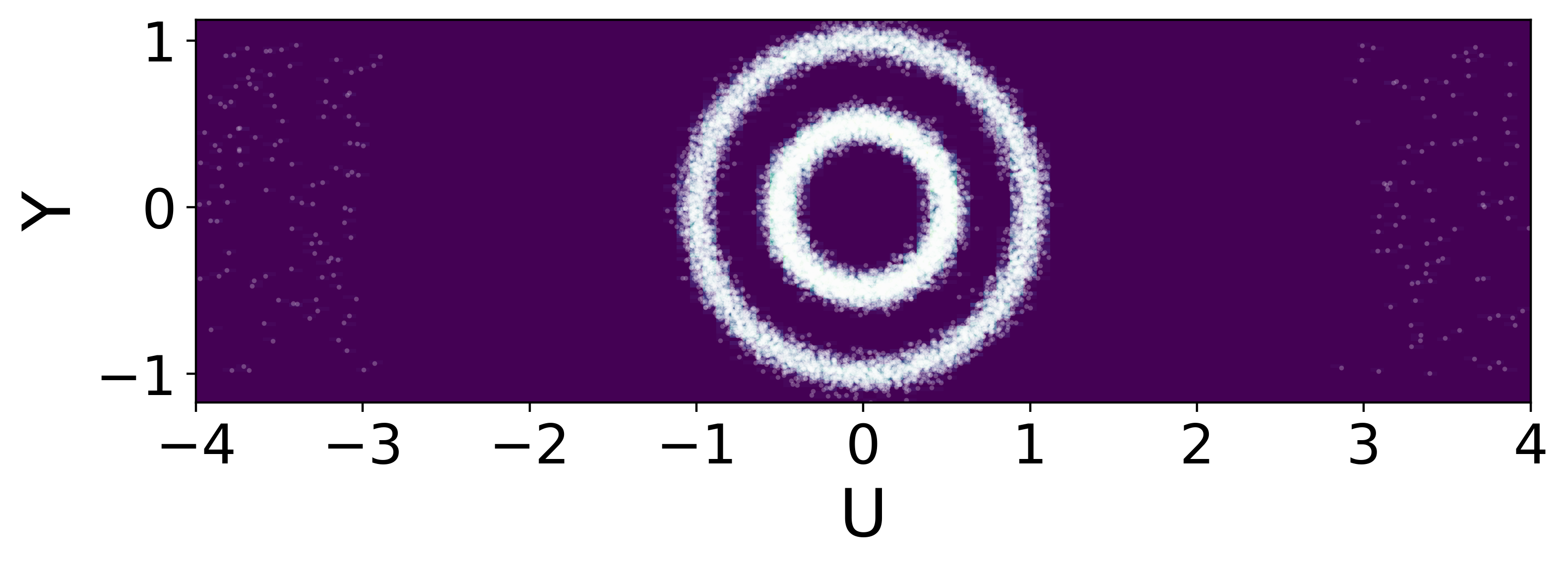}
    }
    \hfill
    \subfloat[Generated distribution (COTM)]{
        \includegraphics[width=0.3\linewidth]{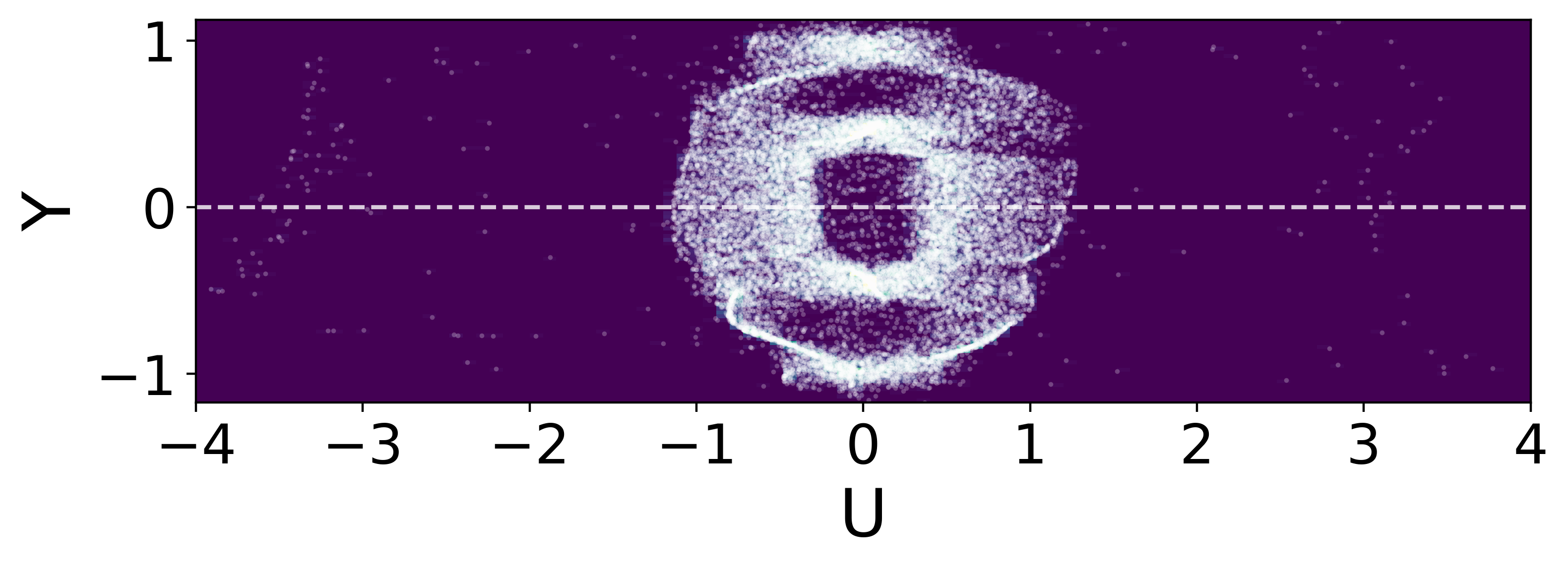}
    }
    \hfill
    \subfloat[Generated distribution (CUOTM)]{
        \includegraphics[width=0.3\linewidth]{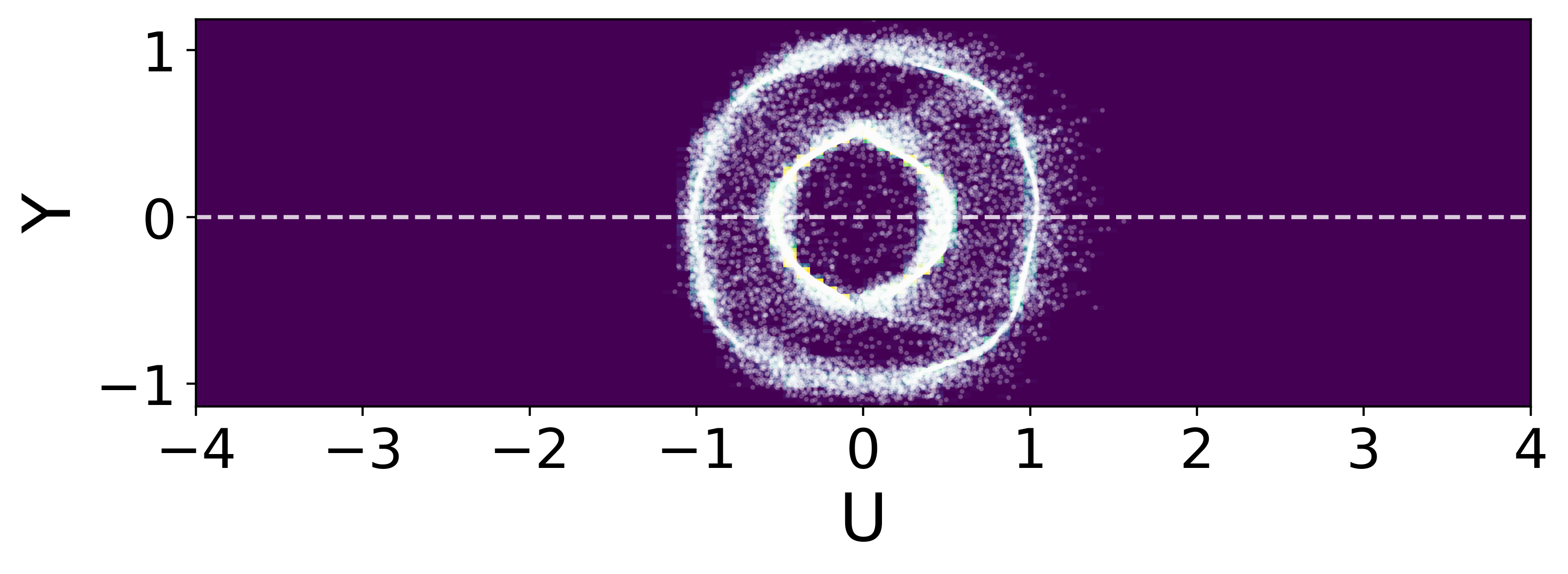}
    }
    
    \caption{
    \textbf{Qualitative comparison of outlier robustness on the Circles dataset.} The in-distribution modes (two center circles) are perturbed with $1\%$ outliers located within annular regions of radius $r \in [1.5, 2]$ (top row) and $r \in [3, 4]$ (bottom row). While COTM (middle) exhibits significant distortion by attempting to match the outlier, CUOTM (right) effectively ignores the noise by prioritizing high-density regions through marginal relaxation, thereby accurately recovering the majority distribution.
    }
    \label{fig:outlier}
\end{figure*}

\subsection{Outlier robustness}\label{sec:outlier}

\begin{table}[t]
    \centering
    \caption{\textbf{Quantitative comparison of outlier robustness on the Circles dataset.} The table reports the Wasserstein-2 distance ($\mathcal{W}_2 \times 10^{-3}$) between the majority (clean) distribution and the generated samples under $1\%$ outlier contamination. The outlier range $[a, b]$ denotes that outliers are distributed within an annular region of radius $r \in [a, b]$ from the origin.}
    \label{tab:outlier}
    \begin{tabular}{ccc}
        \toprule
        \multirow{2}{*}{\textbf{Outlier Range}} & \multicolumn{2}{c}{$\mathcal W_2(\times 10^{-3}, \downarrow)$} \\
        \cmidrule(r){2-3}
         & CUOTM & COTM \\
        \midrule
        {[4,5]}    & \textbf{0.047} & 0.205 \\ 
        {[3,4]}    & \textbf{0.062} & 0.192 \\
        {[2,3]}    & \textbf{0.070} & 0.124 \\
        {[1.5,2]}  & \textbf{0.055} & 0.084 \\ 
        \bottomrule
    \end{tabular}
\end{table}

In this section, we demonstrate that our CUOTM effectively mitigates the sensitivity to outliers, which is a limitation of standard OT-based model, by relaxing marginal constraints. This robustness is particularly crucial in conditional generation scenarios. 

Because the dataset is partitioned across various conditions, each conditional distribution is characterized by a reduced sample size,  making the model more susceptible to outlier-induced distortions than the unconditional setting. Such challenges are ubiquitous in practical, real-world applications \cite{10.1145/2983323.2983660,hu2024anomalydiffusion}.

We verify this robustness by comparing CUOTM against the COTM model. We conducted a controlled experiment using 2D sythetic \textit{Circles} dataset, perturbed with $1\%$ outliers. The outlier intensity is controlled by the outlier position, following \cite{mukherjee2021outlier}. Note that even when the outlier proportion remains constant, the Conditional Kantorovich OT objective (Eq. \ref{eq:cond_kantorovich}) increases sharply as the outliers are positioned further from the majority samples.
To quantify this effect, we define the outlier range by the radial distance from the origin, placing outliers within an annular region of radius $r \in [a, b]$.

Table \ref{tab:outlier} presents the quantitative results by measuring the $\mathcal{W}_2$ distance between the majority (non-outlier) mode distribution and the generated distribution.
Under a strong outlier regime (e.g., $r \in [4,5]$), CUOTM effectively excludes the distant outlier modes, achieving a much lower $W_2$ distance ($0.047 \times 10^{-3}$) than COTM ($0.205 \times 10^{-3}$).
This performance gap decreases as the outliers get closer to the majority distribution manifold, reflecting the theoretical trade-off between transport costs and marginal penalties in Eq. \ref{eq:condUOT}.

As illustrated in the bottom row of Figure \ref{fig:outlier}, when outliers are located at $r \in [3, 4]$, the OT model (COTM) collapses and fails to recover the target distribution even in the presence of merely $1\%$ outliers. In contrast, our CUOTM successfully matches the in-distribution data, effectively avoiding outlier generation. 
This happens because the cost of transporting mass to distant outliers outweighs the penalty incurred by shifting the marginal distributions. Consequently, the relaxation of marginal constraints allows the model to prioritize high-density regions, thereby ensuring robustness against outliers. Moreover, as shown in Figure \ref{fig:outlier}, if the distance is $[1.5,2]$, CUOTM achieves better distribution matching even when outliers are located close to the inliers, whereas COTM fails to distinguish them.

\subsection{Ablation Study} \label{sec:ablation}

This section presents an ablation study on the key components of CUOTM. 
We specifically focus on the cost intensity hyperparameter $\tau$ and the choice of entropy functions $\Psi$, as these are the primary factors governing the training dynamics and distribution-matching behavior of our framework (Eq. \ref{eq:condUOT}).

\begin{figure}
    \centering
    \includegraphics[width=0.7\linewidth]{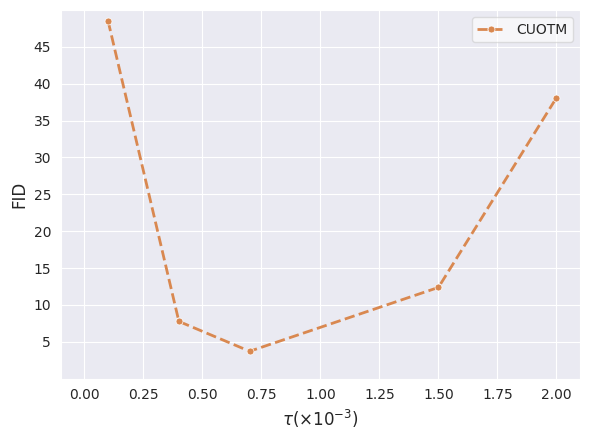}
    \caption{Ablation Study on the cost intensity parameter $\tau$}
    \label{fig:ablationtau}
\end{figure}

\begin{table}[t]
    \centering
    \caption{Ablation Study on Csisz\`{a}r Divergence $D_{\Psi_i}(\cdot|\cdot)$.}
    \label{tab:divergence}
    \begin{tabular}{cc} 
        \toprule
        $(\Psi_1, \Psi_2)$ & FID($\downarrow$) \\
        \midrule
        (KL, KL) & \textbf{3.71} \\
        ($\chi^2, \chi^2$) & 5.05 \\
        (Softplus, Softplus) & 3.79 \\
        \bottomrule
    \end{tabular}
\end{table}

\subsubsection{Cost Intensity}

As formulated in Eq. \ref{eq:condUOT}, the hyperparameter $\tau$ in the cost function $c(z, v; y, u) = \tau \| v-u \|_{2}^{2}$ controls the relative weight between the transport cost and the soft marginal matching penalties. As the theorem \ref{thm:parameter} suggests, a smaller $\tau$ results in tighter marginal distribution matching, whereas increasing $\tau$ further relaxes these constraints.
Figure \ref{fig:ablationtau} shows how the FID score of CUOTM changes with different $\tau$ on the CIFAR-10 dataset. We observed that the optimal performance is achieved at $\tau = 0.0007$. Consistent with our theoretical intuition, increasing $\tau$ beyond this threshold results in a degradation of the FID score, as excessive relaxation prevents the model from achieving a precise match with the target distribution. Conversely, performance also declines when $\tau$ is too small (e.g., $10^{-4}$). Following the interpretation in \cite{uotm}, we consider that a moderate $\tau$ provides a necessary regularization effect that prevents mode collapse during the early stages of training.

\subsubsection{Divergence}

We conducted an ablation study on the choice of the entropy functions $\Psi_1$ and $\Psi_2$ for the Csiszár Divergence $D_{\Psi_1}$ and $D_{\Psi_2}$ used in the CUOT problem (Eq. \ref{eq:condUOT}). 

Following the definition of Csis\'ar Divergence, $\Psi_1$ and $\Psi_2$ must be convex, lower semi-continuous, and non-negative \cite{liero2018optimal, chizat2018unbalanced}. 
This framework provides the freedom to choose various entropy functions $\Psi$. 
We tested several candidate functions $\Psi$ satisfying these criteria, including Kullback-Leibler (KL) divergence and $\chi^2$ divergence. Moreover, following \cite{uotm}, we assessed the Softplus case, where $\Psi^{*}(\cdot) = \textrm{Softplus}(\cdot)$(see Appendix \ref{app:candidate_div} for detailed forms of $\Psi$).

As summarized in Table \ref{tab:divergence}, $(\Psi_1, \Psi_2) = (\text{KL, KL})$ showed the best performance with an FID of 3.71. While the Softplus divergence provided competitive results (3.79), the $\chi^2$ divergence resulted in a relatively higher FID (5.05). Consequently, we employed the KL divergence setting.

\section{Conclusion}

In this work, we established the first formulation of the Conditional Unbalanced Optimal Transport (CUOT) problem and introduced CUOTM, a novel conditional generative framework built upon its semi-dual formulation. By incorporating Csiszár divergence penalties, our approach relaxes rigid conditional matching constraints while strictly preserving the conditioning marginal.
We showed both its theoretical guarantees for outlier robustness and its empirical verification across diverse benchmarks. Our empirical results demonstrate that CUOTM provides superior robustness against outliers compared to standard conditional OT. Notably, CUOTM achieves competitive performance in terms of target distribution matching compared to multi-step dynamic conditional OT baselines despite its one-step static generation process.

From a practical standpoint, the robustness of CUOTM offers a reliable solution for real-world applications where data corruption and outliers are often inevitable. 

Despite these strengths, CUOTM also has some limitations. The adversarial training strategy can lead to training instability, a common challenge in GAN literature. Moreover, the performance of the model is highly sensitive to the choice of the cost intensity hyperparameter $\tau$.  Investigating non-adversarial training strategies for the conditional optimal transport maps within the CUOT framework represents a meaningful future research direction.

\bibliographystyle{IEEEtran}
\bibliography{mybib.bib}

\appendices
\section{Definitions} \label{app:def}
This section provides formal definitions for the mathematical components utilized in our work.

\begin{definition}[Csiszár Divergence]
Let $\varphi: \mathbb{R}_+ \to \mathbb{R}$ be a convex, lower semi-continuous function such that $\varphi(1)=0$. Given two measures $\alpha, \beta \in \mathcal{M}_+(\mathcal{X})$, let $\alpha = \alpha_{ac} + \alpha^{\perp}$ be the Radon-Nikodym-Lebesgue decomposition of $\alpha$ with respect to $\beta$, where $\alpha_{ac} \ll \beta$ and $\alpha^{\perp} \perp \beta$. The \textit{Csiszár divergence} $D_{\varphi}$ is defined as:
\begin{equation}
D_{\varphi}(\alpha \| \beta) := \int_{\mathcal{X}} \varphi\left(\frac{\mathrm{d} \alpha_{ac}}{\mathrm{d} \beta}(x)\right) \mathrm{d} \beta(x) 
+ \varphi_{\infty}^{\prime} \int_{\mathcal{X}} \mathrm{d} \alpha^{\perp}(x),
\end{equation}
where $\frac{\mathrm{d}\alpha_{ac}}{\mathrm{d}\beta}$ is the Radon-Nikodym derivative and $\varphi_{\infty}^{\prime} := \lim_{x \to \infty} \frac{\varphi(x)}{x}$ represents the recession constant, which accounts for the singular part of the measure.
\end{definition}

\begin{definition}[Conditional $p$-Wasserstein Distance]
Following the formulation in \cite{kerrigan2024dynamic}, let $\eta, \nu \in \mathcal{P}_p^\mu(\mathcal{Y} \times \mathcal{U})$ for $1 \leq p < \infty$, where $\mathcal{P}_p^\mu(\mathcal{Y} \times \mathcal{U})$ denotes the conditional $p$-Wasserstein space with a fixed $y$-marginal $\mu$:
\begin{equation}
    \mathcal{P}_p^\mu(Y \times U) = \{ \gamma \in \mathcal{P}_p(Y \times U) \mid \pi_{\#}^Y \gamma = \mu \}.
\end{equation}
The \textit{conditional $p$-Wasserstein distance} $\mathcal{W}_p^\mu : \mathcal{P}_p^\mu(\mathcal{Y} \times \mathcal{U}) \times \mathcal{P}_p^\mu(\mathcal{Y} \times \mathcal{U}) \to \mathbb{R}$ is defined as:
\begin{equation}
\begin{aligned}
    \mathcal{W}_p^\mu(\eta, \nu) &= \left( \mathbb{E}_{y \sim \mu} \left[ \mathcal{W}_p^p(\eta(\cdot|y), \nu(\cdot|y)) \right] \right)^{1/p} \\&= \left( \int_{\mathcal{Y}} \mathcal{W}_p^p(\eta(\cdot|y), \nu(\cdot|y)) \, \mathrm{d}\mu(y) \right)^{1/p}
    \label{eq:cond_wasserstein}
\end{aligned}
\end{equation}
where $\mathcal{W}_p$ denotes the standard $p$-Wasserstein distance for probability measures supported on the data space $\mathcal{U}$. 

\end{definition}

\section{Proofs}
 
\subsection{Proof of Theorem \ref{thm:existence} (Existence and uniqueness of Conditional UOT problem)} 

\begin{theorem}[Existence and Uniqueness of the Minimizer in CUOT]
Assume that $\mathcal{Y}, \mathcal{V},$ and $\mathcal{U}$ are compact metric spaces. Let $\mathcal P(\cdot)$ be equipped with the topology of weak convergence and $\eta \in \mathcal{P}(\mathcal{Y} \times \mathcal{V})$ and $\nu \in \mathcal{P}(\mathcal{Y} \times \mathcal{U})$ be given source and target measures such that $\eta_Y = \nu_Y$. Suppose the cost function $c: (\mathcal{Y} \times \mathcal{V}) \times (\mathcal{Y} \times \mathcal{U}) \rightarrow \mathbb{R}$ is lower semi-continuous. 
Then, there exists a minimizer $\pi^\star$ for the CUOT problem (Eq. \ref{eq:condUOT}). Furthermore, if the entropy functions $\Psi_1$ and $\Psi_2$ are strictly convex, the minimizer $\pi^\star$ is unique.
\end{theorem}

\begin{proof}
    We establish the existence of a minimizer $\pi^\star \in \mathcal{P}(\mathcal{Y} \times \mathcal{V} \times \mathcal{Y} \times \mathcal{U})$ by verifying the conditions of Theorem 4.1 in \cite{villani}. Specifically, we demonstrate that the space of probability measures $\mathcal{P}(\mathcal{Y} \times \mathcal{V} \times \mathcal{Y} \times \mathcal{U})$ is compact and that the functional $K_{\mathrm{CUOT}}(\pi)$ is lower semi-continuous (l.s.c) with respect to the topology of weak convergence. The $K_{\mathrm{CUOT}}$ is defined as :
    \begin{equation}
        \label{eq:K_CUOT_def}
        \begin{aligned}
        K_{\mathrm{CUOT}}(\pi)
        :=\;&\int_{\mathcal{Y}\times\mathcal{V}\times\mathcal{Y}\times\mathcal{U}}
        c(y,v,y',u)\, d\pi(y,v,y',u) \\
        &+ \int_{\mathcal{Y}}
        D_{\Psi_1}\!\left(\pi_1(\cdot\mid y)\,\middle\|\,\eta(\cdot\mid y)\right)\, d\eta_Y(y) \\
        &+ \int_{\mathcal{Y}}
        D_{\Psi_2}\!\left(\pi_2(\cdot\mid y')\,\middle\|\,\nu(\cdot\mid y')\right)\, d\nu_Y(y')
\end{aligned}
\end{equation}

    First, since $\mathcal{Y}$, $\mathcal{V}$, and $\mathcal{U}$ are assumed to be compact metric spaces in Theorem \ref{thm:existence}, their product $\mathcal{Y} \times \mathcal{V} \times \mathcal{Y} \times \mathcal{U}$ is also a compact metric space. By Prokhorov's theorem (as referenced in \cite{villani}, Theorem 4.1), the space of probability measures $\mathcal{P}(\mathcal{Y} \times \mathcal{V} \times \mathcal{Y} \times \mathcal{U})$ is compact under the weak topology.

    Next, we address the lower semi-continuity of the objective functional $K_{\mathrm{CUOT}}$. Given that the cost function $c$ is non-negative and lower semi-continuous, the integral functional $\int c \, d\pi$ is l.s.c by Lemma 4.2 in \cite{villani}. Furthermore, the Csiszár $\Psi$-divergence is inherently lower semi-continuous with respect to the weak convergence of measures. Since the property of lower semi-continuity is preserved under functional summation, $K_{\mathrm{CUOT}}$ is lower semi-continuous.

    In conclusion, as $K_{\mathrm{CUOT}}$ is a lower semi-continuous functional defined on a compact space of measures, the existence of a minimizer $\pi^\star$ such that $K_{\mathrm{CUOT}}(\pi^\star) = \inf K_{\mathrm{CUOT}}$ is guaranteed by the Theorem 4.1 in \cite{villani}.

    Furthermore, we establish the uniqueness of the minimizer $\pi^\star$. This follows directly from the strictly convexity of the objective functional $K_{\mathrm{CUOT}}$. First, the integral term $\int c \, d\pi$ is linear with respect to $\pi$ and is, therefore, convex. Under the assumption in Theorem \ref{thm:existence} that the penalty functions $\Psi_1$ and $\Psi_2$ are strictly convex, the corresponding Csiszár divergences are also strictly convex with respect to $\pi$. Since the sum of a convex functional and strictly convex functionals is itself strictly convex, the objective functional $K_{\mathrm{CUOT}}$ is strictly convex. Consequently, the minimizer $\pi^\star$ is unique.
\end{proof}

\subsection{Proof of Theorem \ref{thm:dual} (Derivation of the Semi Dual Formulation of Conditional UOT)}  

\begin{theorem}[Duality for CUOT] 
    For a non-negative cost function $c$, the Conditional Unbalanced Optimal Transport problem has the following dual (Eq. \ref{eq:cuot-dual}) and semi-dual (Eq. \ref{eq:cuot-semidual}) formulation: 
    
    \textbf{Dual Formulation:}
    \begin{multline} 
        \sup_{\phi + \varphi \le c} \left[ \int_{\mathcal{Y}\times\mathcal{V}} -\Psi_1^* \bigl( -\phi(y,v) \bigr) \, \mathrm{d}\eta(y,v) \right. \\
        \left. + \int_{\mathcal{Y}\times\mathcal{U}} -\Psi_2^* \bigl( -\varphi(y,u) \bigr) \, \mathrm{d}\nu(y,u) \right]
    \end{multline} 
    
    \textbf{Semi-Dual Formulation:}
    \begin{multline} 
        \sup_{\varphi} \left[ \int_{\mathcal{Y}\times\mathcal{V}} -\Psi_1^* \bigl( -\varphi^{c}(y,v) \bigr) \, \mathrm{d}\eta(y,v) \right. \\
        \left. + \int_{\mathcal{Y}\times\mathcal{U}} -\Psi_2^* \bigl( -\varphi(y,u) \bigr) \, \mathrm{d}\nu(y,u) \right]
    \end{multline}
    where $\Psi_1^*$ and $\Psi_2^*$ denote the Legendre-Fenchel conjugates of the divergence-defining functions $\Psi_1$ and $\Psi_2$ in the CUOT problem, respectively. The functional variables $\phi \in \mathcal{C(Y \times V)}$ and $\varphi \in \mathcal{C(Y \times U)}$ are referred to as the \textbf{potential functions}.
 
    Furthermore, $\varphi^{c}(y,v)$ represents the $y$-conditional $c$-transform, which is defined as:
    \begin{equation} 
    \varphi^{c}(y,v) := \inf_{u \in  \mathcal{U}} \left\{ c(y, v; y, u) - \varphi(y, u) \right\}.
    \end{equation}
\end{theorem}

\begin{proof}
Let $\tilde{c}$ be the extension of our cost function $c$ to $y \neq y'$
\begin{equation} 
\label{eq:tilde_cost}
    \tilde{c}(y,v,y',u) = 
    \begin{cases}
        c(y,v,y',u) = \tau \|v-u \|_{2}^{2} \qquad \textrm{if } y = y'\\
        \infty \qquad \textrm{if } y \neq y'
    \end{cases}
\end{equation}
The primal problem of Conditional Unbalanced Optimal Transport (CUOT) Kantorovich is defined using the Csisz\'ar divergence as follows:
\begin{equation}
\begin{aligned}
    \inf_{\pi} \quad & \int_{\mathcal{Y}\times\mathcal{V}\times \mathcal{Y} \times \mathcal{U}} \tilde{c}(y,v,y',u)\, \mathrm{d}\pi(y,v,y',u) \\
    &+ \int_{\mathcal{Y}} \left[ \int_{\mathcal{V}} \Psi_1 \left( \frac{\mathrm{d}\pi_{\mathcal{V}}(v|y)}{\mathrm{d}\eta_{\mathcal{V}}(v|y)} \right) \mathrm{d}\eta_{\mathcal{V}}(v|y) \right] \mathrm{d}\eta_Y(y) \\
    &+ \int_{\mathcal{Y}} \left[ \int_{\mathcal{U}} \Psi_2 \left( \frac{\mathrm{d}\pi_{\mathcal{U}}(u|y')}{\mathrm{d}\nu_{\mathcal{U}}(u|y')} \right) \mathrm{d}\nu_{\mathcal{U}}(u|y') \right] \mathrm{d}\nu_Y(y').
\end{aligned}
\end{equation}

Let $\pi_0$ and $\pi_1$ denote the projections of $\pi$ onto $\mathcal{Y}\times\mathcal{V}$ and $\mathcal{Y}\times\mathcal{U}$, respectively. Under the assumption that the $y$-marginals are balanced (i.e., $\pi_{Y} = \eta_{Y}$), the density ratio simplifies to the joint density ratio:
\begin{equation}
    \frac{\mathrm{d}\pi_{\mathcal{V}}(v|y)}{\mathrm{d}\eta_{\mathcal{V}}(v|y)} = \frac{\mathrm{d}\pi_0(y,v)}{\mathrm{d}\eta(y,v)}.
\end{equation}
Applying the same logic to the second penalty term, the primal problem can be compactly rewritten in terms of the marginals $\pi_0$ and $\pi_1$:
\begin{equation}
\begin{aligned}
    \inf_{\pi} \quad & \int \tilde{c}(y,v,y',u)\, \mathrm{d}\pi(y,v,y',u) \\
    &+ \int_{\mathcal{Y} \times \mathcal{V}} \Psi_1 \left( \frac{\mathrm{d}\pi_0(y,v)}{\mathrm{d}\eta(y,v)} \right) \, \mathrm{d}\eta(y,v) \\
    &+ \int_{\mathcal{Y} \times \mathcal{U}} \Psi_2 \left( \frac{\mathrm{d}\pi_1(y',u)}{\mathrm{d}\nu(y',u)} \right) \, \mathrm{d}\nu(y',u).
\end{aligned}
\end{equation}

To derive the dual, we introduce auxiliary measures $P$ and $Q$ constrained such that $P = \pi_0$ and $Q = \pi_1$. Using Lagrange multipliers $\phi \in C_b(\mathcal{Y}\times\mathcal{V})$ and $\varphi \in C_b(\mathcal{Y}\times\mathcal{U})$ to enforce these constraints, the Lagrangian $\mathcal{L}(\pi, P, Q, \phi, \varphi)$ is given by:
\begin{equation}
\begin{aligned}
    \mathcal{L} &= \int \tilde c \, \mathrm{d}\pi 
    + \int \Psi_1 \left( \frac{\mathrm{d}P}{\mathrm{d}\eta} \right) \mathrm{d}\eta 
    + \int \Psi_2 \left( \frac{\mathrm{d}Q}{\mathrm{d}\nu} \right) \mathrm{d}\nu \\
    &\quad + \int \phi(y,v) \left( \mathrm{d}P(y,v) - \mathrm{d}\pi_0(y,v) \right) \\
    &\quad + \int \varphi(y',u) \left( \mathrm{d}Q(y',u) - \mathrm{d}\pi_1(y',u) \right).
\end{aligned}
\end{equation}
Note that $\int \phi \, \mathrm{d}\pi_0 = \int \phi(y,v) \, \mathrm{d}\pi(y,v,y',u)$. The dual objective function $g(\phi,\varphi) = \inf_{\pi, P, Q} \mathcal{L}$ is then separable:

\begin{equation}
\begin{aligned}
    g(&\phi,\varphi) = \inf_{P} \left[ \int \phi\, \mathrm{d}P + \int \Psi_1 \left( \frac{\mathrm{d}P}{\mathrm{d}\eta} \right) \, \mathrm{d}\eta \right] 
    \\&+ \inf_{Q} \left[ \int \varphi\, \mathrm{d}Q + \int \Psi_2 \left( \frac{\mathrm{d}Q}{\mathrm{d}\nu} \right) \, \mathrm{d}\nu \right]
    \\& + \inf_{\pi} \left[ \int \Bigl( \tilde{c}(y,v,y',u) - \phi(y,v) - \varphi(y',u) \Bigr)\, \mathrm{d}\pi \right].
\end{aligned}
\end{equation}

Using the definition of the convex conjugate $\Psi^*(z) = \sup_{t} (tz - \Psi(t))$, the minimization over $P$ yields:
\begin{equation}
\label{eq:fenchel-transform-step}
\begin{aligned}
    &-\sup_{P} \left[ \int \left( -\phi \frac{\mathrm{d}P}{\mathrm{d}\eta} -  \Psi_1 \left( \frac{\mathrm{d}P}{\mathrm{d}\eta} \right) \right) \mathrm{d}\eta \right] = \\
    & -\int \Psi_1^* \bigl( -\phi(y,v) \bigr) \, \mathrm{d}\eta(y,v).
\end{aligned}
\end{equation}
A similar result holds for the term involving $Q$ and $\Psi_2^*$. To avoid the trivial solution $-\infty$ in the minimization over $\pi$, we impose the constraint:
\[
    \phi(y,v) + \varphi(y',u) \le \tilde c(y,v,y',u) \quad \text{almost everywhere.}
\]
This leads to the dual formulation:
\begin{equation}
\label{eq:final-dual-uot}
\begin{aligned}
    \sup_{\substack{\phi,\varphi \\ \phi + \varphi \le \tilde c}} & \left[ \int_{\mathcal{Y}\times\mathcal{V}} -\Psi_1^* \bigl( -\phi(y,v) \bigr) \, \mathrm{d}\eta(y,v) \right. \\
    & \quad \left. + \int_{\mathcal{Y}\times\mathcal{U}} -\Psi_2^* \bigl( -\varphi(y',u) \bigr) \, \mathrm{d}\nu(y',u) \right].
\end{aligned}
\end{equation}

Since $c \ge 0$, choosing $\phi \equiv -1$ and $\varphi \equiv -1$ ensures finiteness, satisfying the conditions for Fenchel-Rockafellar's theorem, which implies strong duality. Finally, utilizing the tightness condition $\phi(y,v) \approx \inf_{u} [c(y,v,y,u) - \varphi(y,u)]$, we obtain the semi-dual form:
\begin{equation}
\label{eq:semi-dual-final}
\begin{aligned}
    \sup_{\varphi} \biggl[ &\int_{\mathcal{Y}\times\mathcal{V}} -\Psi_1^* \left( - \inf_{u \in  \mathcal{U}} \bigl[ c(y,v,y,u) - \varphi(y,u) \bigr] \right) \mathrm{d}\eta(y,v) \\
    &+ \int_{\mathcal{Y}\times\mathcal{U}} -\Psi_2^*\bigl(-\varphi(y,u)\bigr) \, \mathrm{d}\nu(y,u) \biggr].
\end{aligned}
\end{equation}
\end{proof}

\subsection{Proof of Theorem \ref{thm:parameter} (Validity of Parameterization for CUOT)}

\begin{theorem}[Validity of Parameterization for CUOT] 

Given the existence of an optimal potential $\varphi^\star$ in the CUOT problem [Eq.\ref{eq:cuot-semidual}], there exists an optimal triangular map $T^{\triangle\star} : \mathcal Y \times\mathcal V \rightarrow \mathcal Y \times \mathcal U$, equivalently $T^{\Delta \star}(y,v) = (y,\, T^\star(y,v)).$ 
such that 
\begin{multline}
    T^\star(y,v) \in \operatorname{arginf}_{u \in \mathcal U} [c(y, v; y, u) - \varphi^\star(y, u)] \\ 
    \text{$\eta$-almost surely.}
\end{multline}
 
By the equivalence established in Eq.\ref{eq:c-transform param}, the $y$-conditional $c$-transform $\varphi^c$ is validly parameterized by the triangular map $T_\varphi^\triangle(y, v) = (y, T_\varphi(y, v))$ as follows:
\begin{equation}
    \varphi^c(y, v) = c(y, v; y, T_\varphi(y, v)) - \varphi(y, T_\varphi(y, v)).
\end{equation}
In particular, the following inequality holds:
\begin{equation}
\begin{aligned}
    \tau \left({\mathcal{W}_2^{\nu_Y}}\right)^2(\eta, \nu) &\geq \int_{\mathcal Y} D_{\Psi_1}\!\big(\tilde{\eta}(\cdot\mid y)\,\|\,\eta(\cdot\mid y)\big)\, \mathrm d\nu_Y(y) \\
&+ \int_{\mathcal Y} D_{\Psi_2}\!\big(\tilde{\nu}(\cdot\mid y)\,\|\,\nu(\cdot\mid y)\big)\, \mathrm d\nu_Y(y).
\end{aligned}
\end{equation}
where $\tilde{\eta} := \pi_{1}^{\star}$ and $\tilde{\nu} := \pi_{2}^{\star}$ denote the relaxed conditional marginal distributions of the optimal CUOT plan, which are explicitly given by $d\tilde\eta(y,v) = \Psi_1^{*'}(-\varphi^*(y,v))\,d\eta(y,v)$ and $d\tilde\nu(y,u) = \Psi_2^{*'}(-\phi^*(y,u))\,d\nu(y,u)$. Here, ${\mathcal{W}_2^{\nu_Y}}$ denotes the Conditional Wasserstein-2 distance defined in Eq. \ref{eq:cond_wasserstein}
\end{theorem}

\begin{lemma} \label{lemma:linear}
For $(\varphi,\phi)\in C(Y\times V)\times C(Y\times U)$, define
\begin{equation} \label{eq:lemma}
\begin{aligned}
I(\varphi,\phi)
:=\;& \int_{Y\times V} -\Psi_1^*\!\big(-\varphi(y,v)\big)\, d\eta(y,v) \\ 
&+ \int_{Y\times U} -\Psi_2^*\!\big(-\phi(y',u)\big)\, d\nu(y',u).
\end{aligned}
\end{equation}
Then for any variations $(\delta\varphi,\delta\phi)\in C(Y\times V)\times C(Y\times U)$,
the functional derivative of $I$ is given by
\begin{equation}
\begin{aligned}
&\left.\frac{d}{d\epsilon}\, I(\varphi+\epsilon\delta\varphi,\ \phi+\epsilon\delta\phi)\right|_{\epsilon=0}= \\
&\int_{Y\times V} \delta\varphi(y,v)\, \Psi_1^{*'}\!\big(-\varphi(y,v)\big)\, d\eta(y,v) \\
+
&\int_{Y\times U} \delta\phi(y',u)\, \Psi_2^{*'}\!\big(-\phi(y',u)\big)\, d\nu(y',u).
\end{aligned}
\end{equation}    
\end{lemma}

\begin{proof}
For any $(f,g)\in C(Y\times V)\times C(Y\times U)$, let
\begin{equation}
\begin{aligned}
I(f,g)
&:=
\int_{Y\times V} -\Psi_1^*\!\big(-f(y,v)\big)\, d\eta(y,v) \\
&+\int_{Y\times U} -\Psi_2^*\!\big(-g(y',u)\big)\, d\nu(y',u).
\end{aligned}
\end{equation}
Then, for arbitrarily given potentials $(\varphi,\phi)\in C(Y\times V)\times C(Y\times U)$,
and perturbations $(\delta\varphi,\delta\phi)\in C(Y\times V)\times C(Y\times U)$ which satisfy $
(\varphi+\delta\varphi)(y,v)+(\phi+\delta\phi)(y',u)
\le  c\big(y,v,y',u)\big)$,
\begin{equation}
\begin{aligned}
\mathcal{L}(\epsilon)
&:=\frac{d}{d\epsilon}\, I(\varphi+\epsilon\delta\varphi,\ \phi+\epsilon\delta\phi).
\end{aligned}
\end{equation}
Note that $\mathcal{L}(\epsilon)$ is a continuous function of $\epsilon$. Moreover, by the chain rule,
\begin{equation}
\begin{aligned}
\mathcal{L}(\epsilon)
&=
\int_{Y\times V}\delta\varphi(y,v)\,
\Psi_1^{*'}\!\big(-(\varphi+\epsilon\delta\varphi)(y,v)\big)\, d\eta(y,v)
\\
&\quad+
\int_{Y\times U}\delta\phi(y',u)\,
\Psi_2^{*'}\!\big(-(\phi+\epsilon\delta\phi)(y',u)\big)\, d\nu(y',u).
\end{aligned}
\end{equation}
Thus, the functional derivative of $I(\varphi,\phi)$, i.e.\ $\mathcal{L}(0)$, is given by
\begin{equation}
\begin{aligned}
\mathcal{L}(0)
&=
\int_{Y\times V}\delta\varphi(y,v)\,\Psi_1^{*'}\!\big(-\varphi(y,v)\big)\, d\eta(y,v)\\
&+
\int_{Y\times U}\delta\phi(y',u)\,\Psi_2^{*'}\!\big(-\phi(y',u)\big)\, d\nu(y',u),
\end{aligned}
\end{equation}
where $(\delta\varphi,\delta\phi)\in C(Y\times V)\times C(Y\times U)$ denote the variations of
$(\varphi,\phi)$, respectively.
\end{proof}

\begin{proof}
By Lemma \ref{lemma:linear}, the derivative of Eq.\ref{eq:lemma} is
\begin{equation}\label{eq:cuot_deriv}
\begin{aligned}
    \int_{Y\times V} & \delta\varphi(y,v)\Psi_1^{*'}\!\big(-\varphi^*(y,v)\big)\,d\eta(y,v)
    \\&+
    \int_{Y\times U}\delta\phi(y',u)\Psi_2^{*'}\!\big(-\phi^*(y',u)\big)\,d\nu(y',u).
\end{aligned}
\end{equation}
Any such linearization must satisfy the inequality constraint
\[
(\varphi^*+\delta\varphi)(y,v)+(\phi^*+\delta\phi)(y',u)\le
 \tilde{c}\big(y,v,y',u\big),
\]
so that $(\delta\varphi,\delta\phi)$ defines an admissible variation. Under this constraint, the linearized dual form can be expressed as
\begin{equation}\label{eq:cuot_linearized}
    \sup_{(\delta\varphi,\delta\phi)\in C(Y\times V)\times C(Y\times U)}
    \int_{Y\times V}\delta\varphi\,d\tilde\eta+\int_{Y\times U}\delta\phi\,d\tilde\nu,
\end{equation}
where $\delta\varphi(y,v)+\delta\phi(y',u) \le \tilde{c}\big(y,v,y',u)\big)-\varphi^*(y,v)-\phi^*(y',u)$, $d\tilde\eta(y,v) = \Psi_1^{*'}(-\varphi^*(y,v))\,d\eta(y,v)$, and $d\tilde\nu(y',u) = \Psi_2^{*'}(-\phi^*(y',u))\,d\nu(y',u).$

Moreover, when $(\varphi^*, \phi^*)$ is optimal, the inequality constraint together with tightness
(i.e.\ $ \tilde{c} - \varphi^* - \phi^* = 0$ $\tilde\pi^*$-almost everywhere)
ensures that the linearized dual objective is non-positive.
In other words, Eq.\ref{eq:cuot_linearized} can be viewed as the linearization of the OT dual
\begin{equation}\label{eq:cuot_ot_dual_style}
    \sup_{\varphi+\phi\le  c}
    \left[
    \int_{Y\times V}\varphi\,d\tilde\eta
    +
    \int_{Y\times U}\phi\,d\tilde\nu
    \right].
\end{equation}
Therefore, the CUOT problem reduces to a standard OT formulation with marginals $\tilde{\eta}$ and $\tilde{\nu}$.
By \cite{villani, hosseini2025conditional}, there exists a unique measurable OT map $T^*$ solving the Monge problem between $\tilde{\eta}$ and $\tilde{\nu}$.
Furthermore, as shown in \cite{villani}, it satisfies
\begin{equation}
    T^*(y, v) \in \operatorname*{arg\,inf}_{(y', u) \in \mathcal{Y} \times \mathcal{U}}
\left\{ \tilde{c}(y, v, y', u)) - \phi^*(y', u) \right\}.
\end{equation}

Hence, recall that the cost function $\tilde{c}$ in Eq. \ref{eq:tilde_cost} is infinite whenever $y \neq y'$.
Consequently, the induced optimal transport map $T^*$ inherits the triangular structure, satisfying:
$$ T^*(y, v) = (y, T_{\mathcal{U}}(y, v)) \quad \tilde{\eta}\text{-almost surely}.$$ This justifies the parameterization of the map in our semi-dual formulation as a triangular map.

Eq.~\ref{eq:condUOT} can be equivalently written as
\begin{equation}
\label{eq:condUOT_wise}
\begin{aligned}
\inf_{\pi(\cdot| y)\in\mathcal M_+(\mathcal{V} \times \mathcal{U})}
\int_{\mathcal Y}&\Big[
\int_{\mathcal V\times\mathcal U} c(y,v,y,u)\,\mathrm d\pi(v,u| y)
\\&+ D_{\Psi_{1}}\big(\pi_1(\cdot | y) \| \eta(\cdot | y)\big)
\\&+ D_{\Psi_{2}}\big(\pi_2(\cdot | y) \| \nu(\cdot | y)\big)
\Big]\mathrm d\nu_Y(y)
\end{aligned}
\end{equation}
Now, if we  constrain the conditional plan to the strict coupling set $\Pi(\eta(\cdot| y),\nu(\cdot|y))$, the divergence terms vanish and the objective reduces to the conditional OT cost, whose optimum equals $\tau \mathcal{W}_2^2(\eta(\cdot| y),\nu(\cdot| y))$. Thus, for the optimal $\pi(\cdot | y)\in\mathcal M_+$ where $\tilde{\eta}:=\pi_{1}(\cdot|y)$ and $\tilde{\nu}:=\pi_{2}(\cdot|y)$, it is trivial that

\begin{equation}
\begin{aligned}
    \tau {\mathcal{W}_2^{\nu_Y}}^2&(\eta, \nu) \geq \tau {\mathcal{W}_2^{\nu_Y}}^2(\tilde{\eta}, \tilde{\nu}) \\&+ \int_{\mathcal Y} D_{\Psi_1}\!\big(\tilde{\eta}(\cdot| y)\,\|\,\eta(\cdot| y)\big)\, \mathrm d\nu_Y(y) \\
&+ \int_{\mathcal Y} D_{\Psi_2}\!\big(\tilde{\nu}(\cdot| y)\,\|\,\nu(\cdot| y)\big)\, \mathrm d\nu_Y(y).\\
    &\geq  \int_{\mathcal Y} D_{\Psi_1}\!\big(\tilde{\eta}(\cdot| y)\,\|\,\eta(\cdot| y)\big)\, \mathrm d\nu_Y(y) \\
&+ \int_{\mathcal Y} D_{\Psi_2}\!\big(\tilde{\nu}(\cdot| y)\,\|\,\nu(\cdot| y)\big)\, \mathrm d\nu_Y(y).
\end{aligned}
\end{equation}

\end{proof}

\subsection{Marginal Tightness in \texorpdfstring{$\alpha$}{alpha}-scheduled CUOT} \label{app:tightness_alpha}
\begin{proof}
Recall that CUOT problem Eq.~\ref{eq:condUOT} can be equivalently written as Eq.~\ref{eq:condUOT_wise}. 
In particular, the $\alpha$-scaled CUOT problem $C_{ub}^{\alpha}$ is equivalent to the CUOT problem $C_{ub}$ with a rescaled cost intensity $\tau \mapsto \tau/\alpha$. Indeed,
{
\small
\begin{equation} 
\label{eq:alpah_CUOT_converge}
\begin{aligned}
\pi^{\alpha,\star} =
\operatorname*{arg\,inf}_{\pi^{\alpha}(\cdot| y)\in \mathcal M_+(\mathcal V\times \mathcal U)}
\int_{\mathcal Y}&\Big[
\int_{\mathcal V\times\mathcal U} \tau \| v-u \|_{2}^{2}\,\mathrm d\pi^{\alpha}(v,u| y)
\\&+ \alpha D_{\Psi_{1}}\big(\pi_1^{\alpha}(\cdot | y) \| \eta(\cdot | y)\big)
\\&+ \alpha D_{\Psi_{2}}\big(\pi_2^{\alpha}(\cdot | y) \| \nu(\cdot | y)\big)
\Big]\mathrm d\nu_Y(y)
\\=
\operatorname*{arg\,inf}_{\pi(\cdot| y)\in \mathcal M_+(\mathcal V\times \mathcal U)}
\int_{\mathcal Y}&\Big[
\int_{\mathcal V\times\mathcal U} \frac{\tau}\alpha   \| v-u \|_{2}^{2}\,\mathrm d\pi(v,u| y)
\\&+ D_{\Psi_{1}}\big(\pi_1(\cdot | y) \| \eta(\cdot | y)\big)
\\&+ D_{\Psi_{2}}\big(\pi_2(\cdot | y) \| \nu(\cdot | y)\big)
\Big]\mathrm d\nu_Y(y)
\end{aligned}
\end{equation}}
As established in Theorem.~\ref{thm:parameter}, the marginal discrepancies scale linearly with the cost intensity. For the optimal plan $\pi^{\alpha,\star}$, this relationship can be interpreted as:
\begin{equation}
\label{eq:tightness_alpha}
\begin{aligned}
\frac{\tau}{\alpha}\,{\mathcal{W}_2^{\nu_Y}}^2(\eta, \nu)
&\ge
\int_{\mathcal Y} D_{\Psi_1}\!\big(\tilde{\eta}^\alpha(\cdot| y)\,\|\,\eta(\cdot| y)\big)\, \mathrm d\nu_Y(y) \\
&+ \int_{\mathcal Y} D_{\Psi_2}\!\big(\tilde{\nu}^\alpha(\cdot| y)\,\|\,\nu(\cdot| y)\big)\, \mathrm d\nu_Y(y).
\end{aligned}
\end{equation}
where $\tilde{\eta}^{\alpha}(\cdot| y):=\pi^{\alpha,\star}_{1}(\cdot| y)$ and $\tilde{\nu}^{\alpha}(\cdot | y):=\pi^{\alpha,\star}_{2}(\cdot | y)$. Therefore, as $\alpha \to \infty$, the conditional marginal distributions of $\pi^{\alpha,\star}$ converge in the Csiszár divergences to the source and target conditional distributions, i.e.,
\begin{equation}
\label{eq:alpha_infty_marginals}
\begin{aligned}
\lim_{\alpha\to\infty} D_{\Psi_1}&\!\big(\tilde\eta^{\alpha}(\cdot| y)\,\|\,\eta(\cdot| y)\big)
\\&=
\lim_{\alpha\to\infty} D_{\Psi_2}\!\big(\tilde\nu^{\alpha}(\cdot| y)\,\|\,\nu(\cdot| y)\big)
= 0
\end{aligned}
\end{equation}
for $\nu_Y$-a.e.\ $y\in\mathcal Y$. This formulation implies that $\alpha$-scheduling to dynamically tighten the distribution matching throughout the training process. 
\end{proof}

\section{Experiment Details}
\label{sec:Experiment_details}
In this section, we describe the experimental setup and the implementation details of our model.
\subsection{Data} 
\label{sec:data}
We will describe the dataset setup used to evaluate conditional generative modeling.

\paragraph{2D synthetic}
\label{sec:data-2D}
The 2D synthetic dataset is a two-dimensional dataset \cite{pedregosa2011scikit}. Specifically, it includes moons, circles, swissroll, and checkerboard. The dataset is defined as the set $\{(u,y) \in \mathbb{R}^2 | u \in U = \mathbb{R}, y \in Y = \mathbb{R}\}$. Given $y$, we perform conditional modeling of $\nu(u|y)$. We evaluated the performance using the Wasserstein-2 ($\mathcal W_2$) distance. For the dataset size, we generated 20,000 samples for training and 5,000 samples for testing. All other data arguments were set following \cite{kerrigan2024dynamic}, and we used the generation code from their work.

\paragraph{CIFAR-10}
\label{sec:data-cifar}
The CIFAR-10 dataset \cite{cifar10} consists of 50,000 training images distributed across 10 classes, all of which were utilized for our class-conditional generation experiments. To evaluate the quality of the generated images at regular intervals during training, we used the entire training set of 50,000 samples as the reference distribution. Our evaluation metrics include the Fréchet Inception Distance (FID) \cite{fid}, Inception Score (IS) \cite{salimans2016improved}, and Intra-class FID (IFID) \cite{miyato2018cgans}. All other experimental configurations and setup details follow the protocols described in \cite{uotm}.

\subsection{Evaluation Metric}
For the 2D synthetic dataset, we measured the generative performance using the Wasserstein-2 ($\mathcal{W}_2$) distance. For the CIFAR-10 experiments, we adopted FID \cite{fid}, IS \cite{salimans2016improved}, and IFID \cite{miyato2018cgans} as our primary evaluation metrics. The IFID is determined by averaging the individual FID scores measured across all distinct classes..

\subsection{Implementation Details}
\label{sec:implementation_details}
In this section, we describe the implementation and training details for each dataset. We keep the Conditional UOT framework consistent to ensure theoretical guarantees. At the same time, we choose network architectures based on the domain characteristics of low-dimensional synthetic data and high-dimensional image data.

\paragraph{2D synthetic Data}
\label{2D implementations}
We followed the toy data implementation from \cite{uotm} and extended it to a conditional setting. The triangular mapping structure is defined as $(T_Y(y),T_U(T_Y(y), x, z))$. To implement this, the condition $y$ and data $x$ are first embedded into 128-dimensional vectors via two and three Residual Blocks, respectively. These embeddings are then concatenated and added to a 256-dimensional auxiliary variable $z$. The resulting vector is fed into the output module. The Residual Blocks and the output module are identical to those in \cite{uotm}. Since the first component of the triangular mapping, i.e. $T_Y(y)$, is an identity mapping, the output module only generates the $U$-component. For the potential network (discriminator), we added a module to process $y$ that is identical to the existing module used for $x$. We set the hyperparameters, including $\tau$, epochs, batch size, and learning rate, as follows: $\tau = 0.0007$, 800 epochs, batch size of 256, and learning rates of $1.6 \times 10^{-4}$ and $10^{-4}$ for the generator and discriminator, respectively. We achieved the best performance when using (KL, KL) as the convex conjugate and adjusting the R1 regularization hyperparameter to $\gamma = 0.1$.

\paragraph{CIFAR-10}
\label{CIFAR-10 implementations}
Our implementation is primarily based on \cite{uotm}. But we modify several points of \cite{uotm} for conditioning our generator and potential. The potential network follows Projection discriminator\cite{miyato2018cgans}, which is widely adopted in conditional generative adversarial networks. For the generator, we use NCSN++ \cite{song2020score} as the backbone and integrate the conditional information via AdaGN modulation, drawing inspiration from the AdaIN mechanism in \cite{styleganada2020}.

\paragraph{Generating function $\Psi$ of the Csisár divergence} \label{app:candidate_div}

In this section, we introduce candidate functions for $\Psi$. We focus on convex conjugate $\Psi^*$ that satisfy the requirements of being convex, differentiable, and non-decreasing, as employed in table \ref{tab:divergence}
\\
\begin{itemize}
    \item \textbf{KL divergence}:
    \begin{equation}
    \begin{aligned}
        &\Psi(x) = \begin{cases} x \log x - x + 1, & \text{if } x > 0 \\ \infty, & \text{if } x \le 0 \end{cases} \\
        &\Psi^*(x) = e^x - 1
    \end{aligned}
    \end{equation}

    \item \textbf{$\chi^2$ divergence}:
    \begin{equation}
    \begin{aligned}
        &\Psi(x) = \begin{cases} (x-1)^2, & \text{if } x \ge 0 \\ \infty, & \text{if } x < 0 \end{cases} \\
        &\Psi^*(x) = \begin{cases} \frac{1}{4}x^2 + x, & \text{if } x \ge -2 \\ -1, & \text{if } x < -2 \end{cases}
    \end{aligned}
    \end{equation}
    \item \textbf{softplus}: 
        \begin{equation}
        \begin{aligned}
            &\Psi(x) = \begin{cases} x \log x + (1-x)\log(1-x), & \text{if } 0 \le x \le 1 \\
            \infty, & \text{otherwise}
            \end{cases} \\
            &\Psi^*(x) = \log(1+e^{x})
        \end{aligned}
        \end{equation}
\end{itemize}

\section{Additional Result} \label{app:addtional_results}

\begin{figure*}[t]
    \centering

    \begin{minipage}{0.15\linewidth} \centering \small \hspace*{14pt}Target \end{minipage} \hfill
    \begin{minipage}{0.15\linewidth} \centering \small \hspace*{14pt}COTM \end{minipage} \hfill
    \begin{minipage}{0.15\linewidth} \centering \small \hspace*{14pt}CUOTM \end{minipage} \hfill
    \begin{minipage}{0.21\linewidth} \centering \small \hspace*{11pt}COTM (KDE) \end{minipage} \hfill
    \begin{minipage}{0.21\linewidth} \centering \small \hspace*{11pt}CUOTM (KDE) \end{minipage}
    
    \vspace{0pt}

    \begin{minipage}[t]{0.15\linewidth}
        \includegraphics[width=\textwidth, align=c]{figure/2moons/2moons_true.png}
    \end{minipage} \hfill
    \begin{minipage}[t]{0.15\linewidth}
        \includegraphics[width=\textwidth, align=c]{figure/2moons/2moons_OT.png}
    \end{minipage} \hfill
    \begin{minipage}[t]{0.15\linewidth}
        \includegraphics[width=\textwidth, align=c]{figure/2moons/2moons_UOT.png}
    \end{minipage} \hfill
    \begin{minipage}[t]{0.21\linewidth}
        \includegraphics[width=\textwidth, align=c]{figure/2moons/2moons_OT_kde.png}
    \end{minipage} \hfill
    \begin{minipage}[t]{0.21\linewidth}
        \includegraphics[width=\textwidth, align=c]{figure/2moons/2moons_UOT_kde.png}
    \end{minipage}
    
    \medskip
    
    \begin{minipage}[t]{0.15\linewidth}
        \includegraphics[width=\textwidth, align=c]{figure/board/checkerboard_true.png}
    \end{minipage} \hfill
    \begin{minipage}[t]{0.15\linewidth}
        \includegraphics[width=\textwidth, align=c]{figure/board/checkerboard_OT.png}
    \end{minipage} \hfill
    \begin{minipage}[t]{0.15\linewidth}
        \includegraphics[width=\textwidth, align=c]{figure/board/checkerboard_UOT.png}
    \end{minipage} \hfill
    \begin{minipage}[t]{0.21\linewidth}
        \includegraphics[width=\textwidth, align=c]{figure/board/checkerboard_OT_kde.png}
    \end{minipage} \hfill
    \begin{minipage}[t]{0.21\linewidth}
        \includegraphics[width=\textwidth, align=c]{figure/board/checkerboard_UOT_kde.png}
    \end{minipage}
    
    \medskip
    
    \begin{minipage}{0.15\linewidth}
        \includegraphics[width=\textwidth, align=c]{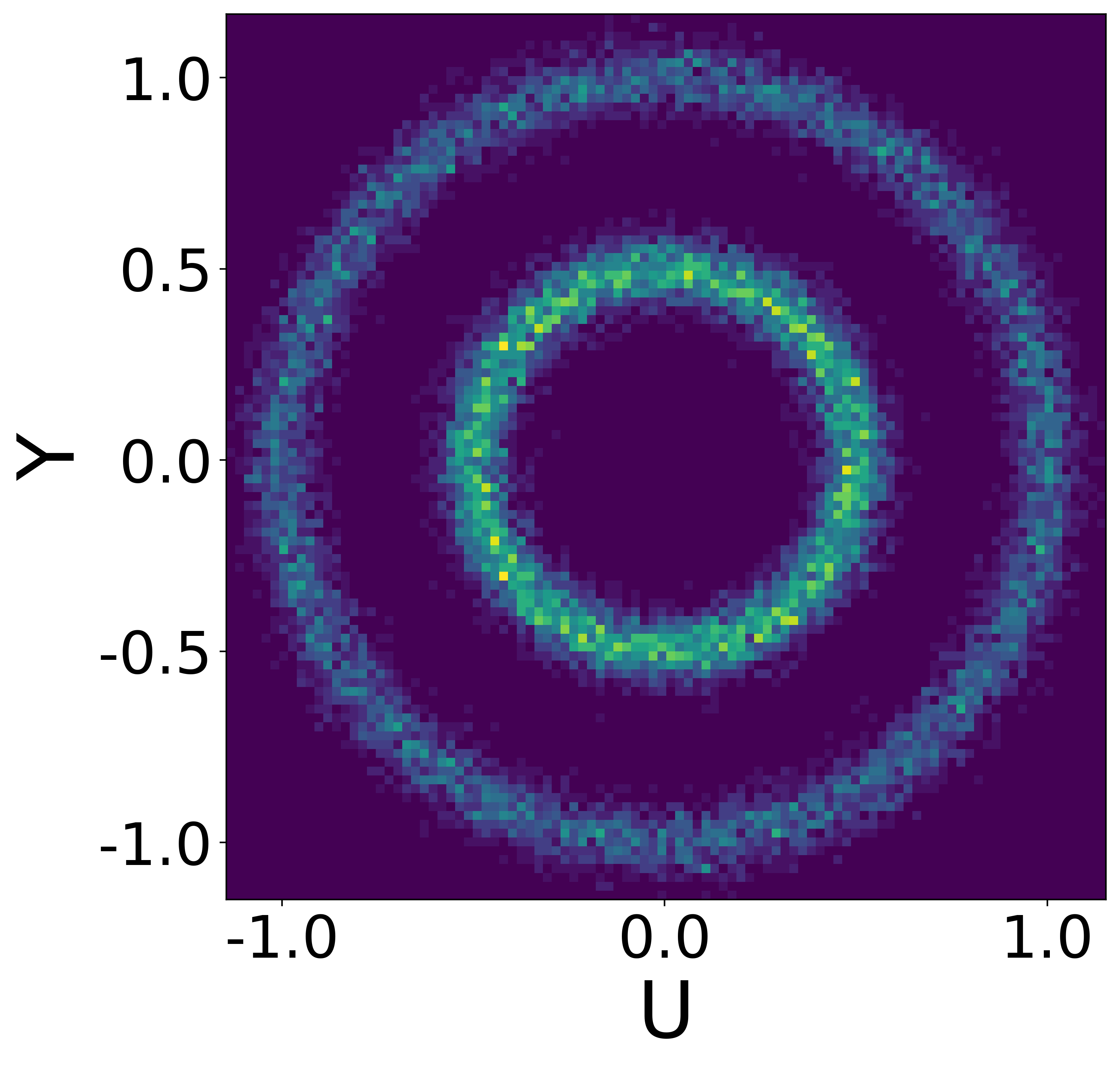}
    \end{minipage} \hfill
    \begin{minipage}{0.15\linewidth}
        \includegraphics[width=\textwidth, align=c]{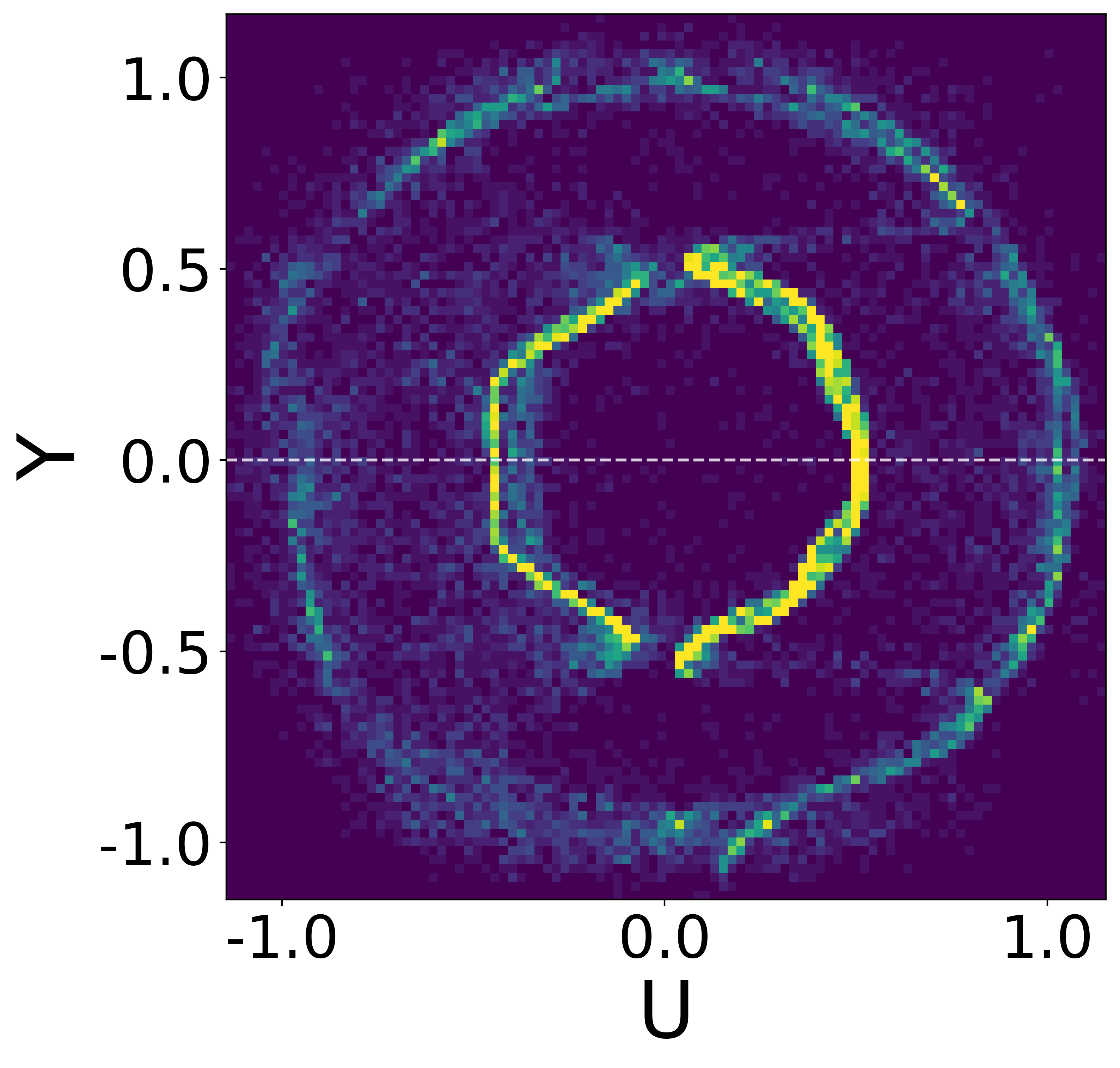}
    \end{minipage} \hfill
    \begin{minipage}{0.15\linewidth}
        \includegraphics[width=\textwidth, align=c]{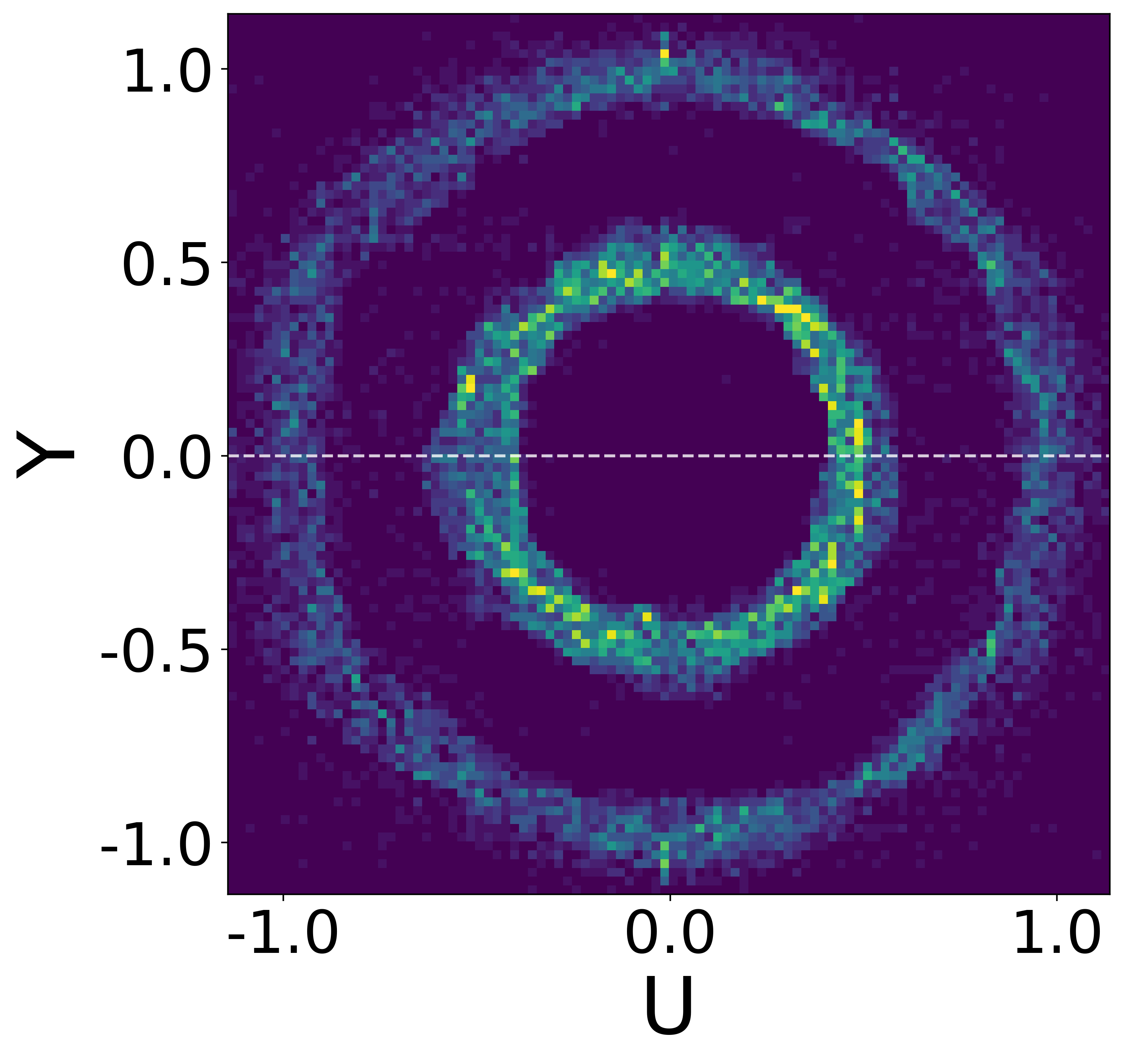}
    \end{minipage} \hfill
    \begin{minipage}{0.21\linewidth}
        \includegraphics[width=\textwidth, align=c]{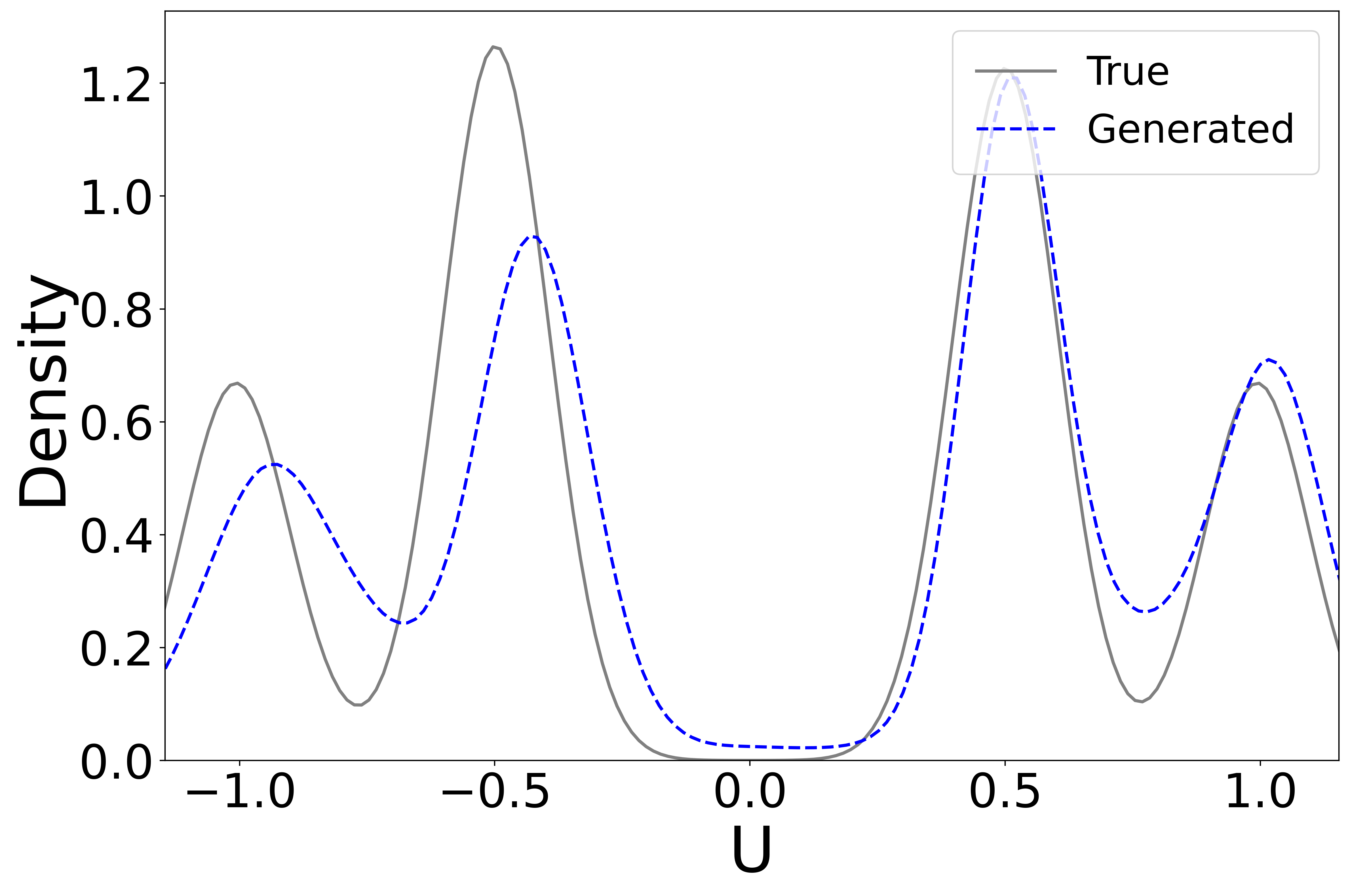}
    \end{minipage} \hfill
    \begin{minipage}{0.21\linewidth}
        \includegraphics[width=\textwidth, align=c]{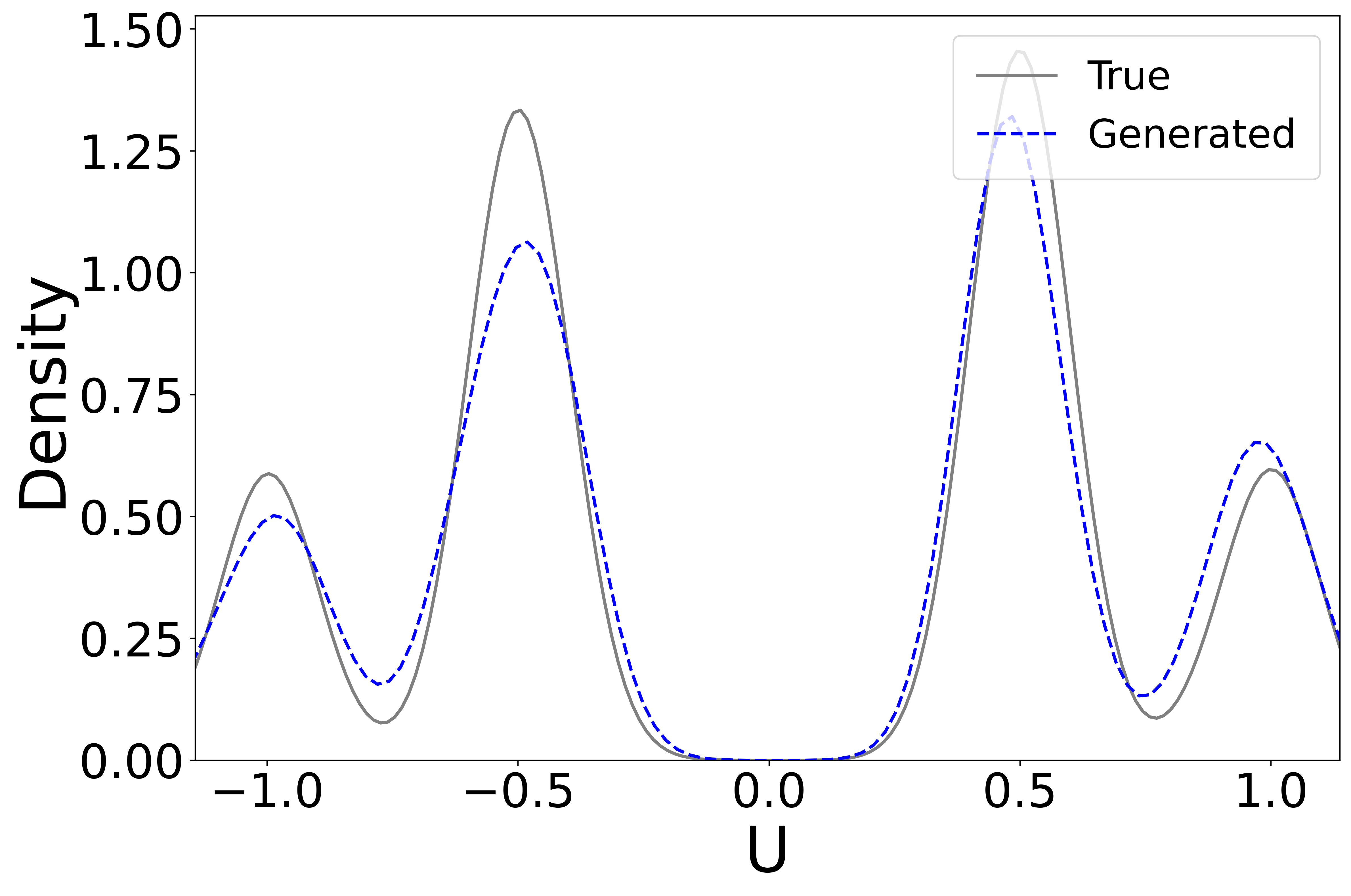}
    \end{minipage}
    
    \medskip
    
    \begin{minipage}{0.15\linewidth}
        \includegraphics[width=\textwidth, align=c]{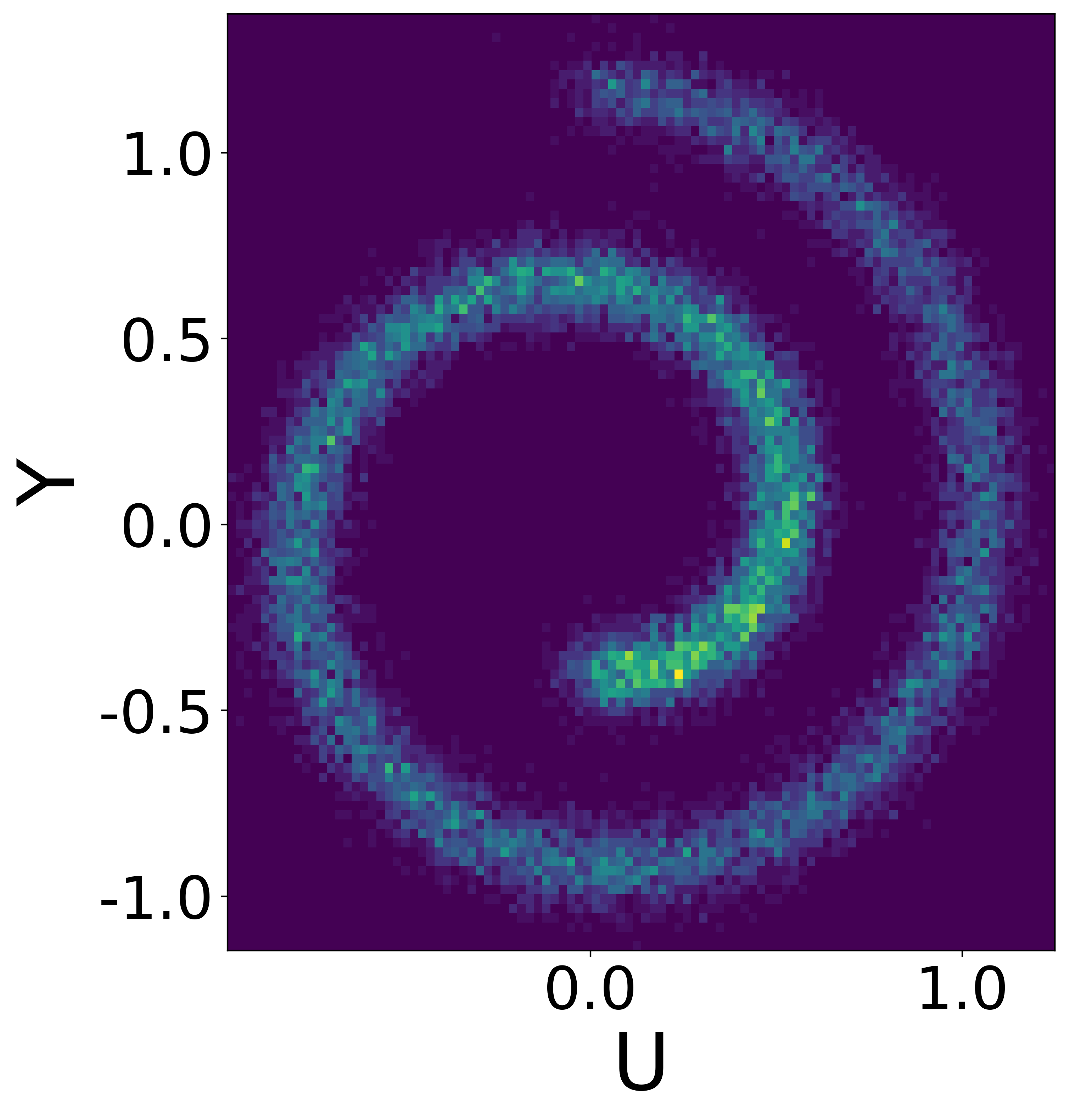}
    \end{minipage} \hfill
    \begin{minipage}{0.15\linewidth}
        \includegraphics[width=\textwidth, align=c]{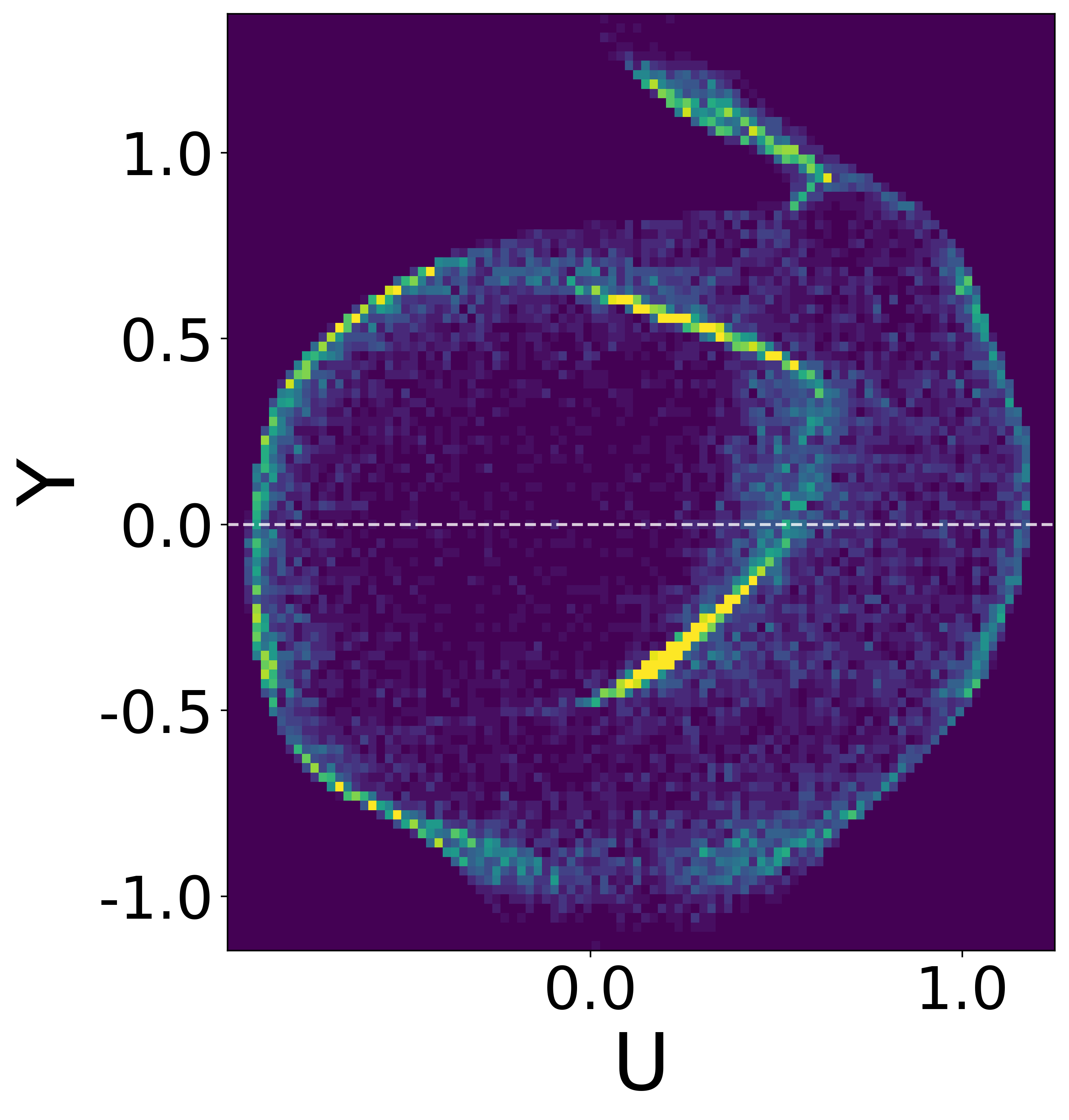}
    \end{minipage} \hfill
    \begin{minipage}{0.15\linewidth}
        \includegraphics[width=\textwidth, align=c]{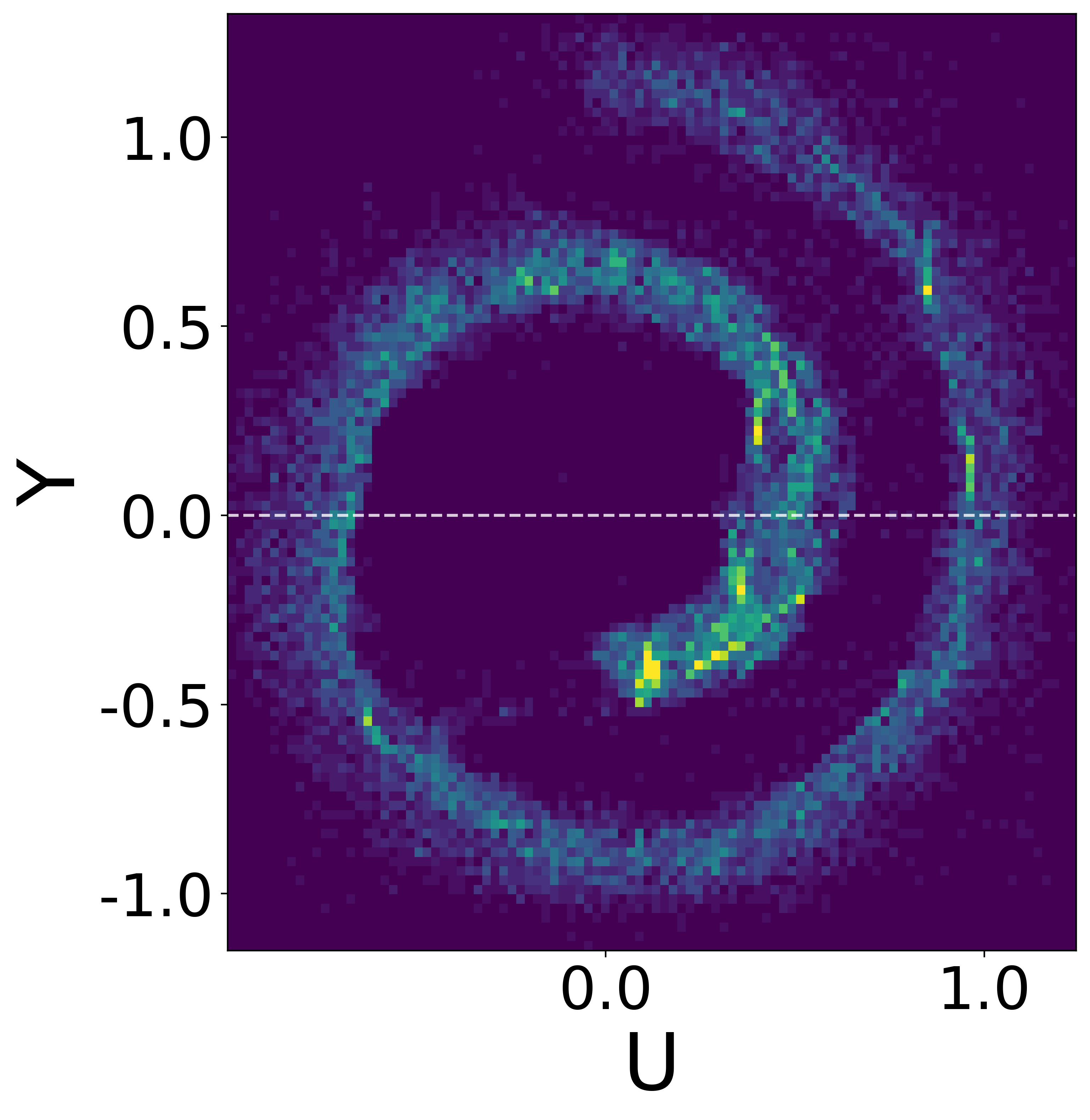}
    \end{minipage} \hfill
    \begin{minipage}{0.21\linewidth}
        \includegraphics[width=\textwidth, align=c]{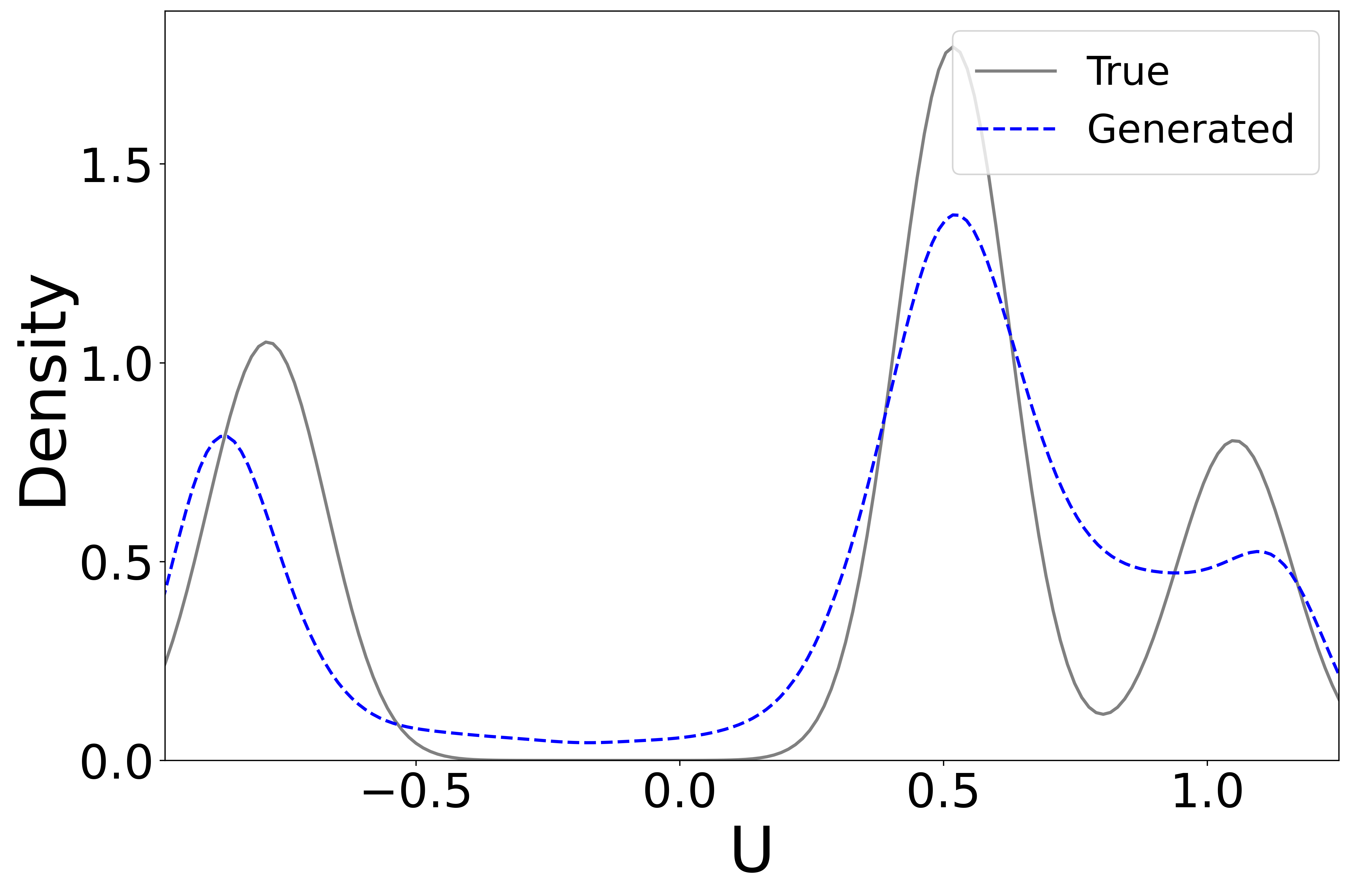}
    \end{minipage} \hfill
    \begin{minipage}{0.21\linewidth}
        \includegraphics[width=\textwidth, align=c]{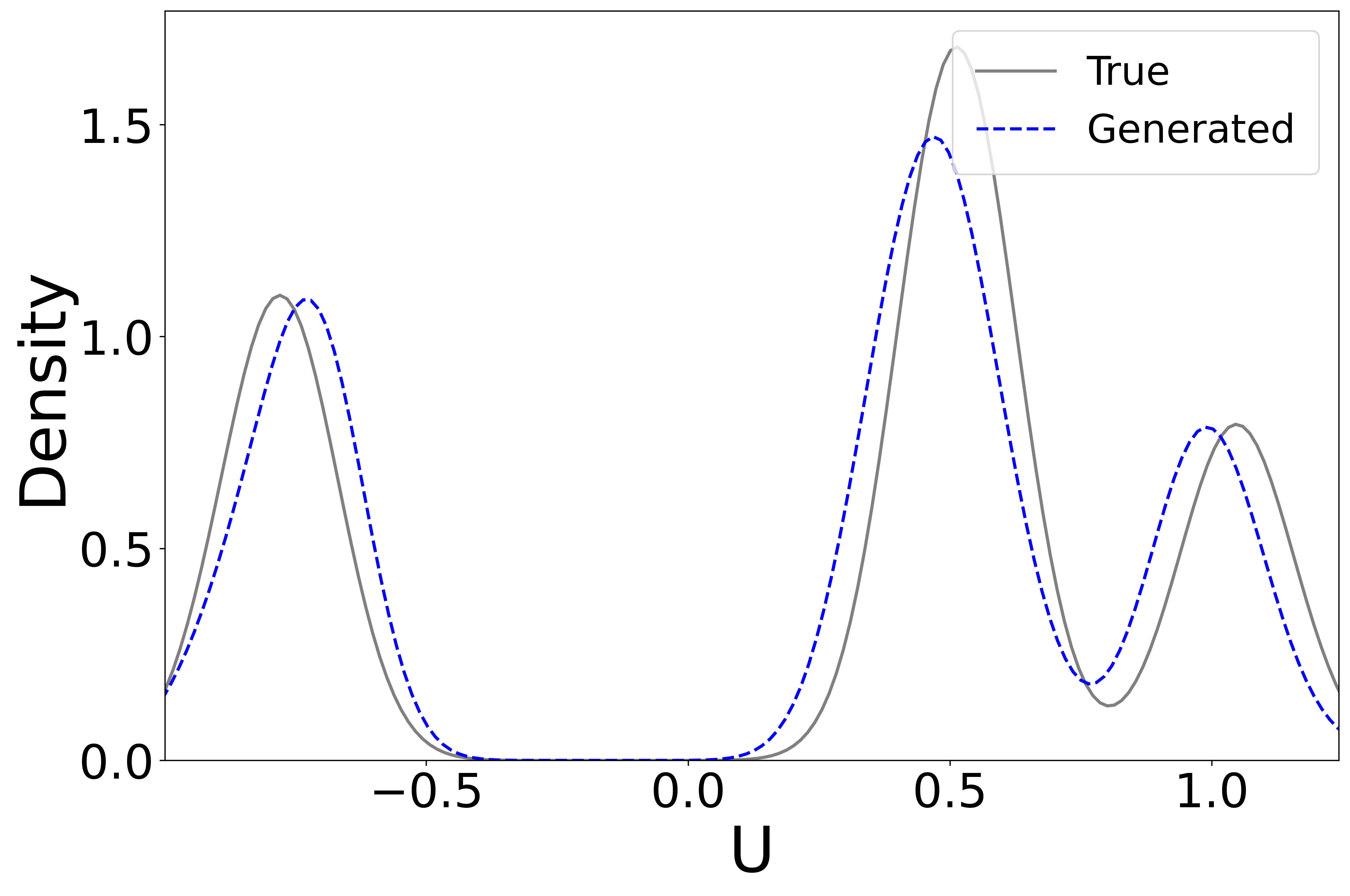}
    \end{minipage}

    \caption{\textbf{Qualitative results on 2D synthetic datasets.} From left to right: target distribution, samples from COTM, samples from CUOTM (Ours), and the respective KDE density visualizations. Each row shows results for the \textit{Moons}, \textit{Checkerboard}, \textit{Circles}, and \textit{Swissroll} datasets, respectively. 
    } 
    \label{fig:2d_synthetic_results_appendix}
\end{figure*}

\begin{figure*}[t]
    \centering
    
    \subfloat{
        \includegraphics[width=0.3\linewidth]{figure/outliar/circles_true_152.png}
    }
    \hfill
    \subfloat{
        \includegraphics[width=0.3\linewidth]{figure/outliar/circles_OT_1.52.png}
    }
    \hfill
    \subfloat{
        \includegraphics[width=0.3\linewidth]{figure/outliar/circles_UOT_1.52.png}
    }   

    \medskip

    \subfloat{
        \includegraphics[width=0.3\linewidth]{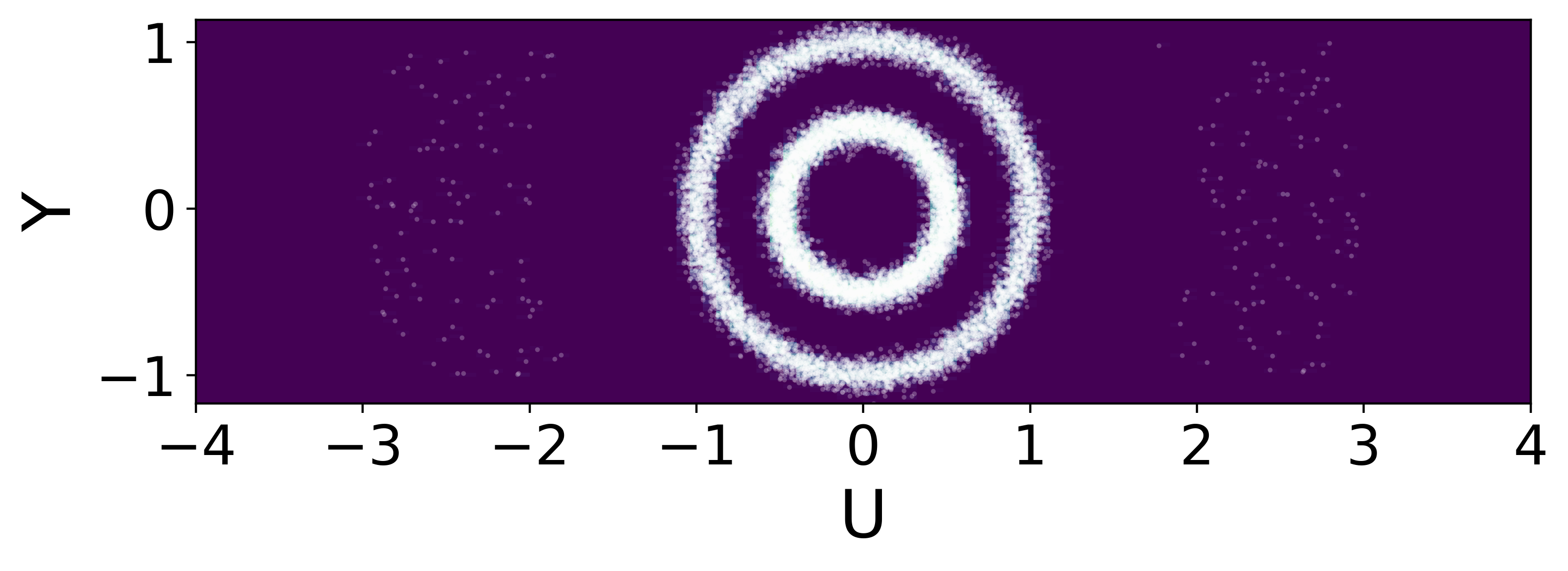}
    }
    \hfill
    \subfloat{
        \includegraphics[width=0.3\linewidth]{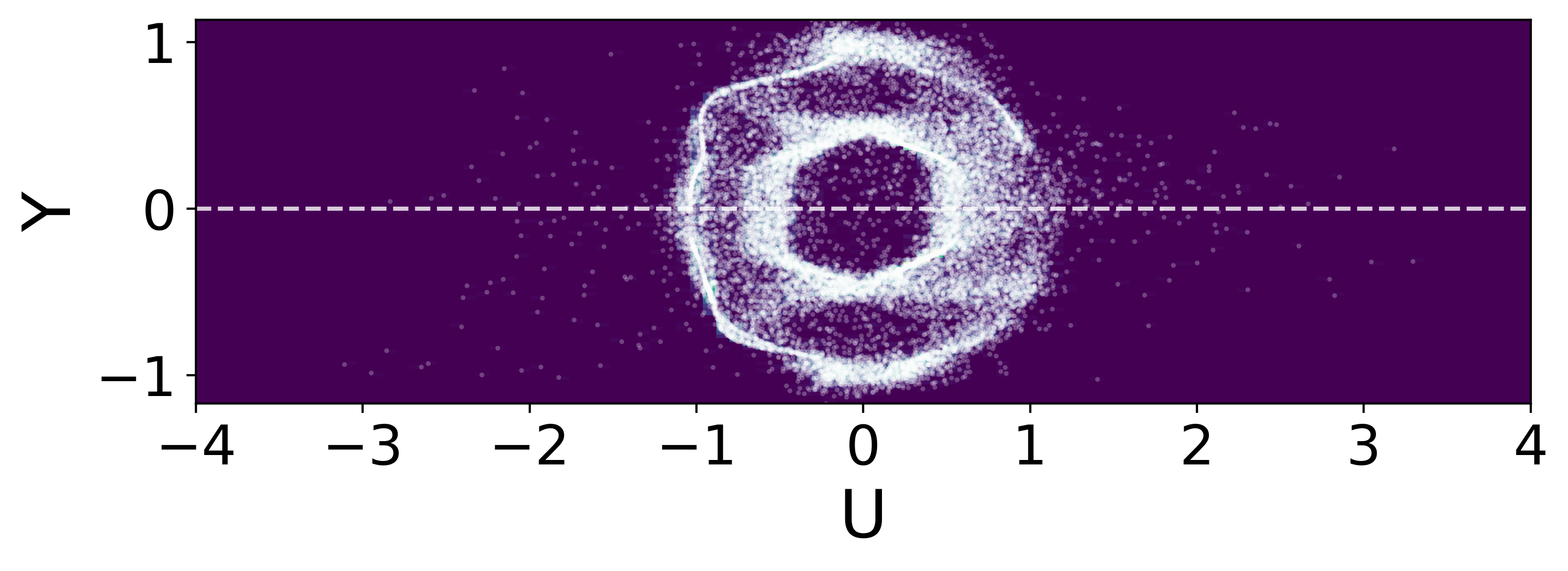}
    }
    \hfill
    \subfloat{
        \includegraphics[width=0.3\linewidth]{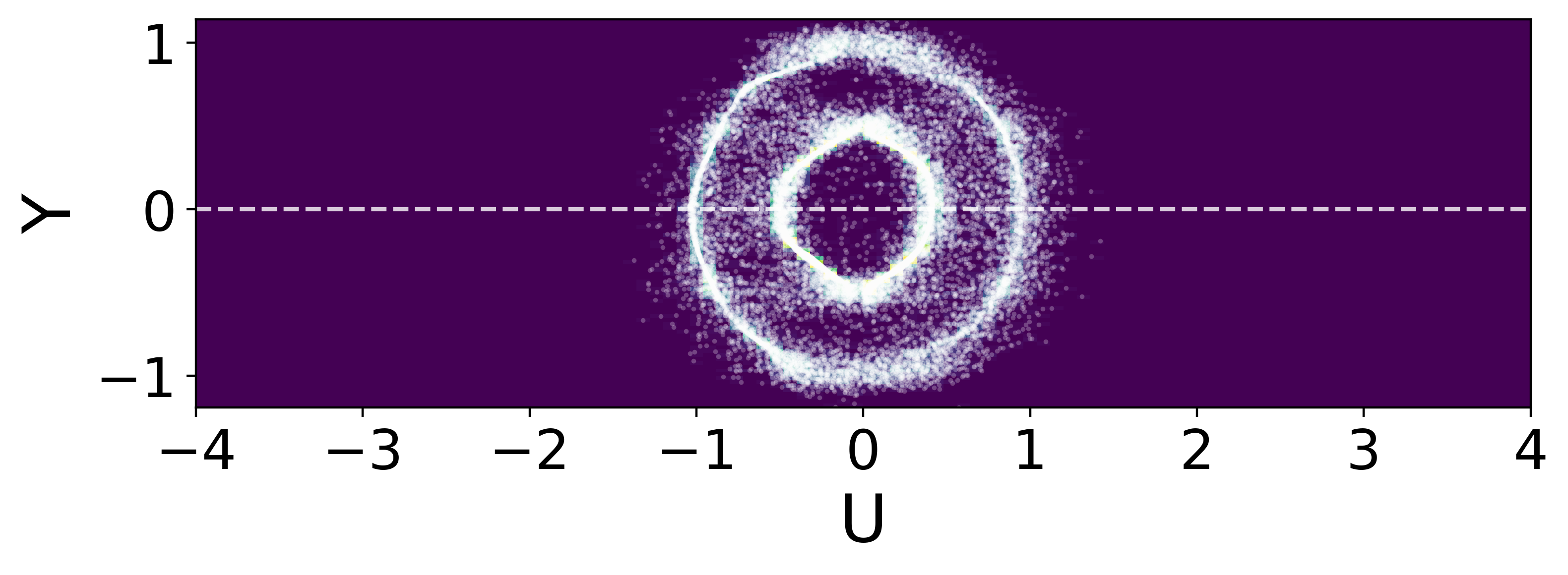}
    }   

    \medskip

    \subfloat{
        \includegraphics[width=0.3\linewidth]{figure/outliar/circles_true_34.png}
    }
    \hfill
    \subfloat{
        \includegraphics[width=0.3\linewidth]{figure/outliar/circles_OT_34.png}
    }
    \hfill
    \subfloat{
        \includegraphics[width=0.3\linewidth]{figure/outliar/circles_UOT_34.png}
    }   

    \medskip

    \setcounter{subfigure}{0}
    \subfloat[True Distribution $\nu$]{
        \includegraphics[width=0.3\linewidth]{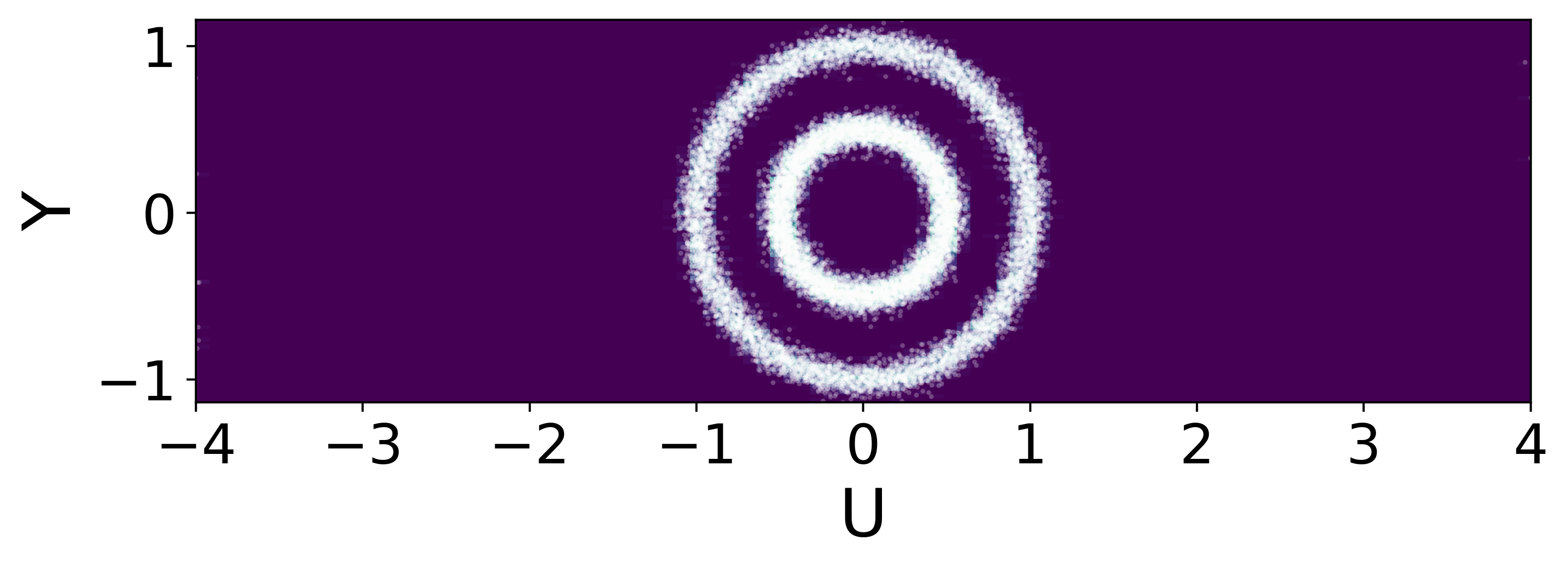}
    }
    \hfill
    \subfloat[Generated distribution (COTM)]{
        \includegraphics[width=0.3\linewidth]{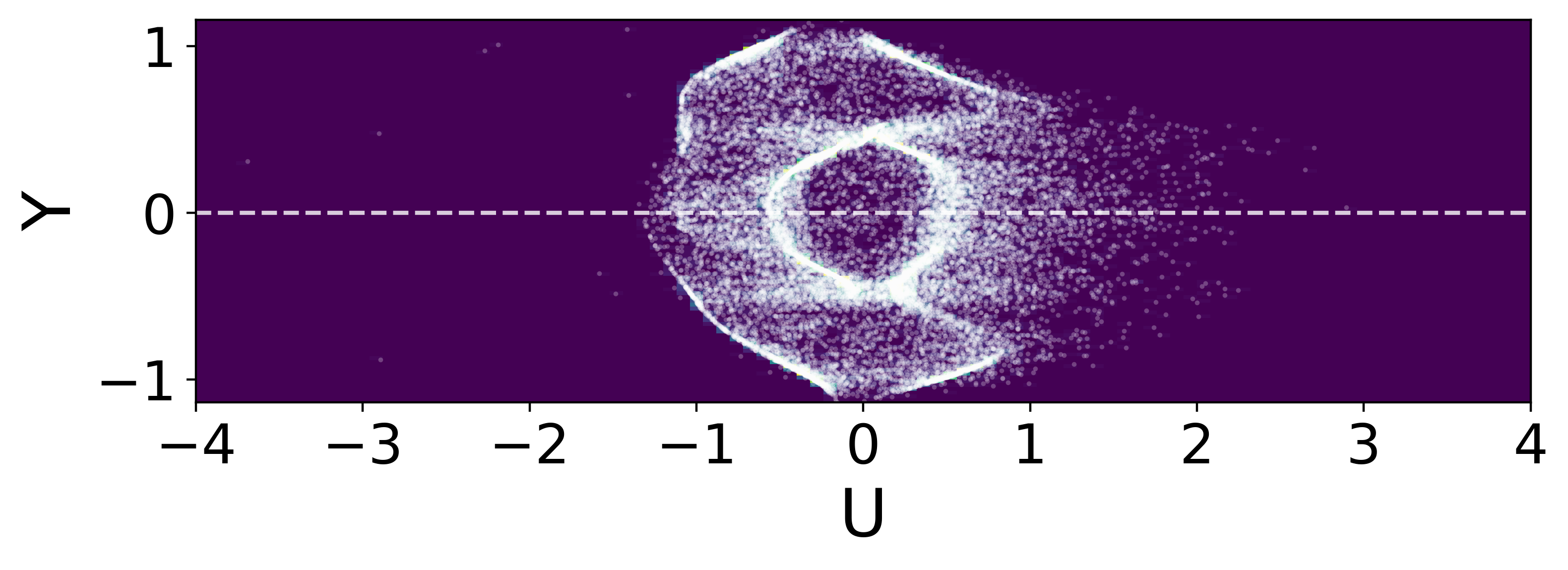}
    }
    \hfill
    \subfloat[Generated distribution (CUOTM)]{
        \includegraphics[width=0.3\linewidth]{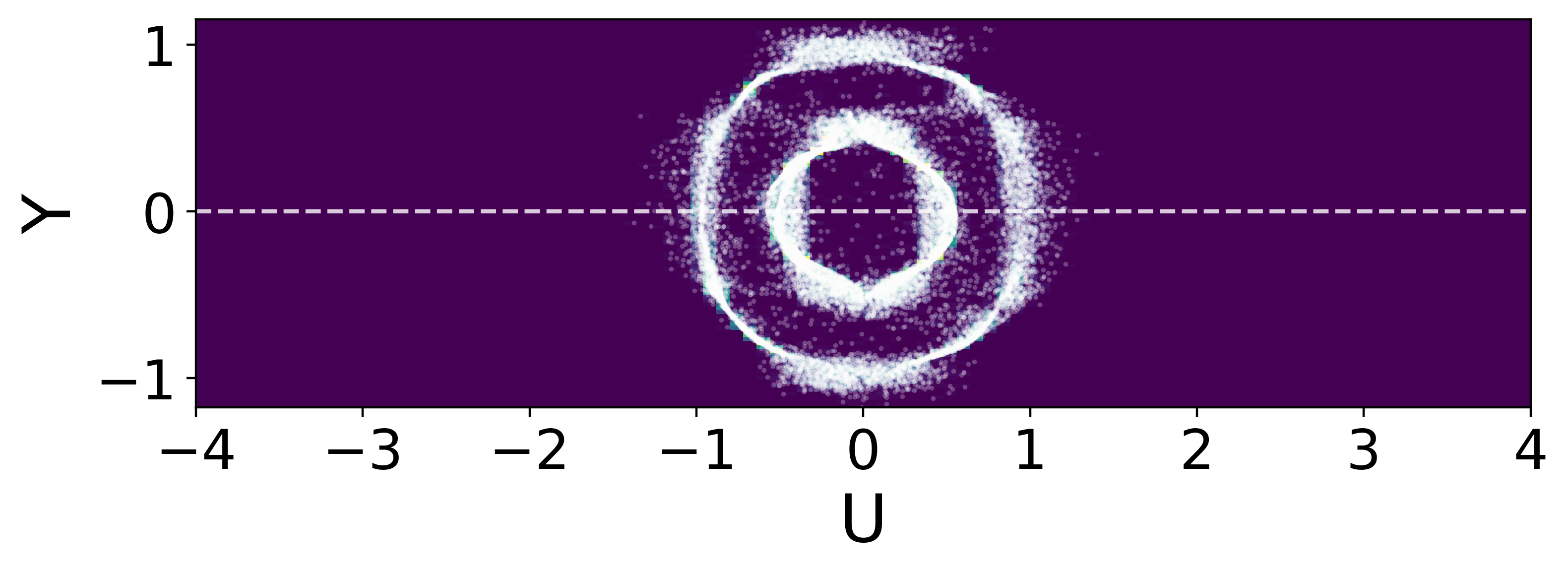}
    }
    
    \caption{\textbf{Qualitative comparison of outlier robustness on the Circles dataset.} The in-distribution modes (two center circles) are perturbed with 1\% outliers. From top to bottom, the outliers are located within annular regions of radii $r \in [1.5, 2]$, $r \in [2, 3]$, $r \in [3, 4]$, and $r \in [4, 5]$.}
    \label{fig:outlier_appendix}
\end{figure*}

\end{document}